\documentclass{article} %

\usepackage{preprint,times} %

\usepackage{amsmath,amsfonts,bm}

\def\eqref#1{equation~\ref{#1}}

\def\1{\bm{1}}

\def\mW{{\bm{W}}}

\DeclareMathAlphabet{\mathsfit}{\encodingdefault}{\sfdefault}{m}{sl}
\SetMathAlphabet{\mathsfit}{bold}{\encodingdefault}{\sfdefault}{bx}{n}

\newcommand{\R}{\mathbb{R}}

\DeclareMathOperator*{\argmax}{arg\,max}

\usepackage{hyperref}
\usepackage{url}

\usepackage{xcolor}         %
\definecolor{burntorange}{rgb}{0.81,.33,0}
\definecolor{darkgreen}{rgb}{0.1,0.6,0.1}

\newcommand{\ShowNotes}{}
\ifdefined\ShowNotes
  \newcommand{\colornote}[3]{{\color{#1}[\textbf{#2} #3\normalfont]}}
\else
  \newcommand{\colornote}[3]{}
\fi

\usepackage{xspace}
\newcommand{\ourmethod}{Gefen\xspace}
\newcommand{\llamaoneb}{Llama 3-1.5B\xspace}
\newcommand{\gptsmall}{GPT-2-125M\xspace}
\newcommand{\gptoneb}{GPT2-1B\xspace}
\usepackage{graphicx}
\usepackage{subcaption}
\usepackage{wrapfig}
\usepackage{tikz}
\usetikzlibrary{spy}
\usepackage{float}
\usepackage{amsthm}
\usepackage[strings]{underscore}
\usepackage{titletoc}
\usepackage{cleveref}
\newtheorem{theorem}{Theorem}[section]
\crefname{theorem}{Theorem}{Theorems}
\Crefname{theorem}{Theorem}{Theorems}
\crefname{figure}{Figure}{Figures}
\Crefname{figure}{Figure}{Figures}
\crefname{table}{Table}{Tables}
\Crefname{table}{Table}{Tables}
\crefname{equation}{Equation}{Equations}
\Crefname{equation}{Equation}{Equations}
\crefname{section}{Section}{Sections}
\Crefname{section}{Section}{Sections}
\crefname{appendix}{Appendix}{Appendices}
\Crefname{appendix}{Appendix}{Appendices}

\usepackage{pifont}
\usepackage{algorithm}
\crefname{algorithm}{Algorithm}{Algorithms}
\Crefname{algorithm}{Algorithm}{Algorithms}
\usepackage[noend]{algorithmic}
\usepackage{graphicx}
\renewcommand{\algorithmiccomment}[1]{\bgroup\hfill $\triangleright$ ~#1\egroup}
\usepackage{hyperref}
\usepackage{siunitx}
\usepackage{pdfrender}
\newcommand{\boldnum}[2][]{%
  \textpdfrender{TextRenderingMode=FillStroke,LineWidth=0.15pt}{\num[#1]{#2}}%
}

\newcommand{\numsub}[2][]{%
  \num[
    uncertainty-mode = full,
    output-open-uncertainty = {_\bgroup\scriptstyle\pm\,},
    output-close-uncertainty = {\egroup},
    #1
  ]{#2}%
}

\newcommand{\boldnumsub}[2][]{%
  \boldnum[
    uncertainty-mode = full,
    output-open-uncertainty = {_\bgroup\scriptstyle\pm\,},
    output-close-uncertainty = {\egroup},
    #1
  ]{#2}%
}

\title{Gefen: Optimized Stochastic Optimizer}

\author{
Nadav Benedek \thanks{Correspondence to \texttt{nadav.benedek@post.runi.ac.il}.}\\
Reichman University \\
\And
Tomer Koren \\
Tel Aviv University, Google Research \\
\And
Ohad Fried \\
Reichman University \\
}

\iclrfinalcopy  %

\begin{document}

\maketitle

\begin{abstract}
AdamW is a default optimizer for modern deep learning, but its first and second moment states add roughly two parameter-sized buffers to training memory, increasing the already substantial cost of large-scale pretraining.
We propose \ourmethod, a memory-efficient optimizer that automatically shares second-moment estimates across parameter blocks and quantizes the first moment using a learned codebook, thereby reducing AdamW's memory footprint by up to \(8\times\) while maintaining the same performance, corresponding to a reduction of 6.5 GiB per billion parameters.
Prior work reduces optimizer memory by sharing second moments across parameters grouped along the Hessian's block-diagonal structure, but relies on hand-specified architectural rules and leaves unexplained why such grouping works.
We prove that large mixed Hessian entries constrain the ratio of squared gradients toward one.
This supplies the missing explanation, since a shared second moment is accurate when the squared gradients it pools are similar.
It also shows the Hessian need not be computed: a block-structured Hessian forces the squared gradients to inherit that structure, so the blocks can be found in the squared gradients themselves.
\ourmethod therefore infers block structure from the initial squared gradients, requiring no architecture-specific metadata or user-tuned hyperparameters beyond AdamW defaults.
\ourmethod learns an exact histogram-based dynamic-programming quantization codebook and reuses the same blocks for first-moment scaling.
Across diverse pretraining experiments, \ourmethod achieves the lowest peak optimizer memory among the compared methods that maintain AdamW-level performance. 
In single-machine and distributed training, the reduced footprint enables larger microbatches and substantially improves throughput over AdamW, making \ourmethod a drop-in replacement that can train larger models or use larger global batch sizes.
We provide the complete Python implementation, including fused CUDA kernels at \url{https://github.com/ndvbd/Gefen}.
\end{abstract}

\section{Introduction}
Adam \citep{kingma2015adam} and AdamW \citep{loshchilov2017decoupled} are workhorse optimizers in deep learning due to their strong and consistent performance across a broad range of architectures and datasets. However, because they store both first and second moment moving averages, their optimizer state alone adds roughly twice the parameter memory. 
Reducing optimizer memory can enable training larger models or using larger batch sizes, which can in turn improve training throughput and final model quality.
Therefore, we are motivated to design an optimizer that maintains AdamW-like performance while reducing memory usage.

\label{sec:desiderata}
Prior work on optimizers varies in the criteria it evaluates, the model architectures it targets, and the metrics it reports, making cross-paper comparisons challenging.
Here, we set forth desiderata for an optimizer: their order does not imply importance, and some elements may be more or
less relevant depending on the use case. In the following, we use
\emph{performance} to denote the relevant evaluation metric, such as validation or training
loss.

\begin{itemize}

    \item \textbf{Peak Memory Footprint}.
    The additional memory beyond model parameters, activations, and gradients should be as small as possible.
    For instance, a model with 1B parameters stores the same number of gradients, and Adam additionally stores the same number of first and second moment values, resulting in two parameter-sized buffers of extra optimizer state.
    Since hardware memory is typically fixed during training, peak memory footprint is a primary constraint on trainable model size, maximum batch size, and often throughput. Note that peak memory footprint is different from the persistent memory of the optimizer, as the optimizer can temporarily allocate intermediates, thereby requiring more memory than the final persistent state.

    \item \textbf{Performance versus Data}. A main goal of an optimizer is to get to the best possible performance as fast as possible (convergence speed).
    However, performance versus wall-clock time depends on the machine configuration, preventing reproducible comparisons across machines, and is therefore better decomposed into two measurements: (1) performance versus data, which is hardware-agnostic, and (2) throughput in units of $data/time$.
    Combining the two gives performance versus time.
    Plots of \emph{performance} versus iteration count are less informative, since different methods may process different amounts of data per step (e.g., due to different batch sizes).
    Furthermore, in many cases, training data is limited, so the primary objective is to maximize performance for a fixed data budget (sample efficiency), and the x-axis should represent the amount of training data consumed rather than time.
    For language models, a typical comparison shows the \emph{performance} (e.g., validation loss) as a function of the amount of training data (e.g., tokens or epochs) consumed.

    \item \textbf{Throughput}. Throughput in units of $data/time$ (e.g., tokens/second) can be estimated from shorter runs under the assumption that it remains approximately stable during training.
    Throughput should be measured over the full training iteration, including the forward and backward passes, data loading, and host-to-device memory transfer.
    An optimizer that enables larger batch sizes, reduces communication overhead in distributed training, or reaches comparable performance in fewer iterations can deliver substantial practical gains.

    \item \textbf{Generality}. A good optimizer should perform robustly across diverse use cases, model architectures, datasets, and hardware configurations. It should require minimal optimizer-specific hyperparameters or manual configuration choices.

    \item \textbf{Theoretical Grounding}. Although many theoretical analyses of optimizers rely on strong assumptions that may not hold in practice (e.g., convexity of the loss landscape), theoretical grounding can still provide useful intuition and increase confidence in the method, even in constrained settings.
    
    \item \textbf{Distributability}. A good optimizer should be compatible with common distributed-training frameworks and sharding strategies, including DDP, FSDP, and others.
    
\end{itemize}
In this paper, we seek to design an optimizer with a \textit{lower peak memory} footprint than AdamW and similar or better \textit{performance} and \textit{throughput}, while remaining \textit{general} and \textit{distributable}.

Adam-mini \citep{zhang2025adam} is a recent variant of Adam that reduces optimizer memory by \textit{sharing} second-moment values using manual rules that mimic the block diagonal structure of parameters' Hessians, but a theoretical justification for \textit{why} Hessian-aligned grouping is beneficial remains missing.
In practice, users may need to provide architecture-specific information (e.g., the number of attention heads), and grouping rules depend on tensor names, which may vary across implementations. Quantized optimizers \citep{dettmers2022eightbit,li2023fourbit} reduce memory, but rely on manually crafted, suboptimal codebooks and introduce an implicit hyperparameter through the quantization block size.

Inspired by these methods, we introduce \ourmethod, an optimizer that shares second-moment estimates across automatically inferred parameter blocks.
\textbf{(1)} We supply the missing justification: we prove that strong Hessian coupling constrains squared gradients to be similar, so Hessian-aligned grouping works precisely because it pools parameters whose squared gradients agree, which is the condition under which a single shared second moment is accurate for every member of a block.
The same result also shows that the Hessian need not be computed: since the Hessian is block-structured, the squared gradients must inherit that structure, so the blocks can be recovered from the squared gradients themselves.
\textbf{(2)} Our method infers periodic parameter blocks from initial squared gradients, without architecture-specific information or tensor-name-dependent rules. These blocks define both second-moment sharing and first-moment quantization scales, avoiding an additional block-size hyperparameter.
\textbf{(3)} We further reduce memory by learning a first-moment quantization codebook using exact histogram-based dynamic programming.
Together, these contributions reduce memory beyond Adam-mini while retaining general applicability and AdamW-level performance across both LLM and non-LLM models.
After reviewing preliminaries and related work in \cref{sec:related_work}, we provide theoretical analysis in \cref{sec:theory} showing that high Hessian affinity implies similar squared gradients, which motivates our grouping strategy.
\cref{sec:algorithm} presents our algorithm, which automatically groups parameters and shares second moments within these groups.
The experiments in \cref{sec:experiments} demonstrate that \ourmethod, a drop-in replacement for AdamW, achieves the lowest peak optimizer memory among the compared methods that match AdamW’s performance, reducing memory by up to \(8\times\) relative to AdamW without learning-rate retuning or additional hyperparameters, while enabling higher throughput.
Readers familiar with AdamW may skip directly to \cref{sec:theory}.

\section{Preliminaries and Related Work}
\label{sec:related_work}
\paragraph{Adam and AdamW.}
Combining ideas from AdaGrad \citep{duchi2011adagrad} and RMSProp \citep{tieleman2012rmsprop}, Adam is an adaptive first-order optimizer that maintains exponential moving averages of the gradient and squared gradient. At step $t$, with gradient $g_t=\nabla_\theta \mathcal{L}(\theta_t)$, Adam computes
\begin{align}
m_t &= \beta_1 m_{t-1} + (1-\beta_1) g_t
&& \text{(EMA of first moment; momentum term)} \\
v_t &= \beta_2 v_{t-1} + (1-\beta_2) g_t^2
&& \text{(EMA of second raw moment)}
\end{align}
then applies bias correction, which is needed because both EMAs are initialized at zero and are therefore biased toward zero at early steps,
\begin{align}
\hat m_t &= \frac{m_t}{1-{\beta_1}^t}, \qquad
\hat v_t = \frac{v_t}{1-{\beta_2}^t} && \text{(Bias correction)} 
\end{align}
and updates the parameters using learning rate $\eta$ as
\begin{align}
\theta_{t+1} = \theta_t - \eta \frac{\hat m_t}{\sqrt{\hat v_t}+\epsilon}
\end{align}
AdamW uses the same adaptive moments, but applies weight decay directly in parameter space using coefficient $\lambda$, rather than injecting a penalty term into the gradient:
\begin{align}
\theta_{t+1} = \theta_t - \eta \frac{\hat m_t}{\sqrt{\hat v_t}+\epsilon} - \eta \lambda \theta_t
\end{align}
This decoupling avoids interactions between adaptive gradient scaling and regularization strength, and typically yields better generalization and more stable hyperparameter tuning.
In modern deep learning practice, Adam and AdamW are the de facto standard optimizers. They keep first and second moment states, typically in FP32, so optimizer-state memory is roughly $8$ bytes per parameter. For example, for a 13B-parameter model, Adam-style optimizer states alone require about $13\times 10^9 \times 8 \approx 97$ GiB; this can make training more difficult on memory-constrained GPUs and force smaller batch sizes, which may also reduce training throughput.

\paragraph{Quantized optimizer states.}
Recent work reduces optimizer memory by quantizing Adam-style moment states. Adam8bit \citep{dettmers2022eightbit} flattens tensors into fixed-size blocks ($B=2048$), quantizes both first and second moments, and stores per-block absolute-max scaling (recomputed each step) together with 8-bit values per element. Its Dynamic Tree Quantization encodes normalized values in $[-1,1]$, providing higher precision near zero while still representing larger magnitudes. However, this design has practical limitations: (1) the dynamic-tree quantizer can be wasteful due to duplicated codebook entries, (2) in practice the method is not applicable to all tensor types (e.g., embeddings should not be quantized), so it is not fully plug-and-play as a universal drop-in, and (3) the fixed user-chosen block size introduces an additional hyperparameter burden. Adam4bit \citep{li2023fourbit} pushes state quantization to 4 bits with improved handling of outliers (including finer-grained and asymmetric treatments for moment statistics), further reducing memory while preserving competitive accuracy in many settings. The method includes a block-size hyperparameter, which defaults to $128$ in the authors' implementation.

See \cref{app:extended_related} for an extended related work discussion, including additional optimizers and gradient communication compression methods. In this work, we focus on general-purpose optimizers that require no manual configuration, additional hyperparameters, or learning-rate retuning, and that can be used as drop-in replacements for AdamW across training regimes while achieving a significantly lower memory footprint.

\section{Theoretical Analysis}
\label{sec:theory}
AdamW maintains, for each parameter, an exponential moving average of its squared gradient. A single such estimate can be shared across a group of parameters with minimal loss of accuracy when their squared gradients are similar. We therefore analyze conditions under which two entries of a parameter tensor with a large-magnitude mixed Hessian entry also have similar squared-gradient magnitudes. This connects the criterion we can measure, squared-gradient similarity, to the Hessian-based coupling that prior work targets with hand-specified, architecture-dependent rules.

\begin{theorem}[Large Hessian entries contract squared-gradient ratios]
\label{thm:hessian-gradient-ratio}
Consider the two-layer MLP
$z=W_1x+b_1$, $y=W_2\sigma(z)+b_2$, where $y\in\mathbb{R}^{d_{\mathrm{out}}}$,
with MSE loss
$L=d_{\mathrm{out}}^{-1}\sum_{j=1}^{d_{\mathrm{out}}} (y_j-t_j)^2$.
Let $e_j:=y_j-t_j$.
Define:
\[
\begin{aligned}
&A_k:=\sum_{i=1}^{d_{\mathrm{out}}} e_i W_2[i,k],\quad
U_{k,l}:=|x_l\sigma'(z_k)|,
&C_{k,k'}:=\left|\frac{A_k}{A_{k'}}\right|,\quad
D_{k,k'}:=\left|\sum_i W_2[i,k]W_2[i,k']\right|
\end{aligned}
\]
Assume:
\[
\begin{aligned}
&D_{k,k'}\le \alpha,\quad
U_{k,l}\le \beta,\quad
|\log C_{k,k'}|\le \gamma,
&\forall k,l,k',l'\ \text{s.t.}\
k\neq k',\ A_k,A_{k'}\neq 0,\ H_{(k,l),(k',l')}\neq 0
\end{aligned}
\]
For any such pair of entries $W_1[k,l]$ and $W_1[k',l']$,
let $g_{k,l}:=\partial L/\partial W_1[k,l]$ and define the ratio between
their squared gradients and their mixed Hessian entry as
\[
R:=\frac{g_{k,l}^2}{g_{k',l'}^2},
\qquad
H_{(k,l),(k',l')}
:=\frac{\partial^2 L}{\partial W_1[k,l]\,\partial W_1[k',l']}
\]
Then, for $A:=2\alpha\beta^2/d_{\mathrm{out}}$,
\[
e^{-2\gamma}\left(\frac{|H_{(k,l),(k',l')}|}{A}\right)^2
\le
R
\le
e^{2\gamma}\left(\frac{A}{|H_{(k,l),(k',l')}|}\right)^2
\]

\end{theorem}

The theorem implies that for feasible $|H_{(k,l),(k',l')}|\le A$, larger mixed Hessian magnitude
tightens the admissible range of the squared-gradient ratio around $1$.
The same form of contraction also holds for second-layer weights that share
the same output coordinate and for hidden-bias terms, whereas output-bias Hessian
entries are either zero or constant. 
Empirically, the same phenomenon is observed in real network architectures, including attention tensors and convolutional layers, as shown in Appendix~\ref{app:empirical-evidence}.
\begin{theorem}[Within-layer Hessian affinity controls gradient magnitudes] \label{thm:gauss-newton}
Let \(f_\theta(x)\in\mathbb{R}^m\) be a feedforward ReLU network, and let \(L(\theta)=\ell(f_\theta(x))\), where \(\ell:\mathbb{R}^m\to\mathbb{R}\) is twice differentiable and convex.
Write \(g=\nabla_\theta L(\theta)\), \(H=\nabla_\theta^2L(\theta)\), and \(z=f_\theta(x)\), and assume that $\nabla \ell(z)$ is in the column span of $\nabla^2 \ell(z)$.
Further assume that \(\theta\) is away from ReLU activation boundaries for this input, i.e., all preactivations are nonzero.
Then for any two parameters \(i,j\) belonging to the same layer, we have \(H_{ii} \geq 0, H_{jj} \geq 0\), and for the normalized Hessian affinity
\(
\rho_{ij} = {|H_{ij}|}/{\sqrt{H_{ii}H_{jj}}}
\)
we have
\[
\left|
\frac{g_i^2}{H_{ii}}
-
\frac{g_j^2}{H_{jj}}
\right|
\le
2\lambda(z)^2\sqrt{2(1-\rho_{ij})},
\]
where $\lambda(z)$ is the Newton decrement at $z$, namely
\(
\lambda(z)^2
=
\nabla\ell(z)^\top
\bigl(\nabla^2\ell(z)\bigr)^\dagger
\nabla\ell(z)
.
\)
Here, we use the convention that $g_i^2/H_{ii}=0$ whenever $H_{ii}=0$, and similarly $\rho_{ij} = 0$ whenever either $H_{ii}=0$ or $H_{jj}=0$.

\end{theorem}
Thus, if the affinity \(\rho_{ij}\) is close to one, then the Hessian-normalized gradient magnitudes of coordinates \(i\) and \(j\) are close.
If, additionally, \(H_{ii}\) and \(H_{jj}\) are comparable, then the raw gradient magnitudes \(|g_i|\) and \(|g_j|\) are comparable up to the same curvature scale. For the squared loss \(\ell(z)=o^{-1}\|z-y\|^2\), we have \(\lambda(z)^2=2\ell(z)\), so the theorem gives
\[
\left|
\frac{g_i^2}{H_{ii}}
-
\frac{g_j^2}{H_{jj}}
\right|
\le
4\ell(z)\sqrt{2(1-\rho_{ij})}.
\]

Proofs are in \cref{app:theory-proof}.

\clearpage
\section{Algorithm}
\label{sec:algorithm}
\newcommand{\automaticpartitioning}{\textsc{Automatic Block Partitioning}}
\newcommand{\exacthistcodebook}{Exact DP Quantization Codebook Learning}

\begin{algorithm}[tbp]
\caption{Gefen: our proposed algorithm for optimized stochastic optimization. Default settings are identical to AdamW. 
All operations on vectors are element-wise. The algorithm has a fused CUDA implementation to avoid unnecessary temporary memory allocations.}
\label{alg:gefen}
\footnotesize
\begin{algorithmic}[1]
\STATE{\textbf{Require:} $\alpha,\lambda\in \R$: Stepsize and Weight decay, $\beta_1, \beta_2 \in [0,1) $: Decay rates for moment estimates. $f(\theta)$: Stochastic objective function with parameters $\theta$. Initialize $t \leftarrow 0$, $\bm{G}_0=\nabla_\theta f_t(\bm{\theta}_{0})$}

\STATE{\textbf{for} each parameter-tensor $\theta$ run \hyperref[alg:automatic_partitioning]{\automaticpartitioning{}} using $\bm{G}_0$} 

\STATE{Run \hyperref[alg:exact_histogram_codebook]{\textsc{\exacthistcodebook{}}}.}

\FOR{each parameter-tensor $\theta$  }

        \STATE{\quad Initialize quantized 1\textsuperscript{st} momentum codebook $ \bm{\bar{m}}_0 \leftarrow \bm{0}$ per weight and $||\bm{m}_0||_\infty$ per block} 

        \STATE{\quad Initialize a single 2\textsuperscript{nd} moment scalar $\bar{\bm{v}}_0\leftarrow 0$ per block}  

\ENDFOR

\REPEAT

	\STATE{$t \leftarrow t + 1$}

    \FOR{each parameter-tensor $\theta$  } 

        \STATE{$\bm{g}_t \leftarrow \nabla_\theta f_t(\bm{\theta}_{t-1})$}  \COMMENT{Get gradients for the current parameter}

        \STATE{$\bm{m}_t \leftarrow \beta_1 \cdot \textsc{dequantize}(\bm{\bar{m}}_{t-1}, ||\bm{m}_{t-1}||_\infty) + (1 - \beta_1) \cdot \bm{g}_t $} \COMMENT{Update biased momentum estimate}

        \STATE{$\bm{\bar{m}}_t, ||\bm{m}_{t}||_\infty \leftarrow \textsc{Quantize}(\bm{m}_t)  $}  \COMMENT{Quantize and recalculate absmax per block}

        \STATE{$\hat{\bm{m}}_t \leftarrow \bm{m}_t/(1 - {\beta_1}^t) $} \COMMENT{Bias correction for momentum}

        \STATE{$\bm{v}_t \leftarrow \beta_2 \cdot \bm{v}_{t-1} + (1 - \beta_2) \cdot mean(\bm{g}^2_t) $}  \COMMENT{Update biased second moment using a mean per block}

        \STATE{$\hat{\bm{{v}}}_t \leftarrow \bm{v}_t/(1 - {\beta_2}^t) $}
        \COMMENT{Bias correction of compact second moment estimate}

        \STATE{$\bm{\theta}_t \leftarrow \bm{\theta}_{t-1} - \alpha \cdot \left(  \hat{\bm{m}}_t / (\sqrt{\hat{\bm{v}}_t} + \epsilon) + \lambda\bm{\theta}_{t-1} \right)$} \COMMENT{Compact $\hat{\bm{{v}}}_t$ are broadcasted to full shape}
    \ENDFOR

\UNTIL{ \textit{stopping criterion is met} }
\RETURN{optimized parameters $\bm{\theta}_t$}
\end{algorithmic}
\end{algorithm}

\begin{wrapfigure}{r}{0.60\linewidth}
\vspace{-1.2em}
\hrule
\vspace{0.4em}
\refstepcounter{algorithm}
\label{alg:automatic_partitioning}
\noindent\textbf{Algorithm~\thealgorithm: \automaticpartitioning{}}
\par\smallskip
\hrule
\vspace{0.4em}
\footnotesize
\begin{algorithmic}[1]
\STATE{\textbf{Require:} first-step gradient $\bm{g}_0$ for a parameter, $n=|\bm{g}_0|$, $\epsilon$}
\STATE{Flatten $\bm{g}_0$ and let $\mathcal{P}$ be all divisors $p$ of $n$, sorted increasingly}
\STATE{$\bm{z}\leftarrow \log(\bm{g}_0^2+\epsilon)$, $\mathcal{C}\leftarrow\emptyset$}

\FOR{each candidate period $p \in \mathcal{P}$}
    \STATE{Reshape $\bm{z}$ into $\bm{B}\in\R^{(n/p)\times p}$}
    \STATE{$\mu_i(p) \leftarrow \operatorname{mean}(\bm{B}_{i,:}) \quad \forall i\in\{1,\ldots,n/p\}$}
    \STATE{$S(p)\leftarrow \operatorname{mean}_i |\mu_{i+1}(p)-\mu_i(p)|$} \COMMENT{\cref{eq:automatic_partition_score}}
\ENDFOR

\FOR{each $p\in\mathcal{P}$ s.t. $1<p<n$}
    \IF{$S(p)$ is larger than the scores of its neighboring candidate periods}
        \STATE{Add $p$ to $\mathcal{C}$}
    \ENDIF
\ENDFOR
\STATE{\textbf{return} $\argmax_{p\in\mathcal{C}} S(p)$ if $\mathcal{C}\neq\emptyset$, otherwise $1$}

\end{algorithmic}
\vspace{0.4em}
\hrule
\vspace{-0.4em}
\end{wrapfigure}

Computing the Hessian directly in large models is not tractable, and it is also unnecessary: the accuracy of a shared second moment is governed by the similarity of the squared gradients it pools, which we can measure at the first step. \cref{sec:theory} explains why such groups exist and align with architecture, since Hessian-coupled entries are forced to have similar squared gradients, so the structure our algorithm detects is the same structure prior work specifies by hand.
We therefore present \cref{alg:gefen}, a method that partitions each parameter tensor into blocks whose squared gradients are similar, and shares the moving average of the second moment within each block in a fully automatic way, without relying on user-specified configuration. 
We further reduce the memory footprint of the first moment using a \textit{learnable} and \textit{exact} dynamic programming quantization, which \textit{reuses} the same block partitioning and therefore does not introduce an additional scaling constant.

\begin{algorithm}[tpb]
\caption{\textsc{\exacthistcodebook{}}}
\label{alg:exact_histogram_codebook}
\footnotesize
\begin{algorithmic}[1]
\STATE{\textbf{Require:} gradients $\{\bm{g}^{(\theta)}_0\}$, automatic periods $\{p_\theta\}$}
\STATE{Set number of codebook entries $k=256$. Initialize histogram counts $\bm{c}\in\R^{16k}$ over $[-1,1]$}
\FOR{each parameter tensor $\theta$}
    \STATE{Reshape $\bm{g}^{(\theta)}_0$ into blocks $\bm{B}\in\R^{(n_\theta/p_\theta)\times p_\theta}$}
    \STATE{Normalize each nonzero block by its maximum absolute value: $\bm{Z}_{i,:}\leftarrow \bm{B}_{i,:}/\|\bm{B}_{i,:}\|_\infty$}
    \STATE{Accumulate all entries of $\bm{Z}$ into histogram counts $\bm{c}$}
\ENDFOR
\STATE{Let $\{(m_i,c_i)\}_{i=1}^{b}$ be the nonempty histogram bin centers and counts, sorted by $m_i$}
\STATE{Compute prefix sums of $c_i$, $c_i m_i$, and $c_i m_i^2$}
    \STATE{Define the interval cost function $C(b_\ell,b_r,k')$ for assigning bins $[b_\ell,b_r]$ to codebook index $k'$: use center $-1.0$ if $k'=1$, center $1.0$ if $k'=k$, and otherwise use the weighted mean of bins $[b_\ell,b_r]$} 
\COMMENT{\cref{app:dp_details}}
\STATE{$D_{b_r,k'}:=$ the cost for quantizing the first $b_r$ bins using the first $k'$ codepoints; initialize $D_{0,0}\leftarrow 0$ and all other entries to $\infty$}
\FOR{$k'=1,\ldots,k$}
    \FOR{$b_r=k',\ldots,b$} 
        \STATE{$D_{b_r,k'}\leftarrow \min_{b_\ell\le b_r}\; D_{b_\ell-1,k'-1}+C(b_\ell,b_r,k')$}
        \STATE{Store the minimizing split point}
    \ENDFOR
\ENDFOR
\STATE{Backtrack the stored splits to recover the $k$ codebook entries}
\STATE{\textbf{return} the sorted codebook}
\end{algorithmic}
\end{algorithm}
\subsection{Automatic Partitioning}
As shown in \cref{fig:hessian_and_sq_grad} and by \citet{collobert2004large,zhang2025adam}, the Hessians of some parameters exhibit block-diagonal structure during pretraining. 
Splitting this parameter into 4 blocks, corresponding to the number of attention heads in this example, groups together parameters with stronger Hessian coupling, which, as suggested by the theorem and empirical evidence, leads to more similar squared gradients. Splitting into 16 groups provides finer granularity, as shown in the figure. However, computing the two-dimensional Hessian of mixed gradients is not tractable in large models.
We therefore use the observed structure of the squared gradients: the same 16-block partition produces regions with more similar squared-gradient magnitudes. This observation motivates our one-dimensional squared-gradient-based algorithm for partitioning parameters into blocks, without computing the Hessian directly. \Cref{sec:pretraining_vs_finetuning} discusses how this squared-gradient periodicity differs between pretraining and finetuning.
For each parameter tensor, \ourmethod infers the sharing pattern from the first-iteration gradient
only.
Let $\bm{g}_0 \in \R^n$ be the \textit{flattened} gradient of a parameter tensor at
the first optimization step.
\begin{wrapfigure}{r}{0.50\linewidth}
    \centering
    \includegraphics[width=0.8\linewidth]{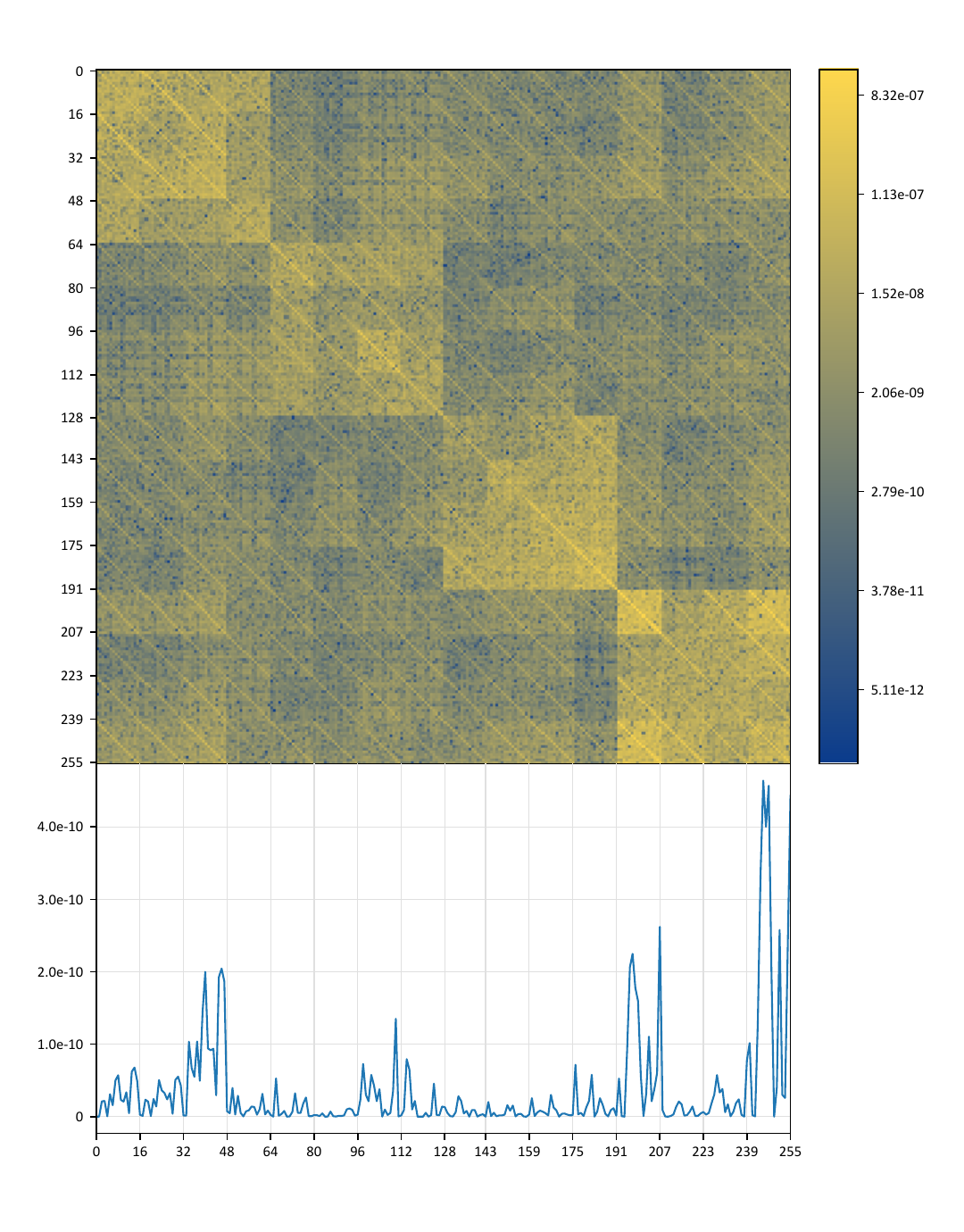}
    \caption{Hessian of a toy transformer attention-query parameter, showing block-diagonal structure, and the corresponding squared gradients of the same parameter below. Qualitatively, both exhibit a visible 16-block structure.}
    \label{fig:hessian_and_sq_grad}
    \vspace{-6.5em}
\end{wrapfigure}
 
For every divisor $p$ of $n$ we reshape
$\log(\bm{g}_0^2+\epsilon)$ into consecutive blocks of size $p$ and measure the average contrast between adjacent block means:
\begin{equation}
\begin{aligned}
    \mu_i(p)=\frac{1}{p}\sum_{j=(i-1)p+1}^{ip}\log(g_{0,j}^2+\epsilon) \\
    S(p) =
    \frac{1}{n/p-1}\sum_{i=1}^{n/p-1}
    \left|\mu_{i+1}(p)-\mu_i(p)\right|
\end{aligned}
    \label{eq:automatic_partition_score}
\end{equation}
For large $n$, the number-of-divisors function $\tau(n)$ satisfies the classical bound $\tau(n) \le n^{(\ln 2 + o(1))/\ln\ln n}$, showing that the number of divisors grows subpolynomially in $n$ \citep{landau1909handbuch}, so the number of candidate periods is low.
For example, the maximum value of $\tau(n)$ over the intervals
$[0.9\cdot 10^6,1.1\cdot 10^6]$ and
$[0.9\cdot 10^9,1.1\cdot 10^9]$ is only $256$ and $1344$,
respectively.
Large values of $S(p)$ indicate that the block boundaries induced by period $p$
separate regions with different squared-gradient magnitudes. 
Ideally, we want large $p$ for reduced memory while still aligning block
boundaries with structure in the squared gradients.

Some divisors align better with the structure of the Hessian and therefore
produce pronounced local peaks in $S(p)$. 
We sort the candidate periods in increasing order, keep only strict local maxima, and choose the remaining period with the largest score.
If no such candidate exists, we set $p=1$.
The selected period is fixed for the remainder
of training. \cref{alg:automatic_partitioning} provides the pseudocode of the algorithm. The period we found has dual use: it is used to share the second moment, and also to quantize the first moment, as described in the next section.

\subsection{Exact Quantization of Momentum Using Histogram-Based Dynamic Programming}
To reduce first-moment memory, we learn a codebook using exact histogram-based dynamic programming. Unlike the fixed quantization schemes of Adam8bit and Adam4bit, our method reuses the existing block partitioning without an additional quantization block-size hyperparameter.
Lloyd-Max \citep{lloyd1982least,max1960quantizing} is a classical algorithm for scalar quantization. It iteratively alternates between assigning each value to the nearest codebook entry and updating each codebook entry to be the mean of the values assigned to it. However, we identify two major problems with this approach. The first problem is \textit{contraction}. Because Lloyd-Max learns its codebook from the observed values, the two extreme codebook entries are not necessarily fixed at the endpoints of the normalized range $[-1,1]$. If we recompute the quantization scale at every optimizer step, this can create an undesirable feedback loop: after dequantization, the largest stored value is bounded by the largest learned codebook entry, which may be strictly smaller than $1$. The next absmax estimate is then smaller as well, causing the effective quantization range to shrink over time. To solve this problem, we force the codebook to always include the two extreme entries $-1$ and $1$, which guarantees that the quantization range is preserved and prevents contraction. Another issue is that we found Lloyd-Max to be sensitive to the codebook initialization strategy: linear spacing or initialization by quantiles can lead to different results, with no consistent winner. See \cref{app:exact-dp-quantization} for more details. 

One-dimensional Lloyd-Max quantization is not guaranteed to converge to the optimal codebook solution and is sensitive to the codebook initialization strategy (\cref{app:exact-dp-quantization}).
We therefore seek an optimal solution.
The quantization problem is NP-hard in dimension two or more \citep{aloise2009np,dasgupta2008hardness,mahajan2012planar}; however, for the one-dimensional case, inspired by \citet{wang2011ckmeans}, we design an optimal dynamic programming solution that operates on the histogram of the first-step gradients and runs in $O(b^2k)$, where $b$ is the number of bins in the histogram and $k$ is the number of quantization levels. 
Histogram calculation is performed in CUDA kernels, and dynamic programming is performed on the CPU using Numba \citep{lam2015numba}, allowing the codebook learning process, which takes place only once at the beginning of training, to finish quickly.
See \cref{alg:exact_histogram_codebook,app:dp_details} for more details on the algorithm. 
See \cref{app:exact-dp-quantization} for a comparison of the exact method with the Lloyd-Max heuristic, and \cref{app:ablations} for ablations.

\section{Experiments}
\label{sec:experiments}
\begin{wraptable}{r}{0.40\linewidth}
\vspace{-1.2em}
\centering
\caption{Memory footprint normalized by parameter memory. All methods use
momentum. The two best values are bolded. As the model size grows, \ourmethod has a \textbf{lower} memory footprint
than all Adam variants. 
}
\label{tab:memory_footprint}
\resizebox{\linewidth}{!}{%
\begin{tabular}{lcccc}
\\
\hline
Optimizer & \multicolumn{2}{c}{CNN (1.2M)} & \multicolumn{2}{c}{GPT-2 (125M)} \\

  & Persistent mem & Peak mem & Persistent mem & \textbf{Peak mem} \\
\hline
AdamW & x2.00 & x2.00 & x2.01 & x2.01 \\
Adam8bit & x0.51 & x0.51 & x0.51 & x0.51 \\
Adam4bit & \textbf{x0.27} & \textbf{x0.27} & \textbf{x0.30} & \textbf{x0.30} \\

Adam-mini  & x1.00 & x2.97  & x1.01 & x1.02 \\
Adafactor  & x1.02 & x3.05  & x1.01 & x3.01 \\
Muon & x1.00 & x3.12  & x1.01 & x1.63 \\
SM3  & x1.01 & x2.02  & x1.01 & x1.84 \\

\ourmethod  & \textbf{x0.27} & \textbf{x0.29} & \textbf{x0.25} & \textbf{x0.25} \\

\hline
\end{tabular}
}
\vspace{-1.2em}
\end{wraptable}

We evaluate \ourmethod across the desiderata: peak memory footprint, performance, throughput, and distributability, using a variety of model architectures, model sizes, and datasets. 
Both in theory and in implementation, \ourmethod reduces the peak and persistent memory of the optimizer state by \(8\times\) in full precision training, while achieving performance similar to AdamW.
In the evaluated distributed-training configurations, \ourmethod improves FSDP throughput by 56\% compared to AdamW and enables the DDP setting where AdamW cannot fit even a microbatch size of 1. Alternatively, one can use the reduced memory footprint to increase model size. Implementation details and hyperparameters are found in \cref{app:hyperparameters}.
\subsection{Peak Memory Footprint}
\newcommand{\gefenCNNMNISTValidationLoss}{0.0283 +- 0.0024}
\begin{wraptable}[8]{r}{0.30\linewidth}
\vspace{-1.2em}
\centering
\caption{CNN on MNIST.
}
\label{tab:mnist_validation_loss}
\tiny
\begin{tabular}{lc}
\\
\hline
Optimizer & Validation Loss $\downarrow$ \\
\hline
AdamW & \boldnumsub{0.02817 +- 0.0023} \\ %
Adam8bit & \boldnumsub{0.0290+-0.0020}  \\ %
Adam4bit & \numsub{0.0313+-0.0013} \\ %
Adam-mini & \numsub{0.0321 +- 0.0004} \\
Adafactor & \numsub{0.0398+-0.0051} \\
Muon & \numsub{0.0351+-0.0003} \\ 
SM3 & \numsub{0.0455+-0.0012} \\ %
\hline
\ourmethod & \boldnumsub{\gefenCNNMNISTValidationLoss} \\ %
\hline
\end{tabular}
\end{wraptable}

We begin with the first criterion in the desiderata above: the optimizer's peak memory footprint. We compare \ourmethod against the methods described in \cref{sec:related_work}, along with a few representative methods from \cref{app:extended_related}, and report the results in \cref{tab:memory_footprint}. 
Consistent with the algorithm description and our empirical measurements, \ourmethod reduces AdamW's peak and persistent 32-bit memory footprints by \(\mathord{\sim}8\times\). 
Although Adam4bit has a compelling memory
footprint, as shown next, it achieves suboptimal quality compared with the other
methods.

In the next section, for the \llamaoneb model, the total optimizer state is 11.2 GiB with AdamW and 1.5 GiB with \ourmethod, a reduction of 9.7 GiB that can be crucial in memory-constrained settings. 
For example, training Llama 8B can be done on four A6000 GPUs with \ourmethod, whereas with AdamW we had to use four H100 GPUs due to insufficient memory.
\begin{wrapfigure}{r}{0.45\linewidth}
    \vspace{0.5em}
    \centering
    \begin{minipage}{0.48\linewidth}
        \centering
        \includegraphics[width=\linewidth]{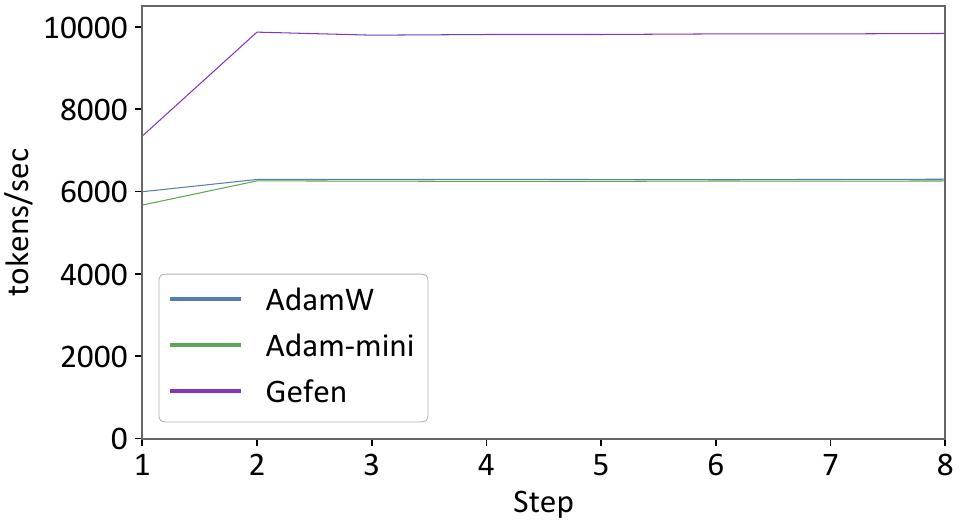}        
        {\scriptsize (a) \llamaoneb with FSDP}
    \end{minipage}
    \hfill
    \begin{minipage}{0.48\linewidth}
        \centering
        \includegraphics[width=\linewidth]{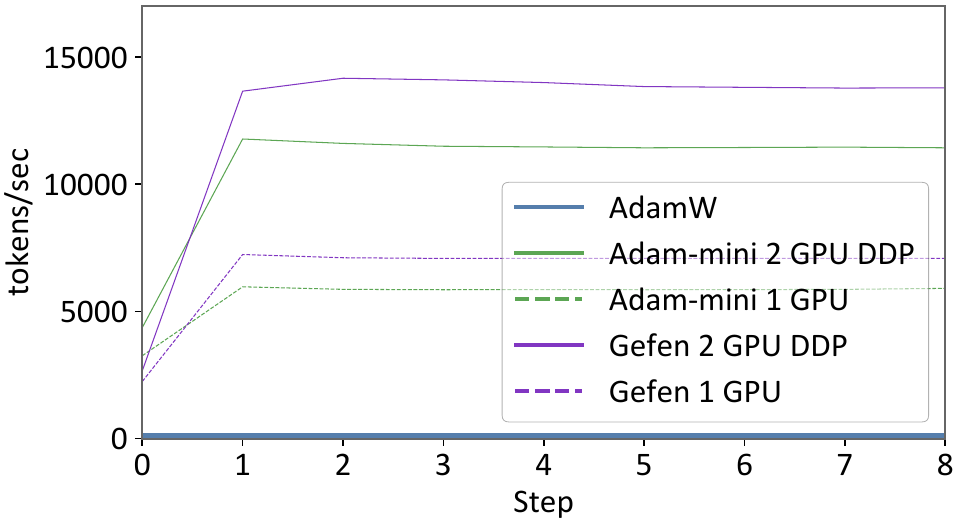}        
        {\scriptsize (b) \gptoneb with DDP}
    \end{minipage}
    \caption{Training throughput in two distributed settings.
    \ourmethod improves FSDP throughput by 56\%, and in DDP and single-GPU settings enables training where AdamW is infeasible and improves throughput by 21\% over Adam-mini.
    }
    \label{fig:throughput}
    \vspace{-2.9em}
\end{wrapfigure}

For Qwen1.5-MoE-14.3B pre-training, AdamW's total optimizer state occupies 106.7 GiB, compared with 13.4 GiB for \ourmethod, an approximately eightfold reduction.
Next, we focus mainly on optimizers whose \textit{peak} memory footprint is \(\times 1\) or lower, i.e., at least a \(2\times\) reduction relative to AdamW.

\subsection{Performance}
\Cref{fig:gpt2_and_llama_train_curve} shows performance across several settings: \llamaoneb and Llama 8B on C4, Qwen1.5-MoE-A2.7B (14.3B total, 2.7B active parameters) on C4, Qwen3-4B on FineWeb-Edu, \gptsmall on OpenWebText, ResNet18, ResNet101, and DiT-XL/2 on ImageNet, CNN on MNIST and DDPM on Bollywood-Celebs.
\ourmethod performs on par with AdamW across the settings. Notably, in all experiments, we use the same learning rate for \ourmethod as AdamW to demonstrate that our method can serve as a drop-in replacement.
In \Cref{tab:mnist_validation_loss} it can be seen that not all methods are AdamW learning rate compatible.
\begin{figure}[H]
    \centering
    \begin{minipage}{0.31\linewidth}
        \centering
        \resizebox{\linewidth}{!}{%
        \begin{tikzpicture}[
            spy using outlines={circle,magnification=2.5,size=1.8cm,connect spies}
        ]
            \node[anchor=south west, inner sep=0] at (0,0) {
                \includegraphics[width=7cm]{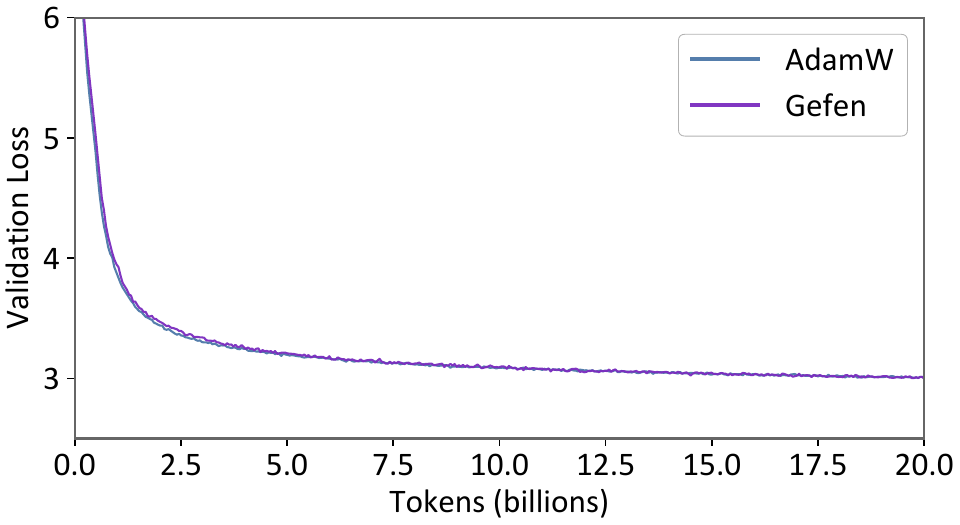}
            };
            \spy[red] on (3.5,1.2) in node at (2.0,2.6);
        \end{tikzpicture}
        }
        
        {\scriptsize (a) Loss: \gptsmall on OpenWebText}
    \end{minipage}
    \hfill
    \begin{minipage}{0.31\linewidth}
        \centering
        \resizebox{\linewidth}{!}{%
        \begin{tikzpicture}[
            spy using outlines={circle,magnification=2.5,size=1.6cm,connect spies}
        ]
            \node[anchor=south west, inner sep=0] at (0,0) {
                \includegraphics[width=7cm]{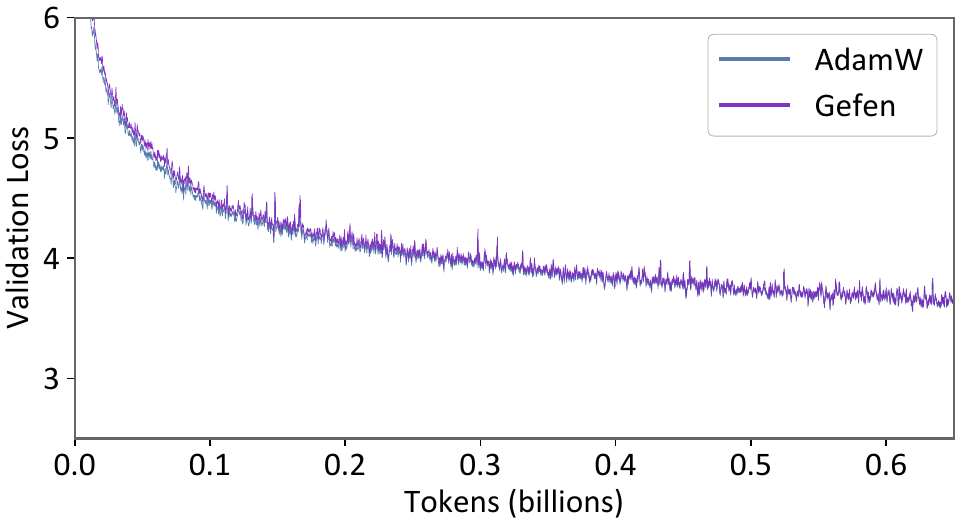}
            };
            \spy[red] on (6.2,1.7) in node at (4.0,2.8);
        \end{tikzpicture}
        }
        
        {\scriptsize (b) Loss: \llamaoneb on C4}
    \end{minipage}
    \hfill
    \begin{minipage}{0.31\linewidth}
        \centering
        \resizebox{\linewidth}{!}{%
        \begin{tikzpicture}[
            spy using outlines={circle,magnification=2.5,size=1.8cm,connect spies}
        ]
            \node[anchor=south west, inner sep=0] at (0,0) {
                \includegraphics[width=7cm]{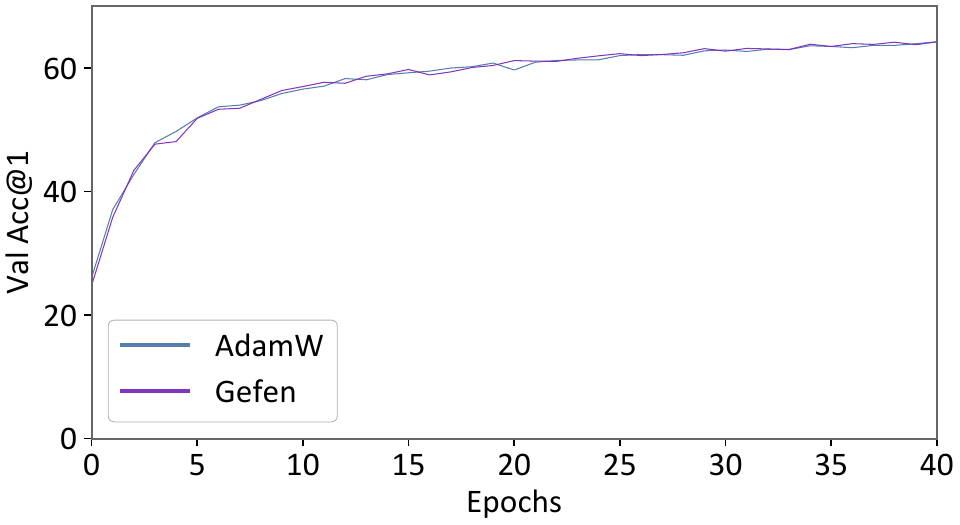}
            };
            \spy[red] on (5.9,3.4) in node at (4.0,1.9);
            
        \end{tikzpicture}
        }
        
        {\scriptsize (c) Acc: ResNet18 on ImageNet}
    \end{minipage}    
    \begin{minipage}{0.31\linewidth}
        \centering
        \resizebox{\linewidth}{!}{%
            \begin{tikzpicture}[
            spy using outlines={circle,magnification=4.5,size=1.2cm,connect spies}
        ]
            \node[anchor=south west, inner sep=0] at (0,0) {
                \includegraphics[width=7cm]{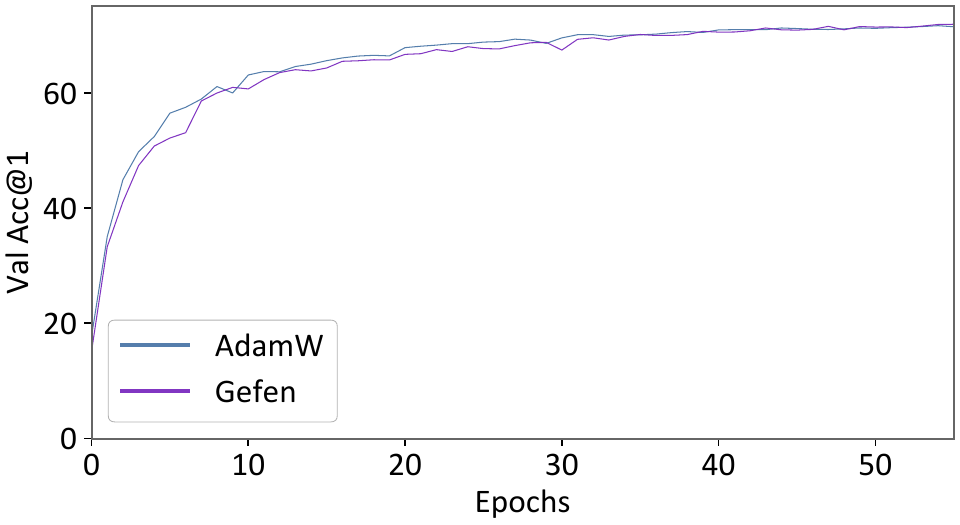}
            };
            \spy[red] on (6.5,3.59) in node at (4.0,1.6);
            
        \end{tikzpicture}
        }
        
        {\scriptsize (d) Acc: ResNet101 on ImageNet}
    \end{minipage}
    \hfill
    \begin{minipage}{0.31\linewidth}
        \centering
        \resizebox{\linewidth}{!}{%
            \begin{tikzpicture}[
            spy using outlines={circle,magnification=4.5,size=1.5cm,connect spies}
        ]
            \node[anchor=south west, inner sep=0] at (0,0) {
                \includegraphics[width=7cm]{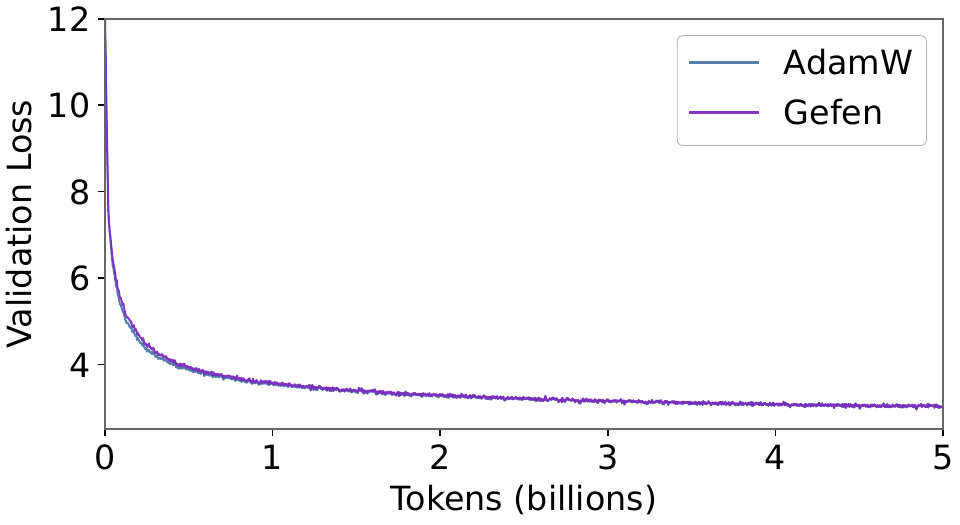}
            };
            \spy[red] on (6.30,0.89) in node at (3.5,2.6);
            
        \end{tikzpicture}
        }        
        {\scriptsize (e) Loss: Qwen-14.3B-MoE on C4}
    \end{minipage}
    \hfill
    \begin{minipage}{0.31\linewidth}
        \centering
        \resizebox{\linewidth}{!}{%
            \begin{tikzpicture}[
            spy using outlines={circle,magnification=4.0,size=1.6cm,connect spies}
        ]
            \node[anchor=south west, inner sep=0] at (0,0) {
                \includegraphics[width=7cm]{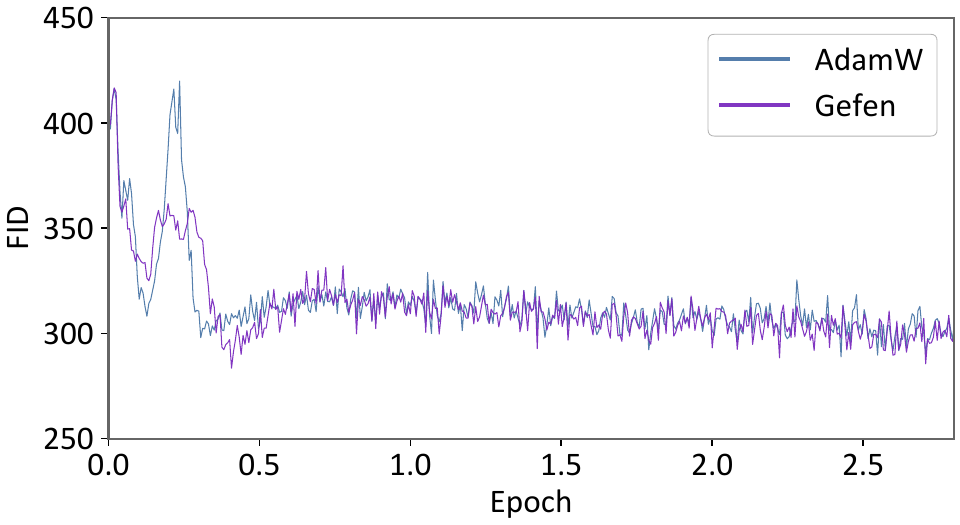}
            };
            \spy[red] on (6.65,1.35) in node at (3.39,2.7);
            
        \end{tikzpicture}
        }
        
        {\scriptsize (f) FID: DiT-XL/2 on ImageNet}
    \end{minipage}    
    \begin{minipage}{0.31\linewidth}
        \centering
        \resizebox{\linewidth}{!}{%
            \begin{tikzpicture}[
            spy using outlines={circle,magnification=2.5,size=1.5cm,connect spies}
        ]
            \node[anchor=south west, inner sep=0] at (0,0) {
                \includegraphics[width=7cm]{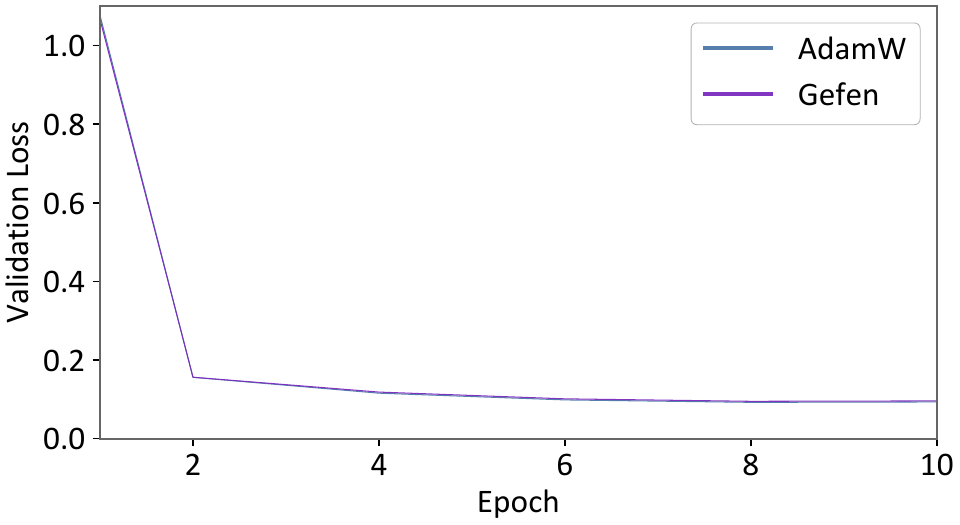}
            };
            \spy[red] on (6.30,1.00) in node at (3.5,2.6);
            
        \end{tikzpicture}
        }        
        {\scriptsize (g) Loss: DDPM (Diffusion)}
    \end{minipage}
    \hfill
    \begin{minipage}{0.31\linewidth}
        \centering
        \resizebox{\linewidth}{!}{%
            \begin{tikzpicture}[
            spy using outlines={circle,magnification=2.5,size=1.5cm,connect spies}
        ]
            \node[anchor=south west, inner sep=0] at (0,0) {
                \includegraphics[width=7cm]{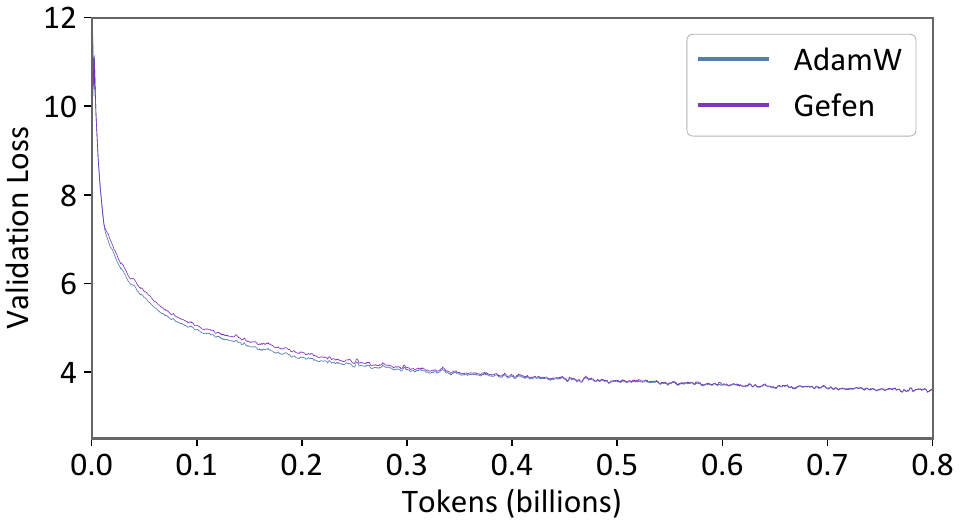}
            };
            \spy[red] on (6.30,1.10) in node at (3.5,2.6);
            
        \end{tikzpicture}
        }
        
        {\scriptsize (h) Loss: Llama 8B on C4}
    \end{minipage}
    \hfill
    \begin{minipage}{0.31\linewidth}
         \centering
        \resizebox{\linewidth}{!}{%
            \begin{tikzpicture}[
            spy using outlines={circle,magnification=2.5,size=1.5cm,connect spies}
        ]
            \node[anchor=south west, inner sep=0] at (0,0) {
                \includegraphics[width=7cm]{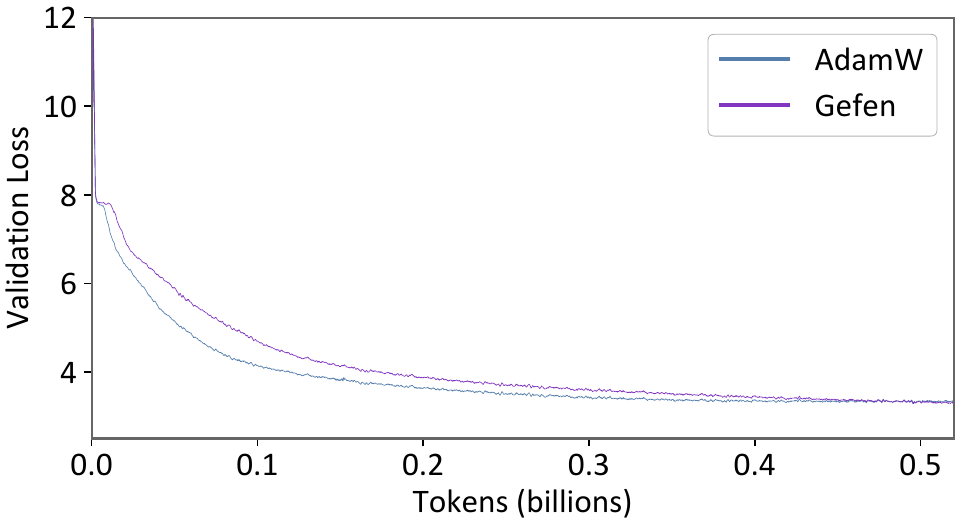}
            };
            \spy[red] on (6.50,1.00) in node at (3.5,2.6);
            
        \end{tikzpicture}
        }
        
        {\scriptsize (i) Loss: Qwen3-4B on FineWeb-Edu}
    \end{minipage}
    \caption{Training curves of \gptsmall on OpenWebText, \llamaoneb and Llama 8B on C4, Qwen1.5-MoE-A2.7B (14.3B total, 2.7B active parameters) on C4, Qwen3-4B on FineWeb-Edu, ResNet18 and ResNet101 on ImageNet, DiT-XL/2 on ImageNet, and DDPM on Bollywood-Celebs. \ourmethod performs on par with AdamW across the settings. Additional graphs can be seen in \cref{app:more_experiments}.
    } 
    \label{fig:gpt2_and_llama_train_curve}    
    \vspace{-0.8em}
\end{figure}

\subsection{Throughput and Distributability}
We next evaluate whether the reduced memory footprint translates into practical throughput gains and whether the method is compatible with distributed training schemes.
Since optimizer-step time is only one component of the full training iteration, we measure end-to-end training throughput in tokens per second under fixed global batch size and sequence length.
In both distributed settings, the lower optimizer-state memory of \ourmethod allows a larger microbatch to fit on each GPU, reducing gradient accumulation overhead and increasing throughput. 
\begin{wrapfigure}{r}{0.40\linewidth}
    \centering
    \includegraphics[width=0.50\linewidth]{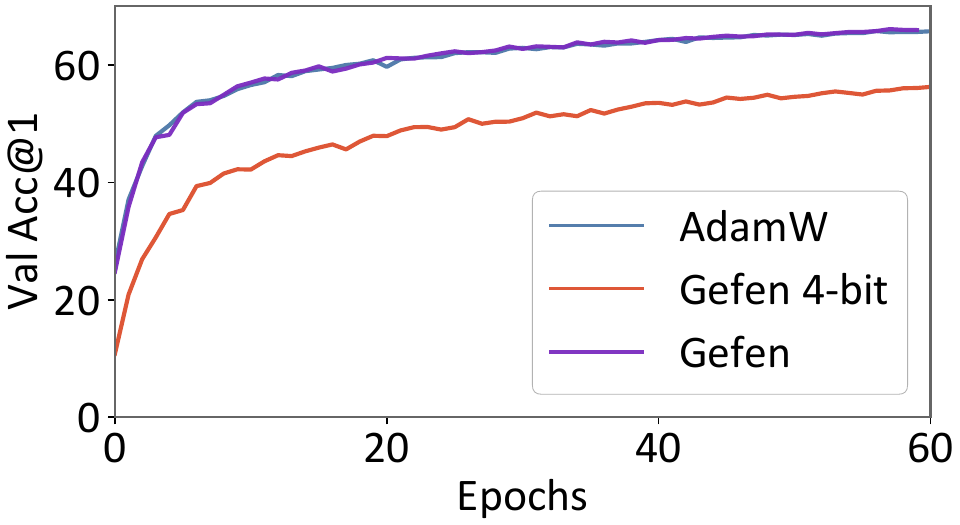}
     \vspace{-1.0em}
    \caption{Four-bit quantization.}
    \label{fig:resnet18_4_bit}
\end{wrapfigure}

As shown in \cref{fig:throughput}, this yields a 56\% throughput improvement over other methods for \llamaoneb with FSDP sharding on C4, by allowing the microbatch size per GPU to increase from 1 to 2, which is not possible with the other methods due to their larger memory footprint. Adam4bit and Adam8bit implementations do not support FSDP sharding \citep{liang2025torchtitan}. 
For the \gptoneb model, under both DDP on two 24GiB GPUs and single-machine training, AdamW is infeasible even with a microbatch size of 1. Adam-mini supports a microbatch size of 1, whereas \ourmethod supports a microbatch size of 2, yielding a 21\% throughput improvement over Adam-mini while making this setting feasible compared with AdamW. We compare optimizer distributability in more detail in \cref{app:distributability}.
\\

\section{Ablation Study}
\label{app:ablations}
\paragraph{Getting Down to Four Bits.}
While taking \ourmethod down to four bits works well in many situations, our goal is a general-purpose optimizer, so we cannot accept performance degradation on some tasks. \Cref{fig:resnet18_4_bit} shows a ResNet18 pretraining example in which reducing quantization to four bits harms performance. The same phenomenon can be seen for other four-bit methods in \cref{fig:additional_experiments}.

\vspace{-0.5em}
\paragraph{Stable Quantization.}
\begin{wrapfigure}{r}{0.55\linewidth}
    \vspace{-1.2em}
    \centering
    \begin{minipage}{0.48\linewidth}
        \centering
        \includegraphics[width=\linewidth]{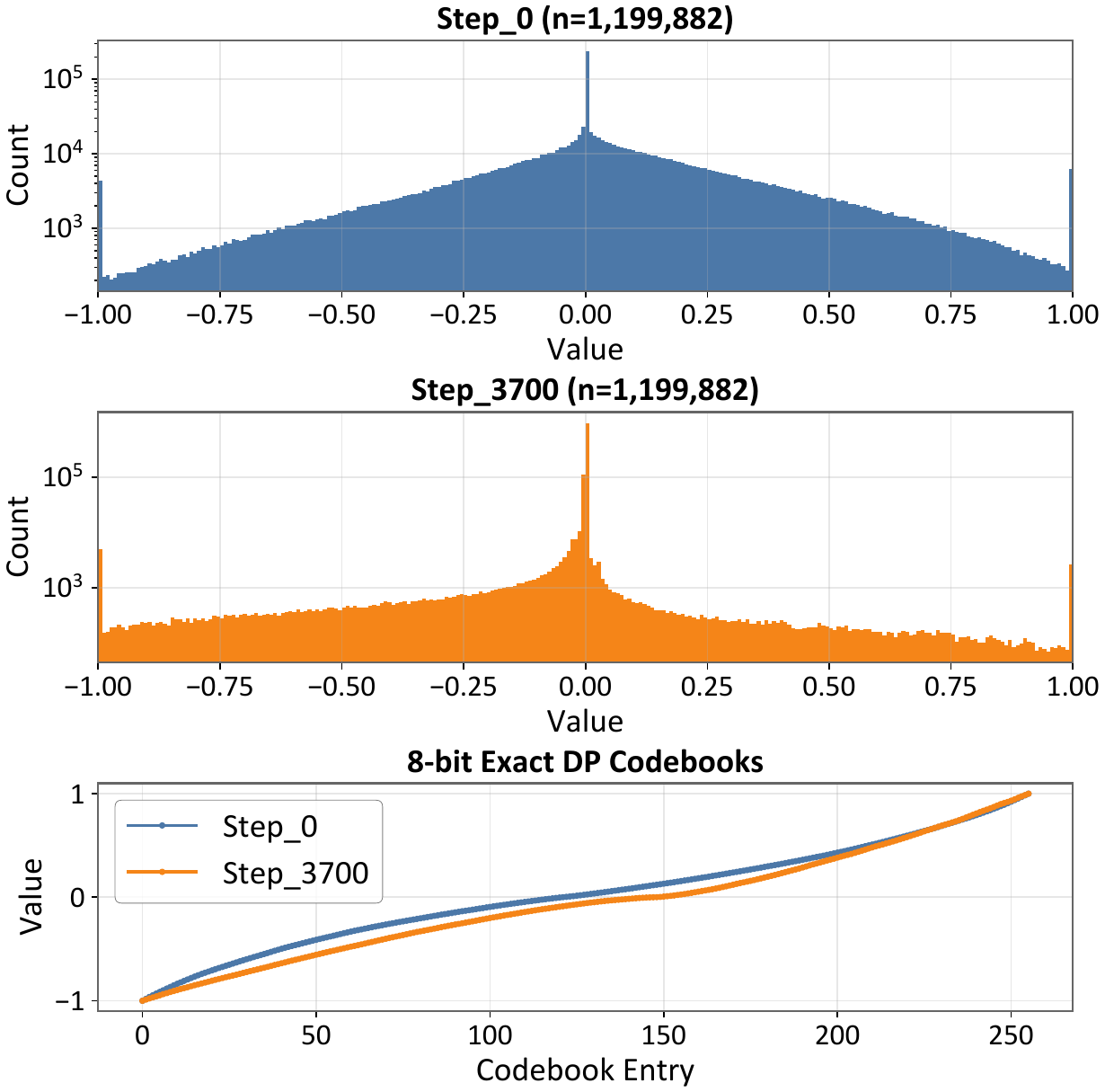}
        {\scriptsize (a) CNN on MNIST.}
    \end{minipage}
    \hfill
    \begin{minipage}{0.48\linewidth}
        \centering
        \includegraphics[width=\linewidth]{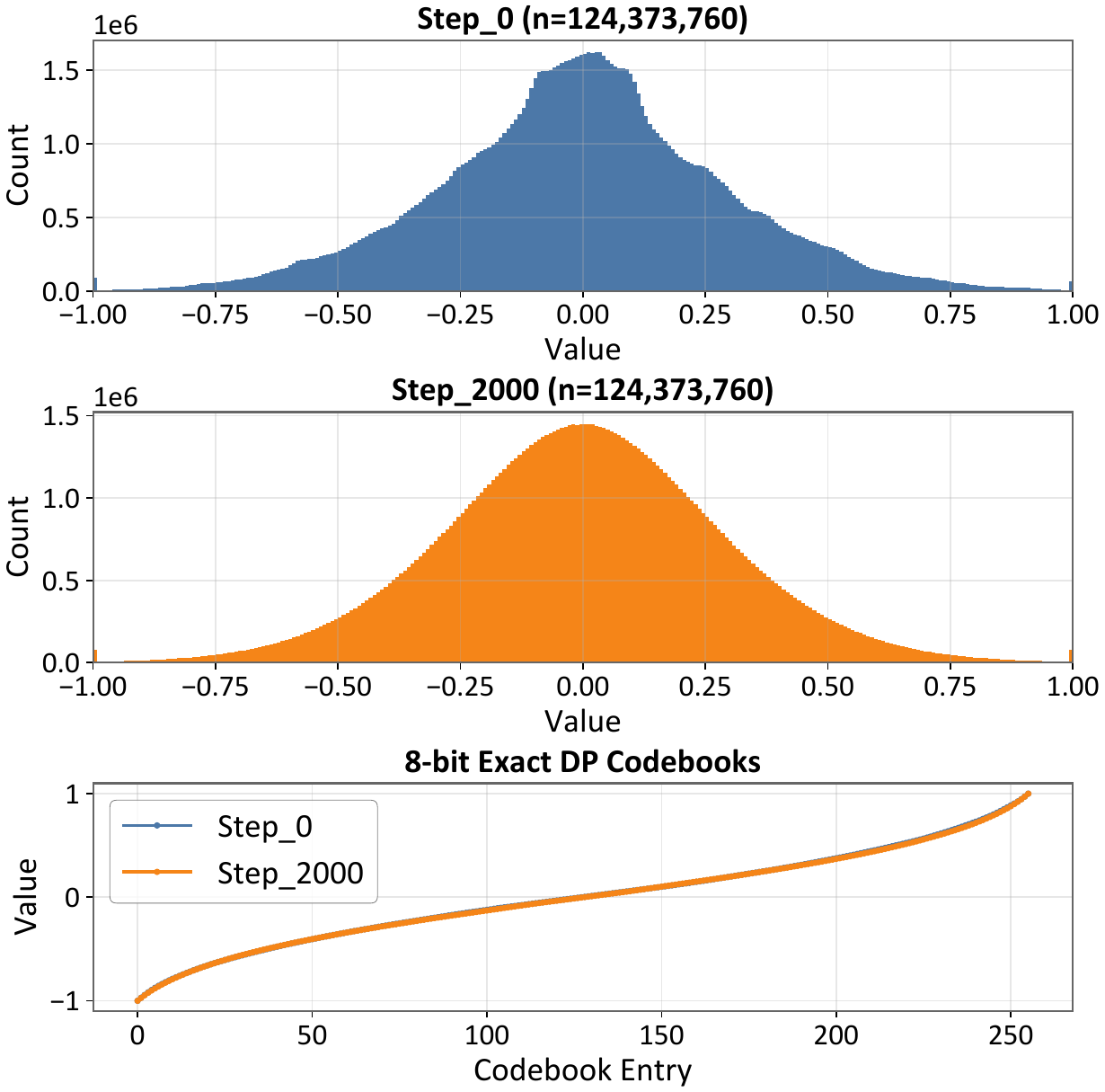}
        {\scriptsize (b) GPT-2 (125M) on OpenWebText.}
    \end{minipage}        
    \caption{Top: Histograms of gradient values after block-wise absmax normalization at different stages of training. Bottom: Corresponding learned codebooks at the same stages. The codebooks are relatively stable throughout training, allowing us to learn the codebook once. 
    }
    \label{fig:stable_quantization}
    \vspace{-3.9em}
\end{wrapfigure}

A natural question for the ablation study is whether it is sufficient to learn
the quantization codebook only at the beginning of training. \Cref{fig:stable_quantization}
shows the gradient distributions at different stages of training and their
corresponding codebooks.
The codebooks are relatively stable, suggesting that they can be reused throughout training.
See \cref{app:extended_ablation_study} for more ablations.

\section{Conclusion}
In this work, we presented \ourmethod, an optimizer that significantly reduces AdamW's second-moment memory by sharing second-moment estimates in parameter regions where the Hessian, or the squared-gradient proxy used by our method, exhibits periodic structure; empirically, this structure appears across the pretraining tasks we considered.
We first formulated desiderata for practical memory-efficient optimizers, and then provided theoretical results showing that strong Hessian coupling makes squared gradients, and hence AdamW second moments, close to each other, motivating shared second-moment estimates. Since computing Hessians is impractical at scale, \ourmethod uses first-step gradients to expose useful information about network topology without requiring user-provided architecture metadata or additional hyperparameters.
It detects periodicity, groups parameters accordingly, shares second moments within the resulting blocks, and reuses the same structure to quantize first moments using an exact dynamic-programming algorithm for codebook construction.
Among the baselines that reduce AdamW's peak memory footprint by \(2\times\) or more, \ourmethod matches AdamW's performance, requires no learning-rate adjustment or additional hyperparameters, and achieves low memory footprint (up to \(8\times\) lower than AdamW) and high throughput. In practice, switching from AdamW to \ourmethod requires changing only a single line of code, since the optimizer's arguments are fully compatible. \ourmethod can also be combined with other optimization methods, as discussed in \cref{app:extensions}.

\subsection*{Reproducibility statement}

We provide the complete implementation including fused CUDA kernels. Algorithmic details are given in \cref{sec:algorithm,app:dp_details}, proofs in \cref{app:theory-proofs}, additional experimental analysis in \crefrange{app:empirical-evidence}{app:more_experiments}, and complete experimental settings and hyperparameters in \cref{app:hyperparameters}.

\subsection*{AI use statement}

In this work, artificial intelligence tools were used for grammar correction, identifying formatting and style issues, rephrasing text, validating and cleaning code, and implementing specific software functions. All artifacts were reviewed and verified by the authors. We take responsibility for the final content of this work.

\bibliography{iclr2027_conference}
\bibliographystyle{iclr2027_conference}

\clearpage
\setcounter{tocdepth}{3}
\tableofcontents

\clearpage
\appendix

\section{Proof of the Main Theoretical Result}
\label[appendix]{app:theory-proofs}
\label[appendix]{app:theory-proof}
\begin{theorem}[Large Hessian entries contract squared-gradient ratios]
Consider the two-layer MLP
$z=W_1x+b_1$, $y=W_2\sigma(z)+b_2$, where $y\in\mathbb{R}^{d_{\mathrm{out}}}$,
with MSE loss
$L=d_{\mathrm{out}}^{-1}\sum_{j=1}^{d_{\mathrm{out}}} (y_j-t_j)^2$.
Let $e_j:=y_j-t_j$.
Define:
\[
\begin{aligned}
&A_k:=\sum_{i=1}^{d_{\mathrm{out}}} e_i W_2[i,k],\quad
U_{k,l}:=|x_l\sigma'(z_k)|,
&C_{k,k'}:=\left|\frac{A_k}{A_{k'}}\right|,\quad
D_{k,k'}:=\left|\sum_i W_2[i,k]W_2[i,k']\right|
\end{aligned}
\]
Assume:
\[
\begin{aligned}
&D_{k,k'}\le \alpha,\quad
U_{k,l}\le \beta,\quad
|\log C_{k,k'}|\le \gamma,
&\forall k,l,k',l'\ \text{s.t.}\
k\neq k',\ A_k,A_{k'}\neq 0,\ H_{(k,l),(k',l')}\neq 0
\end{aligned}
\]
For any such pair of entries $W_1[k,l]$ and $W_1[k',l']$,
let $g_{k,l}:=\partial L/\partial W_1[k,l]$ and define the ratio between
their squared gradients and their mixed Hessian entry as
\[
R:=\frac{g_{k,l}^2}{g_{k',l'}^2},
\qquad
H_{(k,l),(k',l')}
:=\frac{\partial^2 L}{\partial W_1[k,l]\,\partial W_1[k',l']}
\]
Then, for $A:=2\alpha\beta^2/d_{\mathrm{out}}$,
\[
e^{-2\gamma}\left(\frac{|H_{(k,l),(k',l')}|}{A}\right)^2
\le
R
\le
e^{2\gamma}\left(\frac{A}{|H_{(k,l),(k',l')}|}\right)^2
\]

\end{theorem}

\begin{proof}
Consider a simple two-layer MLP with input vector $x$, and output vector $y$:
\begin{align}
z &= W_1 x + b_1 \qquad z\in \mathbb{R}^h \\
y &= W_2 \sigma(z) + b_2 \qquad y\in \mathbb{R}^{d_{\mathrm{out}}}
\end{align}
with output dimension $d_{\mathrm{out}}$, and MSE loss
\begin{equation}
L = \frac{1}{d_{\mathrm{out}}}\sum_{j=1}^{d_{\mathrm{out}}} (y_j - t_j)^2
= \frac{1}{d_{\mathrm{out}}}\sum_{j=1}^{d_{\mathrm{out}}} e_j^2
\end{equation}
Here $t\in\mathbb{R}^{d_{\mathrm{out}}}$ is the target vector, and the error vector is $e_j := y_j - t_j$.

\begin{equation}
y_i=\sum_{k} W_2[i,k] \sigma\left(\underbrace{\sum_{l} W_1[k,l] x_l + b_1[k]}_{=\,z_k}\right) + b_2[i] 
\label{eq:w1-gradient-components}
\end{equation}

\begin{equation}
\frac{\partial L}{\partial y_i} = \frac{2}{d_{\mathrm{out}}}e_i
\qquad
\frac{\partial y_i}{\partial W_1[k,l]}
= W_2[i,k]\sigma'(z_k)x_l
\end{equation}

Define:
\begin{equation}
    A_k := \sum_{i=1}^{d_{\mathrm{out}}} e_i W_2[i,k]
\end{equation}

Hence the gradient entry w.r.t. $W_1[k,l]$ is

\begin{equation}
g_{k,l}
:= \frac{\partial L}{\partial W_1[k,l]}
= \frac{2}{d_{\mathrm{out}}}x_l \sigma'(z_k)\sum_{i=1}^{d_{\mathrm{out}}} e_i W_2[i,k] = \frac{2}{d_{\mathrm{out}}}x_l \sigma'(z_k)A_k
\label{eq:w1-gradient}
\end{equation}
Both $\sigma'(z_k)$ and $e_i$ depend on $W_1[k,l]$.

The Hessian w.r.t. $W_1[k,l]$ and $W_1[k',l']$ is:

\begin{align}
H_{(k,l),(k',l')}:= \frac{\partial^2 L}{\partial W_1[k,l]\,\partial W_1[k',l']}
&= \frac{\partial g_{k,l}}{\partial W_1[k',l']} \nonumber\\
&= \frac{2}{d_{\mathrm{out}}}x_l
\left(
\frac{\partial \sigma'(z_k)}{\partial W_1[k',l']}A_k
\;+\;
\sigma'(z_k)\frac{\partial A_k}{\partial W_1[k',l']}
\right) \label{eq:hess-product}
\end{align}
Now compute each derivative:
\begin{align}
\frac{\partial \sigma'(z_k)}{\partial W_1[k',l']}
&= \sigma''(z_k)\frac{\partial z_k}{\partial W_1[k',l']}
= \sigma''(z_k)\delta_{k,k'}x_{l'} \\
\frac{\partial A_k}{\partial W_1[k',l']}
&= \sum_{i=1}^{d_{\mathrm{out}}} W_2[i,k]\frac{\partial e_i}{\partial W_1[k',l']} \\
\frac{\partial e_i}{\partial W_1[k',l']}
&= \frac{\partial y_i}{\partial W_1[k',l']}
= W_2[i,k']\sigma'(z_{k'})x_{l'}
\end{align}
So
\begin{equation}
\frac{\partial A_k}{\partial W_1[k',l']}
= x_{l'}\sigma'(z_{k'})\sum_{i=1}^{d_{\mathrm{out}}} W_2[i,k]W_2[i,k']
\end{equation}
Substitute into \cref{eq:hess-product} to get the Hessian:
\begin{align}
H_{(k,l),(k',l')}
&= \frac{2}{d_{\mathrm{out}}}x_l x_{l'} \Bigg[
\delta_{k,k'}\,\sigma''(z_k)\sum_{i=1}^{d_{\mathrm{out}}} e_i W_2[i,k] \nonumber\\
&\qquad\qquad\qquad
+ \sigma'(z_k)\sigma'(z_{k'})
\sum_{i=1}^{d_{\mathrm{out}}} W_2[i,k]W_2[i,k']
\Bigg]
\label{eq:w1-hessian-components}
\end{align}
Here $\delta_{k,k'}$ is the Kronecker delta.
This is the second derivative between two \emph{different} entries of $W_1$. For $k\neq k'$ (so the $\delta_{k,k'}$ term is zero), define the indexed quantities
\begin{equation}
U_{k,l}:=|x_l\sigma'(z_k)|,\quad
C_{k,k'}:=\left|\frac{A_k}{A_{k'}}\right|,\quad
D_{k,k'}:=\left|\sum_i W_2[i,k]W_2[i,k']\right|.
\end{equation}
Then from \cref{eq:w1-gradient,eq:w1-hessian-components} we have:
\begin{equation}
R:=\frac{g_{k,l}^2}{g_{k',l'}^2}
=\left(\frac{U_{k,l}}{U_{k',l'}}\right)^2 C_{k,k'}^2 \label{eq:rdef}
\end{equation}
\begin{equation}
|H_{(k,l),(k',l')}|
=\frac{2}{d_{\mathrm{out}}}\,U_{k,l}U_{k',l'}D_{k,k'} \label{eq:abs-h-uvd}
\end{equation}
Assume:
\begin{equation}
\begin{aligned}
&D_{k,k'}\le \alpha,\qquad
U_{k,l}\le \beta,\qquad
|\log C_{k,k'}|\le \gamma,\\
&\forall k,l,k',l'\ \text{s.t.}\
k\neq k',\ A_k,A_{k'}\neq 0,\ H_{(k,l),(k',l')}\neq 0
\end{aligned}
\label{eq:bounds}
\end{equation}
Now fix any such pair. From \cref{eq:abs-h-uvd} and
$D_{k,k'}\le \alpha$, we get
\begin{equation}
U_{k,l}U_{k',l'}\ge \frac{|H_{(k,l),(k',l')}|\,d_{\mathrm{out}}}{2\alpha}
\end{equation}
Since the uniform bound applies to both index pairs, we also have
$U_{k,l},U_{k',l'}\le \beta$, so:
\begin{equation}
 \begin{aligned}
&U_{k,l},U_{k',l'}\in
\left[\frac{|H_{(k,l),(k',l')}|\,d_{\mathrm{out}}}{2\alpha\beta},\,\beta\right]
\\
&\overset{\text{log arithmetic}}{\Longrightarrow}\quad
\left|\log\frac{U_{k,l}}{U_{k',l'}}\right|
\le
\log\!\left(\frac{2\alpha\beta^2}{|H_{(k,l),(k',l')}|\,d_{\mathrm{out}}}\right)
\;=:\;L_{k,l,k',l'} \label{eq:logbound}
 \end{aligned}
\end{equation}
From \cref{eq:rdef,eq:bounds,eq:logbound} we get the following bounds on $R$:
\begin{equation}
e^{-2(L_{k,l,k',l'}+\gamma)}
\le
R
\le
e^{2(L_{k,l,k',l'}+\gamma)}
\end{equation}
Equivalently, define
\begin{equation}
A:=\frac{2\alpha\beta^2}{d_{\mathrm{out}}}
\end{equation}
which upper bounds the Hessian magnitude under the stated assumptions:
\begin{equation}
|H_{(k,l),(k',l')}|=\frac{2}{d_{\mathrm{out}}}U_{k,l}U_{k',l'}D_{k,k'}
\le
\frac{2}{d_{\mathrm{out}}}\alpha\beta^2
=A.
\end{equation}
Since
$L_{k,l,k',l'}=\log\!\left(\frac{2\alpha\beta^2}{|H_{(k,l),(k',l')}|\,d_{\mathrm{out}}}\right)
=\log\!\left(\frac{A}{|H_{(k,l),(k',l')}|}\right)$:
\begin{equation}
\exp\!\left(
-2\left[\log\!\left(\frac{A}{|H_{(k,l),(k',l')}|}\right)+\gamma\right]
\right)
\le
R
\le
\exp\!\left(
2\left[\log\!\left(\frac{A}{|H_{(k,l),(k',l')}|}\right)+\gamma\right]
\right)
\end{equation}
That is,
\begin{equation}
e^{-2\gamma}\left(\frac{|H_{(k,l),(k',l')}|}{A}\right)^2
\le
R
\le
e^{2\gamma}\left(\frac{A}{|H_{(k,l),(k',l')}|}\right)^2
\end{equation}
Thus, as $|H_{(k,l),(k',l')}|$ increases within its feasible range
$0<|H_{(k,l),(k',l')}|\le A$, it
raises the lower bound and lowers the upper bound.
By Hessian symmetry, swapping $(k,l)$ and $(k',l')$ leaves
$|H_{(k,l),(k',l')}|$
unchanged and replaces $R$ by $1/R$; therefore the admissible interval is
reciprocal-symmetric around $1$.
Hence the admissible interval for $R$ shrinks to the band
$[e^{-2\gamma},e^{2\gamma}]$ (and to $1$ if $\gamma$ is small).
Empirically we see that the same phenomenon happens even when $k=k'$.

\paragraph{The Same Argument for $\mW_2$}
The second-layer weights give an even simpler case. For
$W_2[i,k]$, using \cref{eq:w1-gradient-components}, since $z$ does not
depend on $W_2$,
\begin{equation}
g_{i,k}^{(2)}
:= \frac{\partial L}{\partial W_2[i,k]}
= \frac{2}{d_{\mathrm{out}}}e_i\sigma(z_k)
\end{equation}
For two entries $W_2[i,k]$ and $W_2[i',k']$,
\begin{equation}
H^{(2)}_{(i,k),(i',k')}
:=
\frac{\partial^2 L}
{\partial W_2[i,k]\,\partial W_2[i',k']}
=
\frac{2}{d_{\mathrm{out}}}\delta_{i,i'}\sigma(z_k)\sigma(z_{k'})
\end{equation}
Thus if $i\neq i'$, the mixed Hessian is zero. If $i=i'$ and we define
\begin{equation}
U^{(2)}_k:=|\sigma(z_k)|
\end{equation}
then
\begin{equation}
R^{(2)}
:=
\frac{(g_{i,k}^{(2)})^2}{(g_{i,k'}^{(2)})^2}
=
\left(\frac{U^{(2)}_k}{U^{(2)}_{k'}}\right)^2
\qquad
|H^{(2)}_{(i,k),(i,k')}|
=
\frac{2}{d_{\mathrm{out}}}U^{(2)}_kU^{(2)}_{k'}
\end{equation}
Assuming $U^{(2)}_k\le \beta^{(2)}$ for all $k$, the same
log-arithmetic gives
\begin{equation}
\left|\log\frac{U^{(2)}_k}{U^{(2)}_{k'}}\right|
\le
\log\!\left(\frac{2(\beta^{(2)})^2}{|H^{(2)}_{(i,k),(i,k')}|\,d_{\mathrm{out}}}\right)
\end{equation}
Equivalently, with
$A^{(2)}:=2(\beta^{(2)})^2/d_{\mathrm{out}}$,
\begin{equation}
\left(\frac{|H^{(2)}_{(i,k),(i,k')}|}{A^{(2)}}\right)^2
\le
R^{(2)}
\le
\left(\frac{A^{(2)}}{|H^{(2)}_{(i,k),(i,k')}|}\right)^2
\end{equation}
Therefore, for second-layer weights that share the same output coordinate,
larger Hessian magnitude again forces the squared-gradient ratio toward
$1$.

\paragraph{Bias Terms}
The hidden bias $b_1[k]$ behaves like the first-layer weight
$W_1[k,l]$ with the input factor $x_l$ removed. From \cref{eq:w1-gradient-components}, its gradient is
\begin{equation}
g^{(b_1)}_k
:=
\frac{\partial L}{\partial b_1[k]}
=
\frac{2}{d_{\mathrm{out}}}\sigma'(z_k)A_k
\end{equation}
and the mixed Hessian between $b_1[k]$ and $b_1[k']$ is
\begin{equation}
H^{(b_1)}_{k,k'}
:=
\frac{\partial^2 L}{\partial b_1[k]\,\partial b_1[k']}
=
\frac{2}{d_{\mathrm{out}}}
\left[
\delta_{k,k'}\sigma''(z_k)A_k
+
\sigma'(z_k)\sigma'(z_{k'})
\sum_i W_2[i,k]W_2[i,k']
\right]
\end{equation}
Thus, for $k\neq k'$, the $\delta_{k,k'}$ term disappears and the same
argument applies with
\begin{equation}
U^{(b_1)}_k:=|\sigma'(z_k)|,\qquad
C_{k,k'}:=\left|\frac{A_k}{A_{k'}}\right|,\qquad
D_{k,k'}:=\left|\sum_i W_2[i,k]W_2[i,k']\right|
\end{equation}
In this case,
\begin{equation}
R^{(b_1)}
:=
\frac{(g^{(b_1)}_k)^2}{(g^{(b_1)}_{k'})^2}
=
\left(\frac{U^{(b_1)}_k}{U^{(b_1)}_{k'}}\right)^2C_{k,k'}^2,
\qquad
|H^{(b_1)}_{k,k'}|
=
\frac{2}{d_{\mathrm{out}}}U^{(b_1)}_kU^{(b_1)}_{k'}D_{k,k'}
\end{equation}
so the same conditional bound as for $W_1$ follows.

For the output bias $b_2$, the Hessian is either zero or constant, so it
is irrelevant to our claim:
\begin{equation}
g^{(b_2)}_i
:=
\frac{\partial L}{\partial b_2[i]}
=
\frac{2}{d_{\mathrm{out}}}e_i
\end{equation}
and
\begin{equation}
H^{(b_2)}_{i,i'}
:=
\frac{\partial^2 L}{\partial b_2[i]\,\partial b_2[i']}
=
\frac{2}{d_{\mathrm{out}}}\delta_{i,i'}
\end{equation}
\end{proof}

\label[appendix]{app:gauss-newton-proof}
\begin{theorem}[Within-layer Hessian affinity controls gradient magnitudes]
Let \(f_\theta(x)\in\mathbb{R}^m\) be a feedforward ReLU network, and let \(L(\theta)=\ell(f_\theta(x))\), where \(\ell:\mathbb{R}^m\to\mathbb{R}\) is twice differentiable and convex.
Write \(g=\nabla_\theta L(\theta)\), \(H=\nabla_\theta^2L(\theta)\), and \(z=f_\theta(x)\), and assume that $\nabla \ell(z)$ is in the column span of $\nabla^2 \ell(z)$.
Further assume that \(\theta\) is away from ReLU activation boundaries for this input, i.e., all preactivations are nonzero.
Then for any two parameters \(i,j\) belonging to the same layer, we have \(H_{ii} \geq 0, H_{jj} \geq 0\), and for the normalized Hessian affinity
\(
\rho_{ij} = {|H_{ij}|}/{\sqrt{H_{ii}H_{jj}}}
\)
we have
\[
\left|
\frac{g_i^2}{H_{ii}}
-
\frac{g_j^2}{H_{jj}}
\right|
\le
2\lambda(z)^2\sqrt{2(1-\rho_{ij})},
\]
where $\lambda(z)$ is the Newton decrement at $z$, namely
\(
\lambda(z)^2
=
\nabla\ell(z)^\top
\bigl(\nabla^2\ell(z)\bigr)^\dagger
\nabla\ell(z)
.
\)
Here, we use the convention that $g_i^2/H_{ii}=0$ whenever $H_{ii}=0$, and similarly $\rho_{ij} = 0$ whenever either $H_{ii}=0$ or $H_{jj}=0$.

\end{theorem}

\begin{proof}[Proof of \cref{thm:gauss-newton}]
Let \(J_a:=\partial_a f_\theta(x)\in\mathbb{R}^m\), \(q:=\nabla\ell(z)\), and \(S:=\nabla^2\ell(z)\). For a composite loss \(L(\theta)=\ell(f_\theta(x))\), the Hessian entries satisfy
\[
H_{ab}=J_a^\top S J_b+\langle q,\partial_{ab}^2 f_\theta(x)\rangle .
\]

Since the network uses ReLU activations and \(\theta\) is away from activation boundaries, all ReLU masks are locally constant in a neighborhood of \(\theta\). Hence, locally, the network is obtained by composing affine maps and fixed diagonal activation masks. Holding all other layers fixed, \(f_\theta(x)\) is affine in the parameters of any single affine layer. Therefore, for two parameters \(a,b\) belonging to the same affine layer, \(\partial_{ab}^2 f_\theta(x)=0\). In particular, for \(a,b\in\{i,j\}\), we have \(H_{ab}=J_a^\top S J_b\). Since \(S\succeq 0\), this implies \(H_{aa} \geq 0\) and \(0\le \rho_{ij}\le 1\).
Note that if \(H_{aa}=0\), then \(S^{1/2}J_a=0\), and since
\(q\in\operatorname{Range}(S)\), it follows that \(g_a=\langle q , J_a\rangle=0\).
Thus, this is consistent with our convention that $g_a^2/H_{aa} = 0$ for zero-diagonal coordinates.

Define \(u_i:=S^{1/2}J_i/\sqrt{H_{ii}}\) and \(u_j:=S^{1/2}J_j/\sqrt{H_{jj}}\). Since \(H_{ii}=J_i^\top S J_i\) and \(H_{jj}=J_j^\top S J_j\), we have \(\|u_i\|=\|u_j\|=1\). Moreover,
\[
|\langle u_i,u_j\rangle|
=
\frac{|J_i^\top S J_j|}{\sqrt{H_{ii}H_{jj}}}
=
\frac{|H_{ij}|}{\sqrt{H_{ii}H_{jj}}}
=
\rho_{ij}.
\]
Let \(r:=(S^{\dagger})^{1/2}q\). Since \(q\in\operatorname{Range}(S)\), we have \(q=S^{1/2}r\), and \(\|r\|=\lambda(z)\). Therefore,
\[
g_i=\partial_iL(\theta)=\langle q,J_i\rangle
=\langle r,S^{1/2}J_i\rangle
=\sqrt{H_{ii}}\langle r,u_i\rangle,
\]
and similarly \(g_j=\sqrt{H_{jj}}\langle r,u_j\rangle\).

Choose \(s\in\{-1,1\}\) such that \(s\langle u_i,u_j\rangle=|\langle u_i,u_j\rangle|\). Then
\[
\|u_i-su_j\|^2=2-2|\langle u_i,u_j\rangle|=2(1-\rho_{ij}).
\]
Using \(\big||\xi|-|\eta|\big|\le |\xi-s\eta|\), valid for any \(\xi,\eta\in\mathbb{R}\) and \(s\in\{-1,1\}\), we get
\[
\begin{aligned}
\left|
\frac{|g_i|}{\sqrt{H_{ii}}}
-
\frac{|g_j|}{\sqrt{H_{jj}}}
\right|
&=
\left|
|\langle r,u_i\rangle|-|\langle r,u_j\rangle|
\right| \\
&\le
|\langle r,u_i-su_j\rangle| \\
&\le
\|r\|\,\|u_i-su_j\| \\
&=
\lambda(z)\sqrt{2(1-\rho_{ij})}.
\end{aligned}
\]
Note that \(|g_i|/\sqrt{H_{ii}}\le \lambda(z)\) and \(|g_j|/\sqrt{H_{jj}}\le \lambda(z)\). Hence
\[
\begin{aligned}
\left|
\frac{g_i^2}{H_{ii}}
-
\frac{g_j^2}{H_{jj}}
\right|
&\le
\left(
\frac{|g_i|}{\sqrt{H_{ii}}}
+
\frac{|g_j|}{\sqrt{H_{jj}}}
\right)
\left|
\frac{|g_i|}{\sqrt{H_{ii}}}
-
\frac{|g_j|}{\sqrt{H_{jj}}}
\right|  \\
&\le
2\lambda(z)^2\sqrt{2(1-\rho_{ij})}.
\end{aligned}
\]
The proof is complete.
\end{proof}

A possible convergence analysis for our method may be developed by adapting existing analyses of Adam-style adaptive optimizers. The original convergence argument for Adam \citep{kingma2015adam} was later shown to be incomplete by \citet{reddi2018convergence}, which introduced AMSGrad and established convergence guarantees under standard assumptions. More recently, convergence analyses have also been developed for quantized adaptive optimizers such as Adam4bit \citep{li2023fourbit}. The theorem above shows that, under the assumption that the ratios between squared gradients within each shared block are bounded, the shared second-moment estimate remains within a bounded factor of the corresponding per-parameter estimates. This suggests that convergence analyses developed for adaptive optimizers may potentially be extended to the shared-second-moment setting. Establishing such guarantees rigorously is left for future work. Since convergence analyses of adaptive optimizers typically rely on assumptions that may not hold in modern deep-learning settings, we focus in this paper on empirical validation. The theoretical result is intended primarily as motivation for sharing second moments among Hessian-aligned parameters.

\section{Empirical Evidence That High Hessian Implies Close Squared Gradients}

\label[appendix]{app:empirical-evidence}

We trained small models and examined the relationship between the squared
gradients of selected weight entries within several parameter tensors and the
corresponding Hessian entries. The models are intentionally small, since
computing the Hessian for large models is difficult, but they all tell the same
empirical story. This matches the theory in
\cref{thm:hessian-gradient-ratio}: when the Hessian magnitude between two
weights is high, the ratio between their squared gradients tends to be close to
$1$.

In \cref{fig:gpt2_hessian_sqgrad_ratio}, as the Hessian entry moves away from
zero in either direction, the ratio of the squared gradients, shown on a log
scale, approaches one, marked by the horizontal dashed line. The same phenomenon
appears both when the two weights share the same first dimension and when they
do not.
\begin{figure}[tbp]
    \centering
    \begin{minipage}[t]{0.48\linewidth}
        \centering
        \begin{minipage}{0.48\linewidth}
            \centering
            \includegraphics[width=\linewidth]{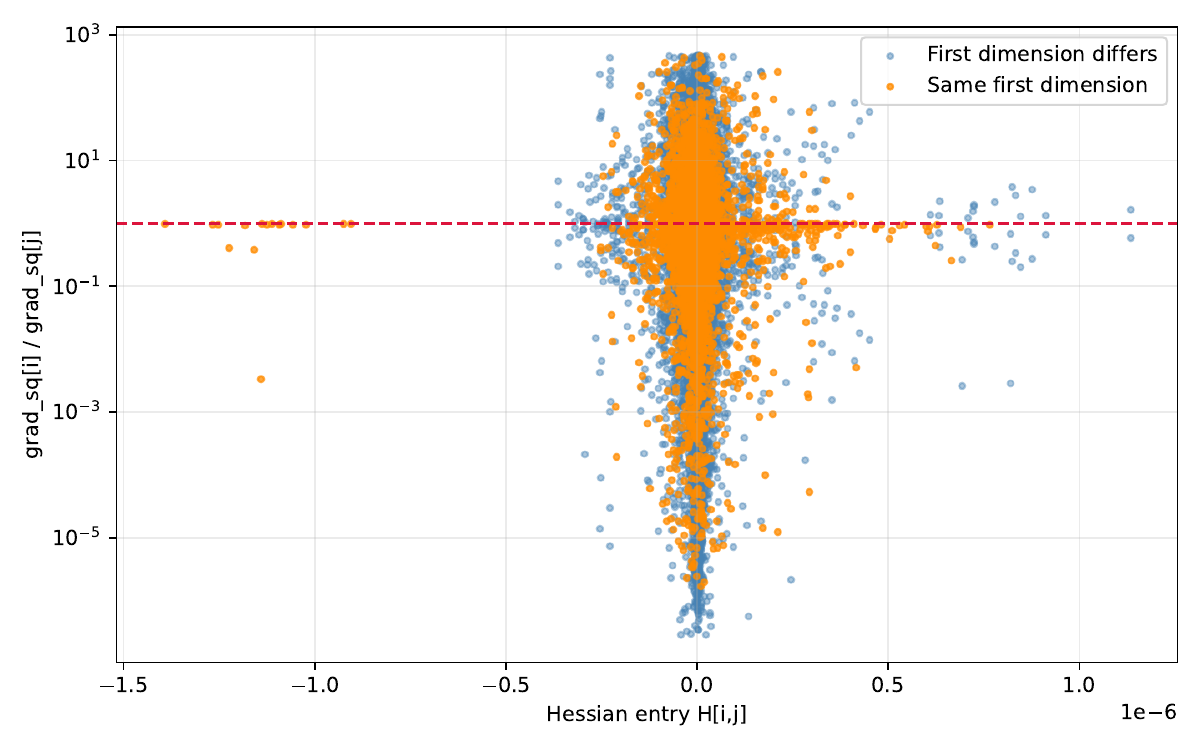}
            \textbf{(a)}
        \end{minipage}
        \hfill
        \begin{minipage}{0.48\linewidth}
            \centering
            \includegraphics[width=\linewidth]{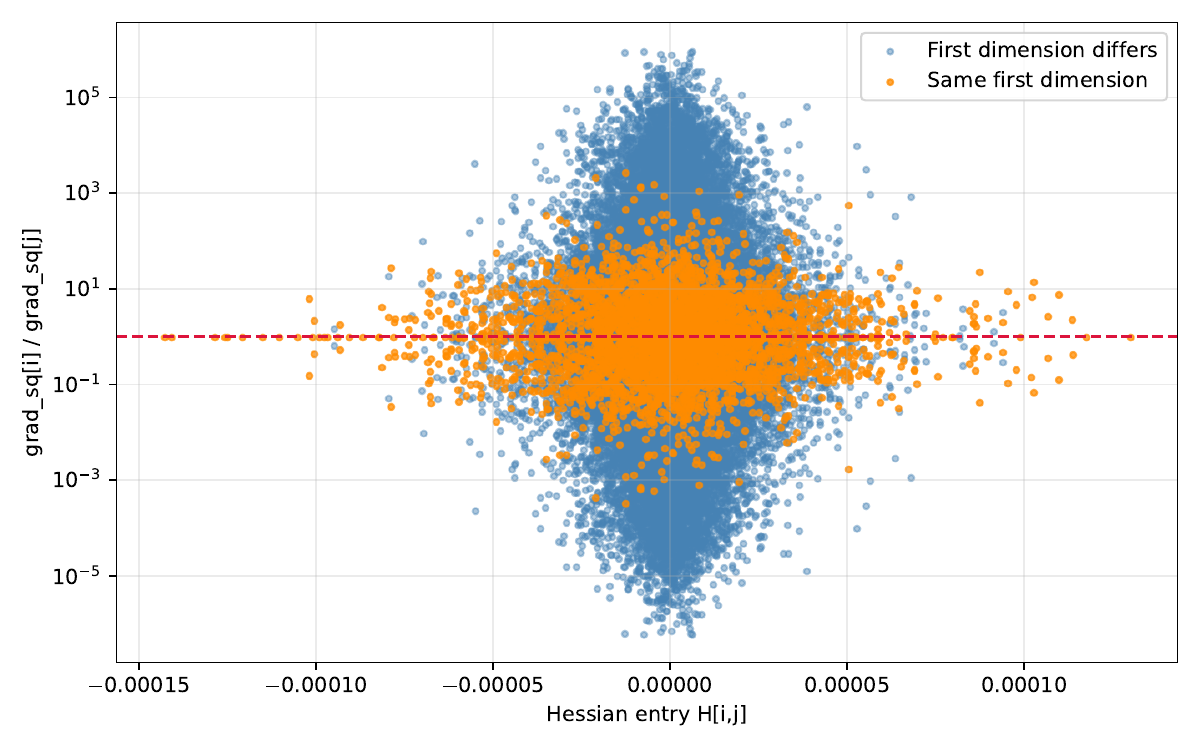}
            \textbf{(b)}
        \end{minipage}
        \caption{Squared-gradient ratio on a log scale vs. Hessian for GPT-toy at the beginning of
        training: (a) first-layer query matrix; (b) fourth-layer value matrix. As the Hessian entry moves away from zero, the squared-gradient ratio approaches one, as predicted by the theorem.}
        \label{fig:gpt2_hessian_sqgrad_ratio}
    \end{minipage}
    \hfill
    \begin{minipage}[t]{0.48\linewidth}
        \centering
        \begin{minipage}{0.48\linewidth}
            \centering
            \includegraphics[width=\linewidth]{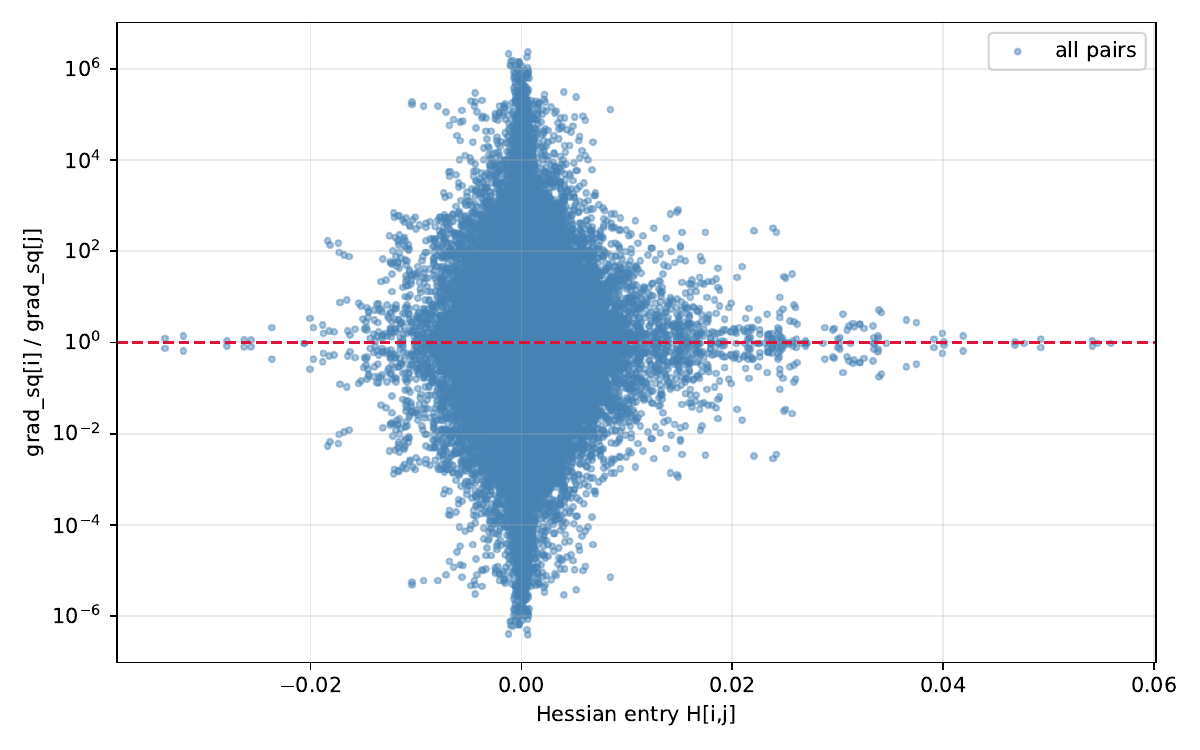}
            \textbf{(a)}
        \end{minipage}
        \hfill
        \begin{minipage}{0.48\linewidth}
            \centering
            \includegraphics[width=\linewidth]{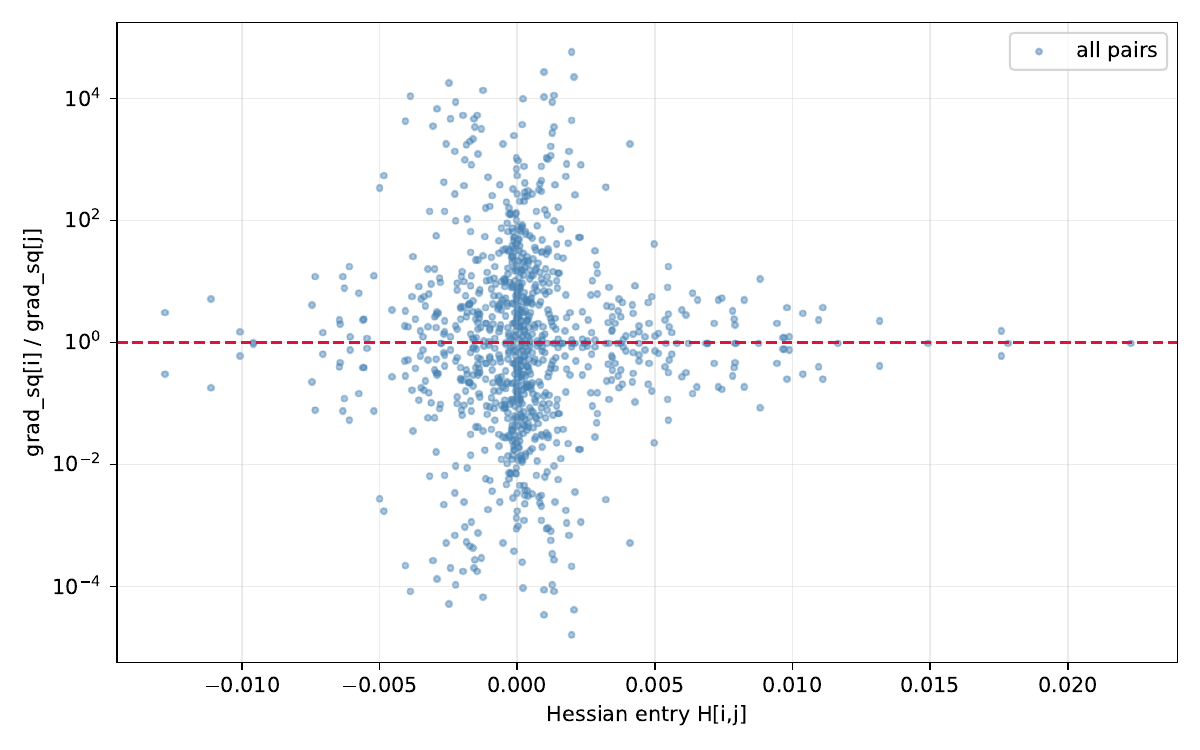}
            \textbf{(b)}
        \end{minipage}
        \caption{Squared-gradient ratio on a log scale vs. Hessian for CNN on MNIST at the
        beginning of training: (a) \texttt{conv1.weight}; (b) \texttt{conv1.bias}. The same phenomenon appears in both the weight and bias parameters, as in the transformer model.}
        \label{fig:cnn_mnist_hessian_sqgrad_ratio}
    \end{minipage}
\end{figure}
Another example, this time for a convolutional architecture, is shown in \cref{fig:cnn_mnist_hessian_sqgrad_ratio}. The same phenomenon appears there, even though the networks are very different.

\section{Extended Ablation Study}
\label{app:extended_ablation_study}
In addition to the previous ablations, we present an extended ablation study below.
\subsection{Exact Dynamic Programming Quantization}
\label{app:exact-dp-quantization}
We compare three methods for quantizing gradients, which reduces the memory footprint of the first moment. 

Lloyd-Max is an iterative method for learning a codebook for a given distribution. However, it has no theoretical guarantee of convergence to the optimal codebook, and in practice it can get stuck in local minima. We also find that it is sensitive to the initialization strategy: linear initialization (i.e., a uniformly spaced initial codebook) and quantile initialization (i.e., initializing the codebook at distribution quantiles) can lead to substantially different results. As shown in \cref{fig:quantization_comparison}, linear initialization sometimes outperforms quantile initialization, while in other cases the reverse is true.

\begin{figure}[H]
    \centering
       \begin{minipage}{0.48\linewidth}
        \centering
        \includegraphics[width=\linewidth]{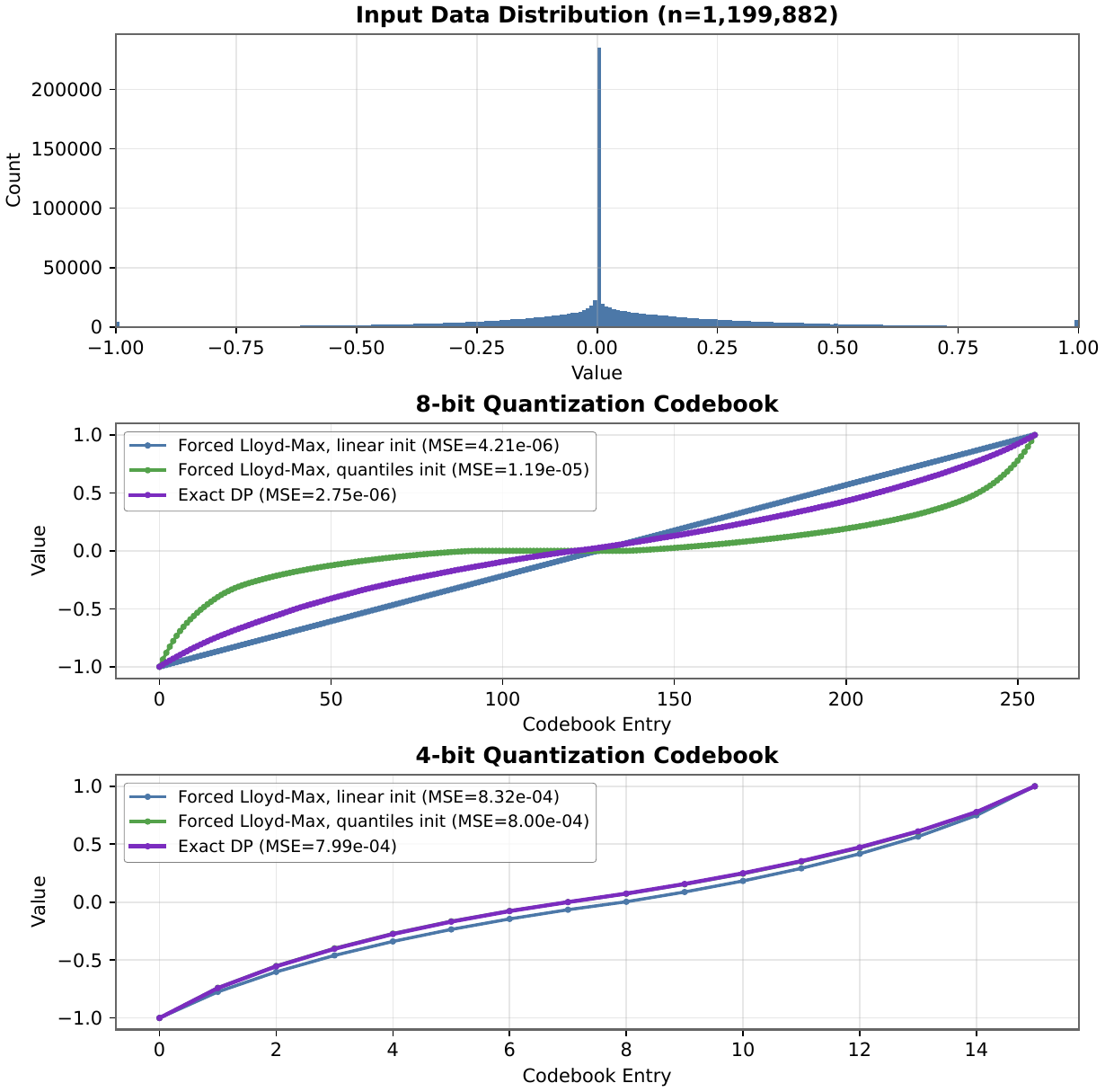}      
        
        (a) CNN on MNIST 
    \end{minipage}
    \hfill
    \begin{minipage}{0.48\linewidth}
        \centering
        \includegraphics[width=\linewidth]{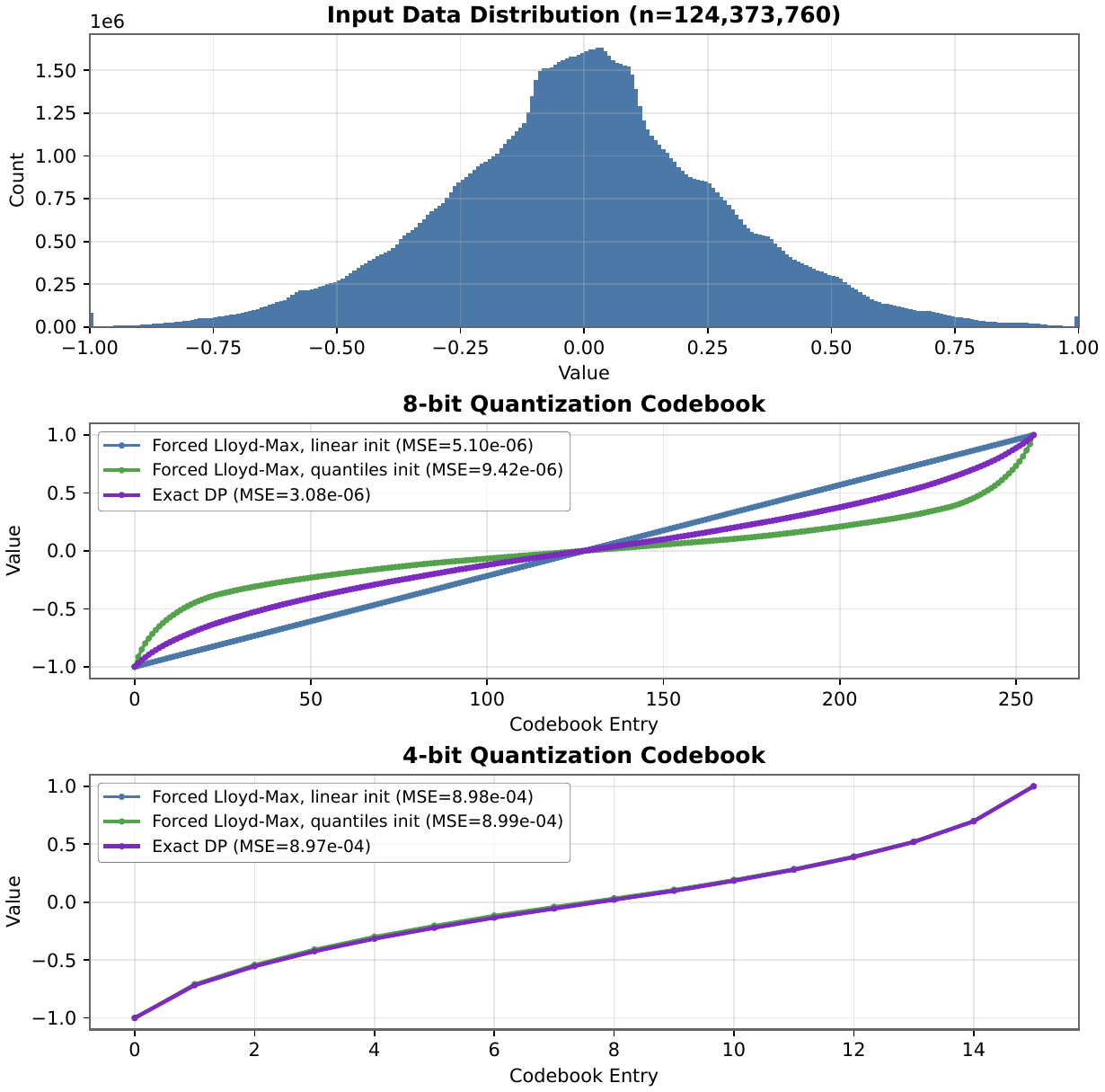}

        (b) GPT-2 on OpenWebText 
    \end{minipage}
    \caption{Comparison of gradient quantization methods. The top panels show the gradient distributions used as input to the quantization algorithms, and the bottom panels show the learned 8-bit and 4-bit codebooks for each method. Lloyd-Max is sensitive to the initialization scheme. In contrast, the exact dynamic programming method requires no initialization and achieves the lowest quantization error while remaining fast.}    
    \label{fig:quantization_comparison}
\end{figure}

We also compare against our exact histogram-based dynamic programming quantization algorithm. As expected, this method achieves the lowest quantization error because it is guaranteed to find the optimal codebook for the given distribution and quantization level, up to the histogram binning resolution. 
For example, for GPT-2 with 8-bit quantization, \ourmethod achieves an MSE of \(3.08 \times 10^{-6}\), compared with \(5.10 \times 10^{-6}\) for Lloyd-Max with linear (uniform) initialization and \(9.42 \times 10^{-6}\) with quantile initialization.

Finally, the optimal codebook found by our exact algorithm is nearly identical across models, datasets, and input distributions for the same quantization level. This is surprising: although the underlying distributions differ, they produce very similar optimal codebooks. 

As an ablation, we also compared the final performance of our exact quantization method against Lloyd-Max with different initialization strategies. For a CNN on MNIST over three consecutive seeds, the exact method achieved a final validation loss of \num{0.0269 +- 0.0009}, compared with \num{0.0283 +- 0.0020} for Lloyd-Max with quantile initialization and \num{0.0276 +- 0.0007} for Lloyd-Max with linear initialization. 

\subsection{Stable Partitioning}
Because gradients are noisy, it is important to verify whether the partitioning inferred from the first-step gradients remains stable throughout training.
In particular, we ask whether rerunning the partitioning algorithm at later steps recovers the same period for each tensor.
To evaluate this, we train the \gptsmall{} and CNN models and rerun the partitioning algorithm periodically, comparing the period inferred for each tensor with its first-step period.
\Cref{fig:gpt2_partitioning_info_per_step_grid,fig:mnist_partitioning_info} report the results.
Each entry is marked in green when the inferred period is identical to the first-step period or is an integer multiple of it. The latter case is also acceptable because the smaller first-step period yields a finer partition and cannot increase within-block heterogeneity. All other entries are marked in red.
Overall, the vast majority of entries are green, indicating that the inferred periods remain relatively stable throughout training.
Thus, the first-step partitioning appears to be a good proxy for the partitioning that would be obtained later, and repeated repartitioning is
likely to provide only limited additional benefit.
Despite the noise in the gradient estimates, both models exhibit largely consistent partitioning structures throughout training.
Interestingly, the figures show that, although the automatic partitioning algorithm receives only a flattened tensor, it often recovers a period equal to one of the original tensor dimensions, or a combination of a few dimensions.
\begin{figure}[!htbp]
    \centering
    \begin{subfigure}{\linewidth}
        \centering
        \includegraphics[width=\linewidth]{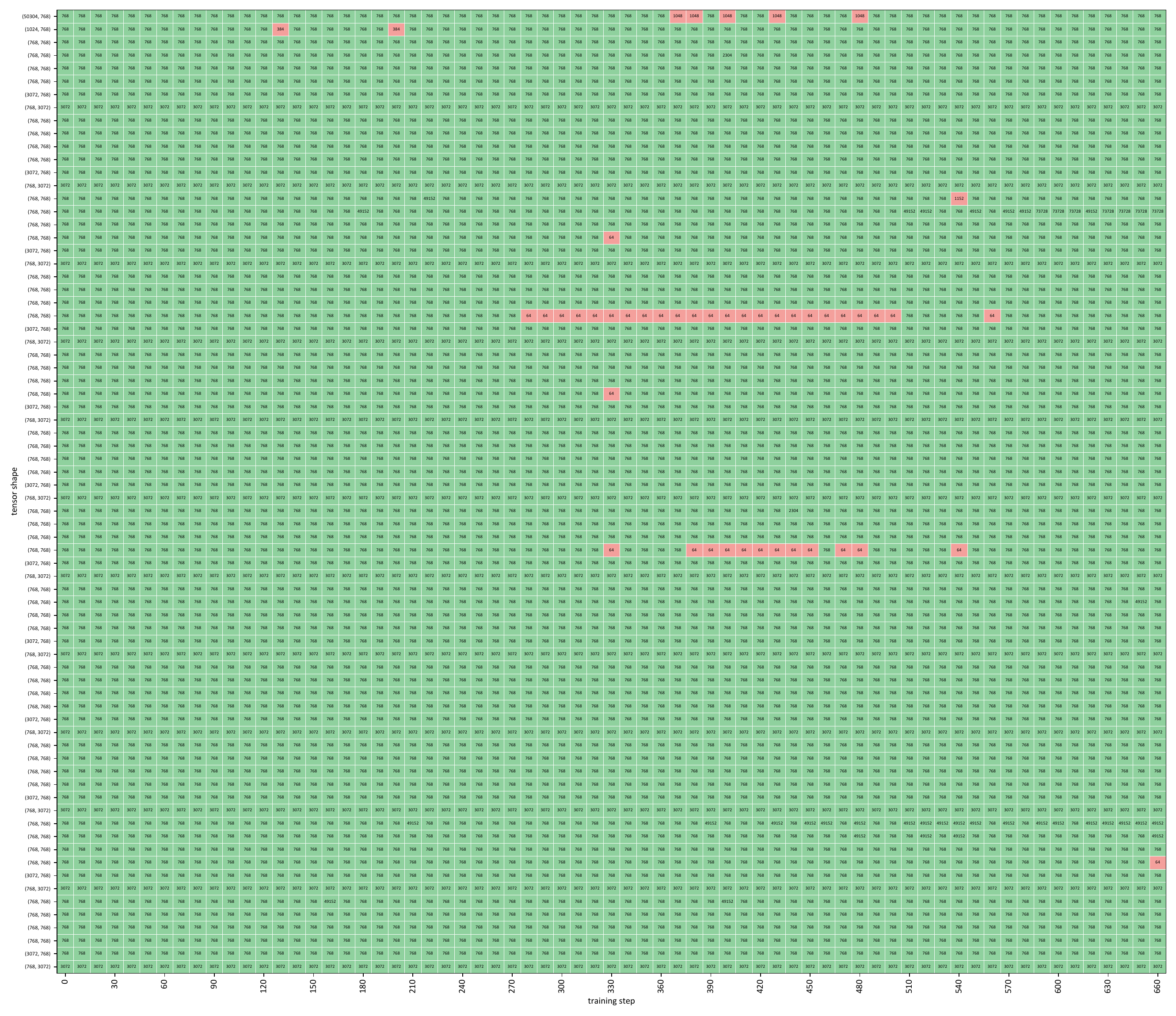}
        \caption{GPT-2.}
        \label{fig:gpt2_partitioning_info_per_step_grid}
    \end{subfigure}

    \begin{subfigure}{\linewidth}
        \centering
        \includegraphics[width=\linewidth]{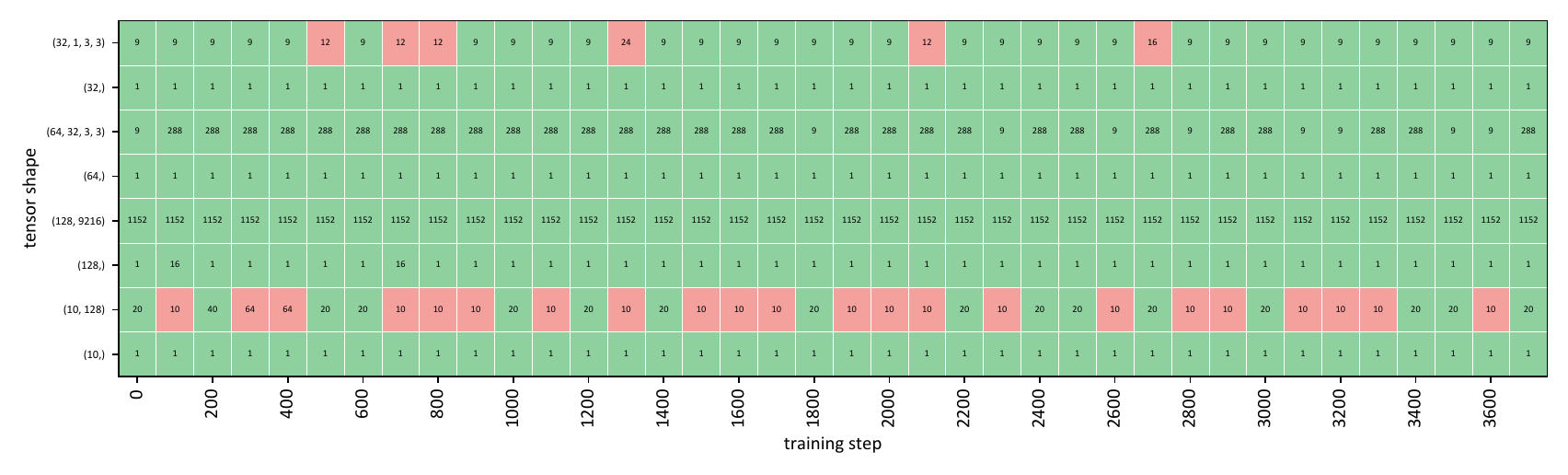}
        \caption{CNN-MNIST.}
        \label{fig:mnist_partitioning_info}
    \end{subfigure}
    \caption{Partitioning stability across training steps.}
    \label{fig:partitioning_stability}
\end{figure}

\subsection{Correct Partitioning}
We ablate the importance of correctly aligning the partitions.
Starting from the partition used by \ourmethod, we artificially perturb it in a few ways: we double the partition size, we shift the partition boundaries by half a period to deliberately create misalignment, and we randomize which weights in a given tensor belong to each partition.
The last perturbation keeps both the number of partitions and their size exactly as chosen by \ourmethod, and replaces the consecutive coordinates of every partition with a random subset of the same size, drawn without replacement so that the partitions stay mutually exclusive and jointly cover all weights.
The random assignment is drawn once and held fixed throughout training.
As shown in \cref{fig:gpt2_partition_ablation}, all perturbations degrade performance compared with the original partition.
\begin{figure}[!htbp]
    \centering
    \includegraphics[width=0.5\linewidth]{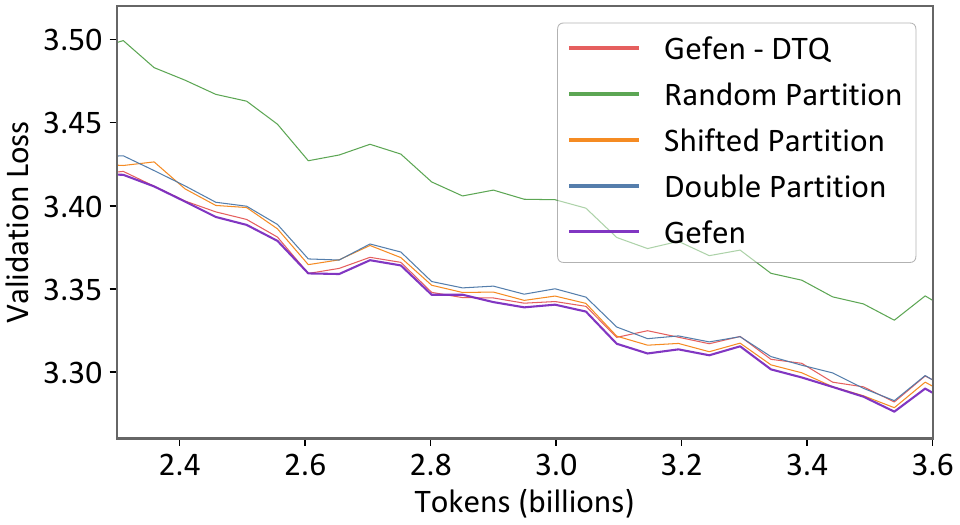}
    \caption{Partition perturbation and DTQ quantization ablations using \gptsmall{}
    trained on OpenWebText.}
    \label{fig:gpt2_partition_ablation}
\end{figure}

\subsection{Dynamic Tree Quantization}
We ablate \ourmethod's exact quantization of the first moment by replacing it with Dynamic Tree Quantization (DTQ), the scheme used by Adam8bit \citep{dettmers2022eightbit} (\cref{sec:related_work}).
Everything else is kept identical to \ourmethod: the first moment is quantized using the same partition as \ourmethod, with a single scale factor per partition, found from the second moment, exactly as in \ourmethod. The only change is that DTQ, rather than \ourmethod's exact dynamic programming quantization, is used to encode the normalized values within each partition.
As shown in \cref{fig:gpt2_partition_ablation}, \ourmethod is almost always better than the DTQ variant. 
In addition, \ourmethod does not need the partition size as a hyperparameter, since it reuses the partition found automatically for the second moment.

\section{\exacthistcodebook{}: Details}
\label[appendix]{app:dp_details}
The dynamic program \cref{alg:exact_histogram_codebook} exploits the fact that one-dimensional optimal quantization partitions the sorted input values into contiguous intervals.
After histogramming the normalized gradients, we discard empty bins and sort the remaining bin centers $m_i$ with counts $c_i$.

For any interval of bins $b_\ell,\ldots,b_r$, the weighted squared quantization error is:

\[
    C(b_\ell,b_r,k')=\sum_{i=b_\ell}^{b_r} c_i(m_i-\mu)^2
    = S_2 - 2\mu S_1 + \mu^2 S_0
\]

where $S_0:=\sum_{i=b_\ell}^{b_r} c_i$, $S_1:=\sum_{i=b_\ell}^{b_r} c_i m_i$, and $S_2:=\sum_{i=b_\ell}^{b_r} c_i m_i^2$ and $\mu$ is the value of the $k'$-th codebook entry. These are obtained in $O(1)$ time for any interval using prefix sums of $c_i$, $c_i m_i$, and $c_i m_i^2$.

For an unconstrained codebook entry, the optimal representative of the interval is its weighted mean $\mu=S_1/S_0$, and the interval cost is $S_2-S_1^2/S_0$. For the forced endpoints, the first and last codebook entries are fixed to $-1$ and $1$, respectively, so the cost for a fixed center $a$ is $S_2-2aS_1+a^2S_0$.

Let $D_{b_r,k'}$ be the minimum squared error for representing the first $b_r$
nonempty bins using $k'$ codebook entries. The recurrence considers the left
boundary $b_\ell$ of the final interval:
\[
    D_{b_r,k'} = \min_{b_\ell \le b_r} \left[ D_{b_\ell-1,k'-1} + C(b_\ell,b_r,k') \right]
\]

This is correct by optimal substructure: in any optimal solution, the bins assigned to the last codebook entry form a suffix interval $b_\ell,\ldots,b_r$, and all earlier bins must themselves be optimally quantized by the previous $k'-1$ entries. Conversely, combining any optimal prefix solution with the best cost for the final interval gives a valid $k'$-entry quantizer.
Backtracking the minimizing split points therefore recovers the globally optimal codebook for the histogram. Since there are $O(bk)$ DP states and each state checks $O(b)$ split points, the running time is $O(b^2k)$.

\section{Pretraining vs. Finetuning}
\label{sec:pretraining_vs_finetuning}
Pretraining is substantially more compute intensive than finetuning, and is therefore the primary focus of this work. Nevertheless, it is useful to examine whether the gradient structure that motivates our second-moment compression also appears during finetuning. We find that the squared-gradient structure in finetuning differs markedly from that observed in pretraining. During pretraining, the Hessian of a parameter tensor, or the squared-gradient proxy used in our method, reveals architectural structure: the periodicity underlying our second-moment sharing is visually apparent and can be detected reliably. In contrast, this periodicity is much less pronounced during finetuning. One possible explanation is that pretraining starts from randomly initialized weights, whereas finetuning begins from parameters already located near low-loss regions of the optimization landscape, which may alter the local curvature and gradient statistics. \Cref{fig:pre_vs_fine} illustrates this behavior for two architectures and three parameter tensors. The top row shows squared gradients at the beginning of pretraining, while the bottom row shows squared gradients during finetuning. 
When zoomed in (\cref{fig:qwen_zoomed}), the pretraining examples exhibit clear periodicity, in this case with period 2048, matching the hidden dimension of the corresponding networks. The finetuning examples, however, do not exhibit the same visible periodic structure. Since our second-moment sharing relies on periodicity across squared-gradient blocks, we focus this paper on the compute-intensive pretraining setting. By contrast, our first-moment quantization does not rely on grouping and applies equally well in both pretraining and finetuning.

\newcommand{\preFineSpySize}{1.5cm}
\newcommand{\preFineSpyMagnification}{5.0}
\newcommand{\preFineTopSpy}{\spy[red] on (4.7,3.4) in node at (1.7,4.7);}
\newcommand{\preFineBottomSpy}{\spy[red] on (4.7,0.66) in node at (1.7,2.0);}
\begin{figure}[H]
    \centering
       \begin{minipage}{0.31\linewidth}
        \centering
        \resizebox{\linewidth}{!}{%
            \begin{tikzpicture}[
                spy using outlines={circle,magnification=\preFineSpyMagnification,size=\preFineSpySize,connect spies}
            ]
                \node[anchor=south west, inner sep=0] at (0,0) {
                    \includegraphics[width=7cm]{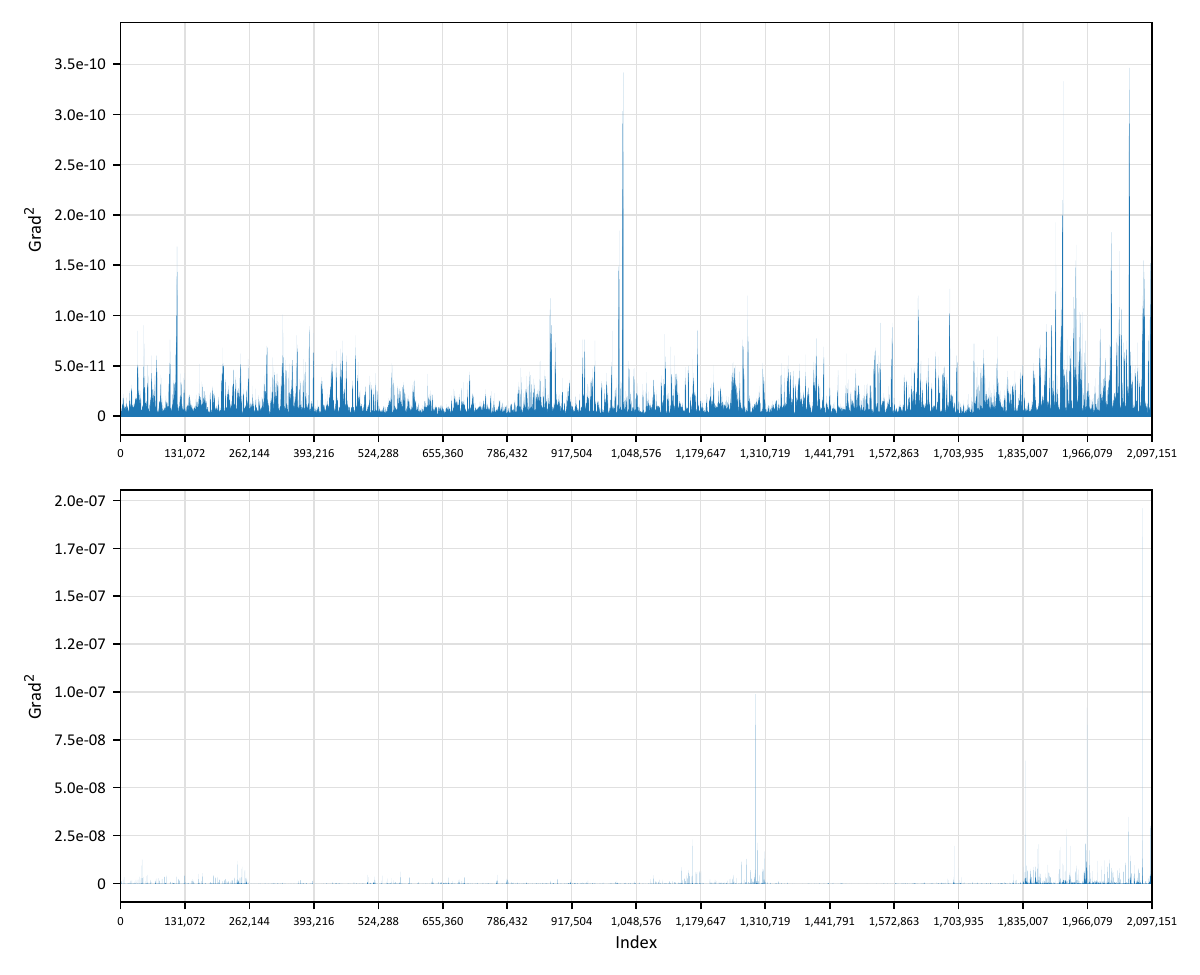}
                };
                \preFineTopSpy
                \preFineBottomSpy
            \end{tikzpicture}
        }

        (a) Qwen3-1.7B 27th-K
    \end{minipage}
    \hfill
    \begin{minipage}{0.31\linewidth}
        \centering
        \resizebox{\linewidth}{!}{%
            \begin{tikzpicture}[
                spy using outlines={circle,magnification=\preFineSpyMagnification,size=\preFineSpySize,connect spies}
            ]
            
                \node[anchor=south west, inner sep=0] at (0,0) {
                    \includegraphics[width=7cm]{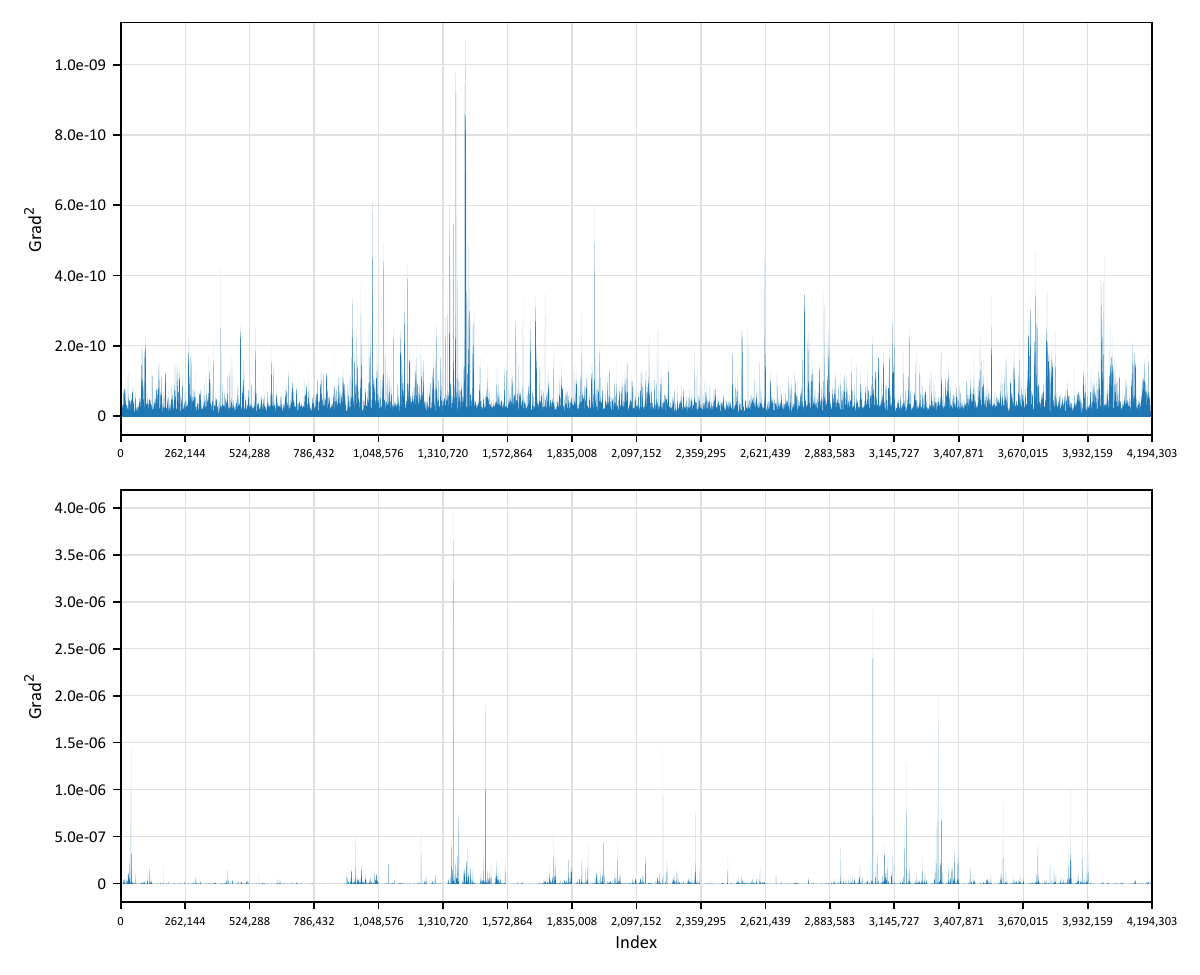}
                };
                \preFineTopSpy
                \preFineBottomSpy
            \end{tikzpicture}
        }

        (b) Llama-3.2-1B 7th-Q
    \end{minipage}
    \hfill
    \begin{minipage}{0.31\linewidth}
        \centering
        \resizebox{\linewidth}{!}{%
            \begin{tikzpicture}[
                spy using outlines={circle,magnification=\preFineSpyMagnification,size=\preFineSpySize,connect spies}
            ]
                \node[anchor=south west, inner sep=0] at (0,0) {
                    \includegraphics[width=7cm]{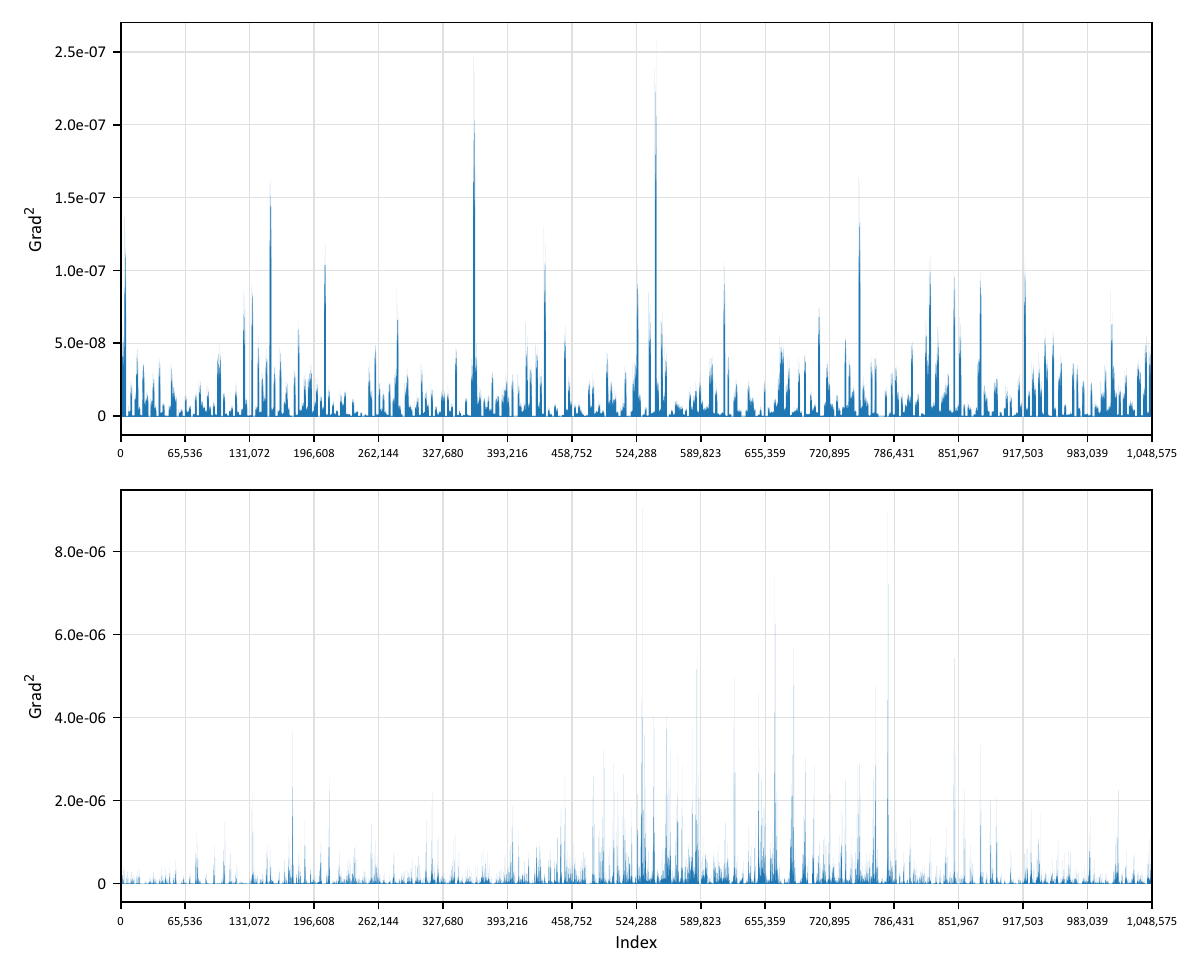}
                };
                \preFineTopSpy
                \preFineBottomSpy
            \end{tikzpicture}
        }

        (c) Llama-3.2-1B 13th-V
    \end{minipage}
    \caption{Squared-gradient structure in pretraining and finetuning. Each panel compares the same parameter tensor at the beginning of pretraining (top) and during finetuning (bottom). The magnified regions show that the periodic structure used for second-moment sharing is clear in pretraining but substantially less pronounced in finetuning.}    
    \label{fig:pre_vs_fine}
\end{figure}

\begin{figure}[H]
    \centering
    \includegraphics[width=0.6\linewidth]{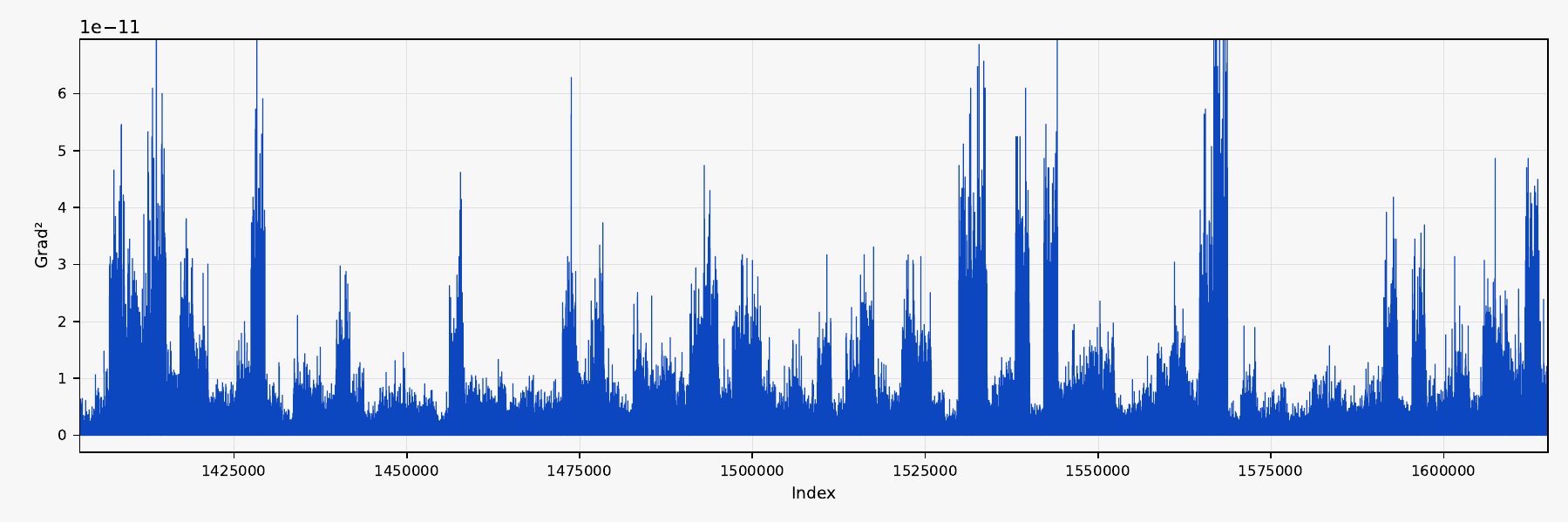}
    \caption{Zoomed-in squared-gradient structure for Qwen3-1.7B 27th layer, attention key parameter. Every 2048 parameters, corresponding to the hidden dimension of the network, exhibit similar squared-gradient magnitudes. This periodicity is used for second-moment sharing in \ourmethod.}
    \label{fig:qwen_zoomed}
\end{figure}

\section{Distributability of Optimizers}
\label{app:distributability}

Not every optimizer\footnote{For this comparison, we used
\texttt{torch.optim.AdamW} for AdamW,
\href{https://github.com/thu-ml/low-bit-optimizers}{\texttt{thu-ml/low-bit-optimizers}}
for Adam4bit,
\href{https://github.com/bitsandbytes-foundation/bitsandbytes/}{\texttt{bitsandbytes-foundation/bitsandbytes}}
for Adam8bit, \href{https://github.com/zyushun/Adam-mini}{\texttt{zyushun/Adam-mini}} for
Adam-mini, \href{https://github.com/jettify/pytorch-optimizer}{\texttt{jettify/pytorch-optimizer}}
for SM3 and Adafactor, and \texttt{torch.optim.Muon} for Muon.} is compatible with every distributed-training mechanism.
In some cases, the limitation is inherent to the optimizer design, while in
others it reflects implementation difficulty. As shown in
\cref{tab:optimizer-distributability}, AdamW and
\ourmethod are practical, general-purpose optimizers that support the common
methods for distributed training.

We evaluate multiple common distributed-training settings. DDP refers to standard
PyTorch \texttt{DistributedDataParallel}, where each rank holds a full model
replica and gradients are all-reduced. FSDP refers to PyTorch FSDP with full
sharding, applied after the model has been materialized on CUDA and before the
optimizer is constructed. FSDP2 Sharded Init refers to the TorchTitan-style
composable FSDP path, where the model is constructed on the meta device, sharded
with \texttt{fully\_shard}, materialized with \texttt{to\_empty(device="cuda")},
initialized after sharding, and then passed to the optimizer. DeepSpeed ZeRO-3
\citep{rajbhandari2020zero} refers to ZeRO stage 3 applied to a normally
materialized CUDA model. Tensor Parallelism partitions individual model tensors
and their computations across ranks.
\begingroup
\definecolor{green}{rgb}{0,0.55,0}
\definecolor{red}{rgb}{0.75,0,0}
\begin{table}[H]
    \centering
    \caption{Compatibility of optimizers with common distributed-training configurations.}
    \label{tab:optimizer-distributability}
    \begin{tabular}{lccccc}
        \hline
        Optimizer & DDP & FSDP & FSDP2 Sharded Init & DeepSpeed ZeRO-3 & Tensor Parallelism \\
        \hline
        AdamW & \textcolor{green}{\ding{51}} & \textcolor{green}{\ding{51}} & \textcolor{green}{\ding{51}} & \textcolor{green}{\ding{51}} & \textcolor{green}{\ding{51}} \\
        Adam4bit & \textcolor{green}{\ding{51}} & \textcolor{green}{\ding{51}} & \textcolor{red}{\ding{55}} & \textcolor{green}{\ding{51}} & \textcolor{red}{\ding{55}} \\
        Adam8bit & \textcolor{green}{\ding{51}} & \textcolor{green}{\ding{51}} & \textcolor{red}{\ding{55}} & \textcolor{green}{\ding{51}} & \textcolor{red}{\ding{55}} \\
        Adam-mini & \textcolor{green}{\ding{51}} & \textcolor{red}{\ding{55}} & \textcolor{green}{\ding{51}} & \textcolor{green}{\ding{51}} & \textcolor{green}{\ding{51}} \\
        SM3 & \textcolor{green}{\ding{51}} & \textcolor{green}{\ding{51}} & \textcolor{red}{\ding{55}} & \textcolor{green}{\ding{51}} & \textcolor{red}{\ding{55}} \\
        Adafactor & \textcolor{green}{\ding{51}} & \textcolor{green}{\ding{51}} & \textcolor{red}{\ding{55}} & \textcolor{green}{\ding{51}} & \textcolor{red}{\ding{55}} \\
        Gefen & \textcolor{green}{\ding{51}} & \textcolor{green}{\ding{51}} & \textcolor{green}{\ding{51}} & \textcolor{green}{\ding{51}} & \textcolor{green}{\ding{51}} \\
        \hline
    \end{tabular}
\end{table}
\endgroup

\section{Extended Related Work}
\label{app:extended_related}

\paragraph{Second-Moment Compression.}
A complementary line of work reduces optimizer memory by compressing how second-moment statistics are stored. \textbf{Adafactor} \citep{shazeer2018adafactor} replaces a full $n\times m$ second-moment tensor with factored row and column accumulators (roughly $n+m$ state), then reconstructs elementwise scales from these marginals. \textbf{SM3} \citep{anil2019sm3} keeps per-axis accumulators (for matrices, row-wise and column-wise maxima), forms each element's preconditioner from the elementwise minimum across axes, and scales updates by the inverse square root of that value (with $\epsilon$ in practice), yielding a memory footprint similar to Adafactor. \textbf{Lion} \citep{chen2023symbolic} is a sign-momentum optimizer discovered through symbolic program search, and its design rationale is less transparent. However, these methods rely on structural compression assumptions (factorizability in Adafactor, axis-wise summarization in SM3) that may not match every layer or training regime, and they often require learning-rate retuning relative to AdamW. These limitations make them less convenient as drop-in replacements, and they have already been found to underperform \textbf{Adam-mini}.

\paragraph{Other Optimizers.}
\textbf{Muon} (MomentUm Orthogonalized by Newton-Schulz) was recently introduced as an optimizer for hidden-layer matrix parameters \citep{jordan2024muon}. Its update starts from momentum and then applies matrix orthogonalization using a few Newton--Schulz iterations, essentially equalizing singular values \citep{bjorck1971orthogonal,higham2008functions}. Muon uses about half the optimizer-state memory of AdamW because it tracks only the first-moment estimates (momentum). However, it is not a complete general-purpose drop-in optimizer: it is not intended for non-matrix tensors, usually handles input/output embeddings and classifier heads with a different optimizer (such as AdamW), requires manually selecting which parameters are handled by Muon and which by a general-purpose optimizer, and introduces additional optimizer-specific hyperparameters such as the number of Newton--Schulz iterations and its polynomial coefficients. In addition, it may have difficulty fine-tuning models pre-trained with AdamW \citep{liu2025muon}, and its optimal learning rate is often different from that of AdamW.
Muon is complementary to \ourmethod and can work in tandem with it, as shown in \cref{app:extensions}.
\textbf{Apollo} \citep{zhu2024apollosgdlikememoryadamwlevel} projects gradients, only for parameters selected by the user, into a low-dimensional subspace and then performs AdamW iteration, thereby reducing optimizer state memory. However, similar to Adam-mini, its implementation relies on manual if/else rules based on module names (e.g., applying their method only to "attn" and "mlp" in the parameter name), while requiring additional optimizer-specific hyperparameters from the user beyond AdamW, including the desired low-rank dimension, scaling factors, and learning-rate retuning, making it less convenient as a drop-in replacement.
\textbf{Shampoo} \citep{gupta2018shampoo} is a structure-aware preconditioned optimizer that maintains matrix preconditioners for each tensor dimension, enabling richer second-order-style updates than diagonal adaptive methods. However, its goal is improved preconditioning rather than AdamW-style optimizer-state compression, and the additional matrix states and inverse-root computations make it less directly comparable to lightweight drop-in AdamW replacements.
\textbf{Sophia} \citep{liu2024sophia} is a stochastic second-order optimizer that divides an exponential moving average of gradients by an exponential moving average of lightweight diagonal Hessian estimates, and then applies elementwise clipping to the resulting update, for faster convergence. However, Sophia has the same optimizer-state memory footprint as AdamW, since it stores both first-moment and Hessian-preconditioner states, and it introduces additional optimizer-specific hyperparameters such as the Hessian update frequency and the clipping parameter; in contrast, we seek a general-purpose optimizer with lower memory footprint and no additional user-tuned hyperparameters beyond AdamW.

\paragraph{Gradient Communication Compression.}
Another complementary direction reduces distributed-training communication rather than optimizer-state memory. \textbf{1-bit Adam} \citep{tang2021onebitadam} compresses Adam's gradient communication by transmitting one bit per coordinate, corresponding to the sign of the gradient, $\operatorname{sign}(g_t)$, together with error compensation, using a warmup stage to stabilize Adam's variance state before switching to compressed communication. However, each GPU still maintains Adam's full first- and second-moment states, so the per-GPU optimizer-state memory footprint remains roughly $2\times$ the parameter size. This line of work targets the network bottleneck in data-parallel training, whereas \ourmethod targets the memory footprint of optimizer states. Consequently, 1-bit Adam and \ourmethod are largely orthogonal: the communication compression of 1-bit Adam should compose naturally with the optimizer-state compression in \ourmethod, allowing the two techniques to work in tandem.

\paragraph{Zeroth-Order Optimizers.}
Zeroth-order (ZO) optimizers estimate search directions from function evaluations rather than backpropagated gradients, often using simultaneous perturbation or finite-difference estimators \citep{spall1992multivariate}. This makes them attractive for memory-constrained fine-tuning because they avoid storing activations for the backward pass. For example, \textbf{MeZO} \citep{malladi2023finetuning} adapts ZO-SGD to fine-tune language models using only forward passes, and follow-up methods such as \textbf{ZO-AdaMU} \citep{jiang2023zoadamuoptimizer} and \textbf{Sparse MeZO} \citep{liu2026sparsemezo} improve the stability, convergence speed, or parameter efficiency of this approach. However, ZO methods optimize through noisy low-dimensional projections of the gradient and therefore trade memory savings for weaker or less robust optimization compared with standard backpropagation-based AdamW in the general-purpose training regimes considered here. We therefore view them as an important but distinct direction: they reduce memory by changing the optimization oracle itself, whereas \ourmethod preserves the AdamW training interface and targets AdamW-level performance with a smaller optimizer-state footprint.

\paragraph{Distributed Training.}
Large-scale training commonly has several forms of parallelism. In this work, we focus on two common distributed-training schemes: DDP (Distributed Data Parallel) and FSDP (Fully Sharded Data Parallel). Data-parallel methods such as DDP \citep{li2020pytorchdistributed} replicate the model on each worker and synchronize gradients after the backward pass. Sharded data-parallel methods such as FSDP \citep{zhao2023pytorchfsdp} and DeepSpeed-ZeRO \citep{rajbhandari2020zero} reduce per-GPU memory by partitioning some or all of the parameters, gradients, and optimizer states across workers, at the cost of additional communication and implementation constraints. Model-parallel techniques, including tensor parallelism \citep{shoeybi2019megatron} and pipeline parallelism \citep{huang2019gpipe}, split the model computation itself across devices and are often combined with data parallelism for very large models. These techniques address memory and throughput at the system level, whereas \ourmethod reduces the optimizer state that each training configuration must store or shard. Therefore, optimizer-state compression is complementary to distributed-training schemes: it can reduce the memory pressure within DDP, lower the amount of optimizer state handled by sharded methods, and, as shown in our experiments, enable larger per-GPU microbatches in both FSDP and DDP settings.

\section{Additional Experiments}
\label[appendix]{app:more_experiments}
In \cref{fig:additional_experiments}, we present additional plots for the experiments reported in \cref{fig:gpt2_and_llama_train_curve}.
\ourmethod performs on par with AdamW across the settings.
In the ResNet18 and \gptsmall experiments, 4-bit quantization negatively affects performance. Furthermore, it can be seen that not all existing methods perform equally well with AdamW's learning rate. In the \gptsmall experiment, SM3 and Adafactor perform worse than AdamW when using the same learning rate.
Adam-mini cannot be evaluated in all experiments because it relies on hard-coded tensor names and therefore does not support every model.

In the diffusion task DiT-XL/2 on ImageNet, \ourmethod reduces the AdamW memory state from 5720MiB to 722MiB.

For \gptsmall, the optimizer step takes 11 ms with AdamW and 3 ms with \ourmethod; for \llamaoneb, it takes 115 ms with AdamW and 29 ms with \ourmethod. This reduction in latency occurs because \ourmethod writes less optimizer-state data to memory.

\begin{figure}[tbp]
    \centering    
    \begin{minipage}{0.31\linewidth}
        \centering
        \resizebox{\linewidth}{!}{%
            \begin{tikzpicture}[
                spy using outlines={circle,magnification=2.5,size=1.5cm,connect spies}
            ]
                \node[anchor=south west, inner sep=0] at (0,0) {
                    \includegraphics[width=7cm]{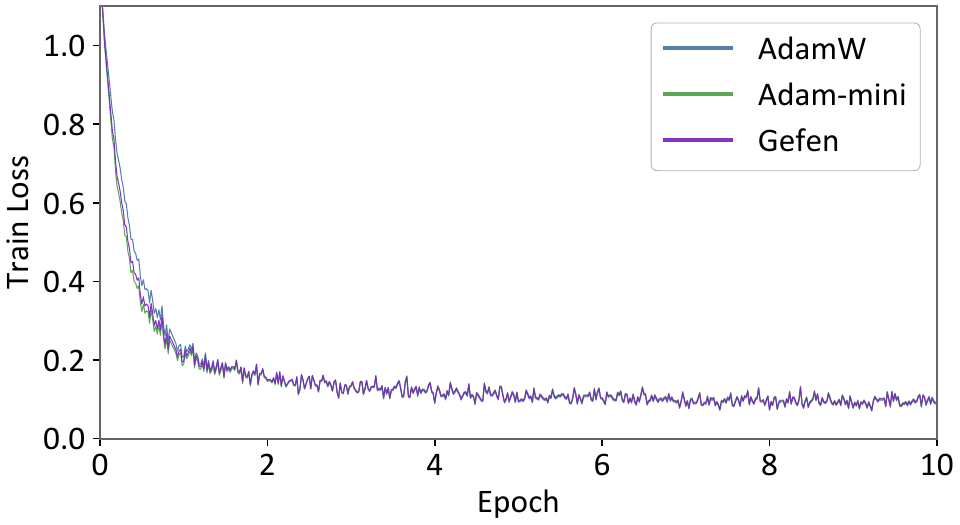}
                };
                \spy[red] on (6.30,1.00) in node at (3.5,2.6);
            \end{tikzpicture}
        }

        {\scriptsize (a) Loss: DDPM (Diffusion)}
    \end{minipage}
    \hfill
    \begin{minipage}{0.31\linewidth}
        \centering
        \resizebox{\linewidth}{!}{%
            \begin{tikzpicture}[
                spy using outlines={circle,magnification=2.7,size=1.4cm,connect spies}
            ]
                \node[anchor=south west, inner sep=0] at (0,0) {
                    \includegraphics[width=7cm]{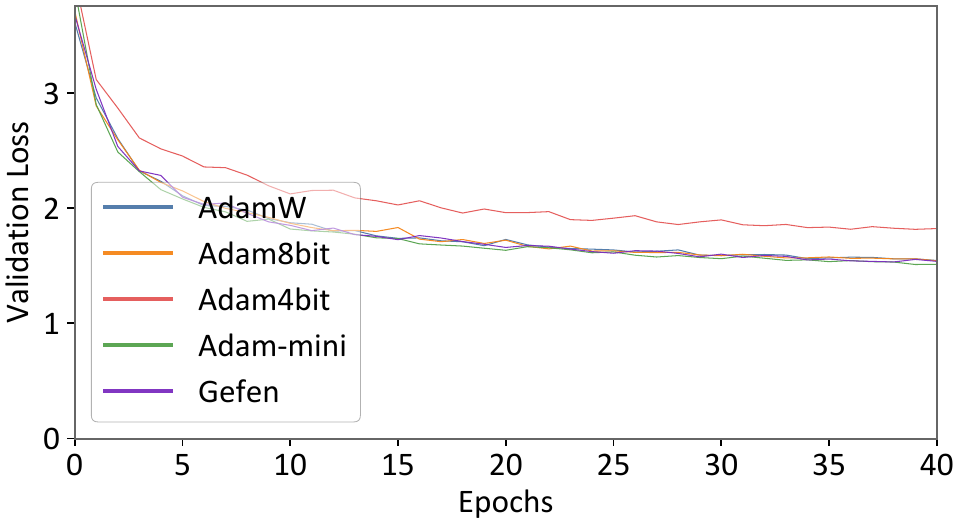}
                };
                \spy[red] on (6.30,1.89) in node at (4.45,3.0);
            \end{tikzpicture}
        }

        {\scriptsize (b) Loss: ResNet18 on ImageNet}
    \end{minipage}
    \hfill
    \begin{minipage}{0.31\linewidth}
        \centering
        \resizebox{\linewidth}{!}{%
            \begin{tikzpicture}[
                spy using outlines={circle,magnification=2.5,size=1.8cm,connect spies}
            ]
                \node[anchor=south west, inner sep=0] at (0,0) {
                    \includegraphics[width=7cm]{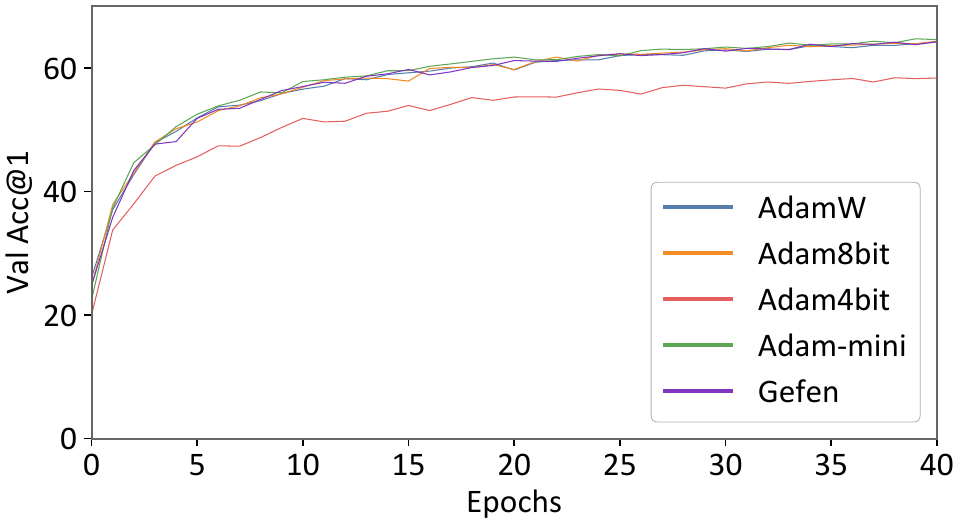}
                };
                \spy[red] on (5.9,3.3) in node at (3.0,1.9);
            \end{tikzpicture}
        }

        {\scriptsize (c) Acc: ResNet18 on ImageNet}
    \end{minipage}
    \par\medskip
    \begin{minipage}{0.31\linewidth}
        \centering
        \resizebox{\linewidth}{!}{%
            \begin{tikzpicture}[
                spy using outlines={circle,magnification=2.5,size=1.6cm,connect spies}
            ]
                \node[anchor=south west, inner sep=0] at (0,0) {
                    \includegraphics[width=7cm]{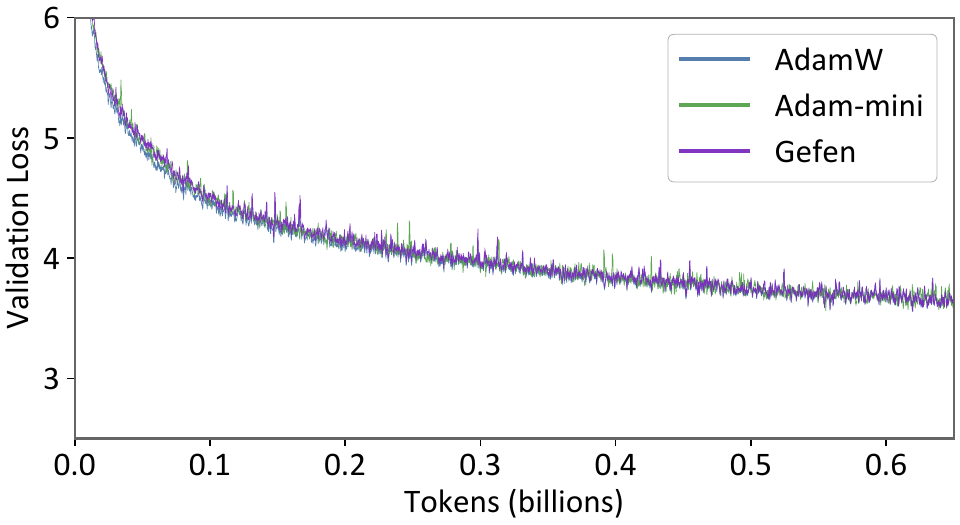}
                };
                \spy[red] on (6.2,1.7) in node at (4.0,2.8);
            \end{tikzpicture}
        }

        {\scriptsize (d) Loss: \llamaoneb on C4}
    \end{minipage}
    \hfill
    \begin{minipage}{0.31\linewidth}
        \centering
        \resizebox{\linewidth}{!}{%
            \begin{tikzpicture}[
                spy using outlines={circle,magnification=2.5,size=1.5cm,connect spies}
            ]
                \node[anchor=south west, inner sep=0] at (0,0) {
                    \includegraphics[width=7cm]{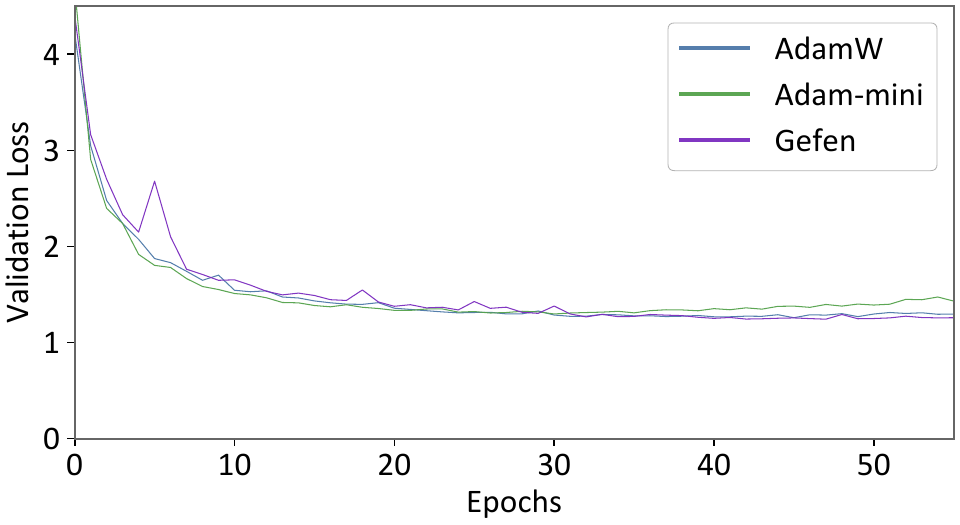}
                };
                \spy[red] on (6.30,1.50) in node at (3.5,2.6);
            \end{tikzpicture}
        }

        {\scriptsize (e) Loss: ResNet101 on ImageNet}
    \end{minipage}
    \hfill
    \begin{minipage}{0.31\linewidth}
        \centering
        \resizebox{\linewidth}{!}{%
            \begin{tikzpicture}[
                spy using outlines={circle,magnification=4.5,size=1.2cm,connect spies}
            ]
                \node[anchor=south west, inner sep=0] at (0,0) {
                    \includegraphics[width=7cm]{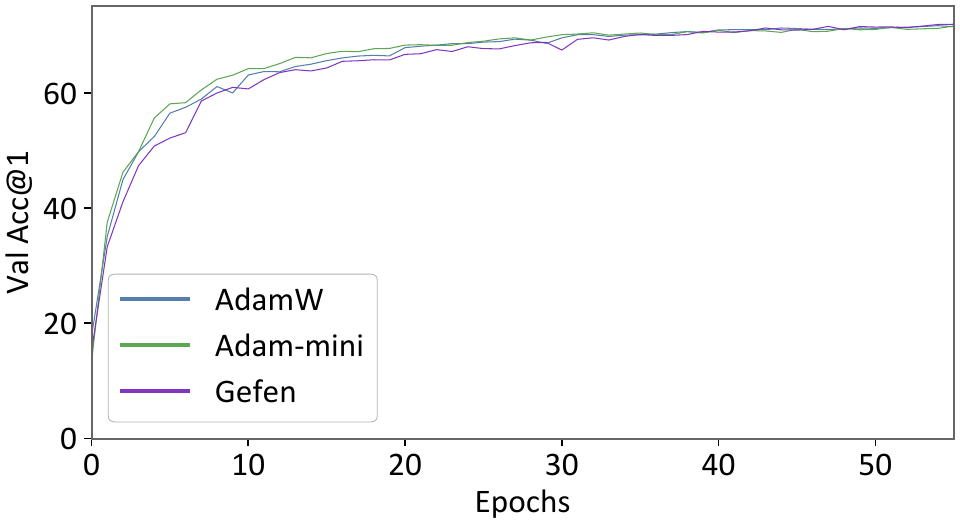}
                };
                \spy[red] on (6.5,3.59) in node at (4.0,1.6);
            \end{tikzpicture}
        }

        {\scriptsize (f) Acc: ResNet101 on ImageNet}
    \end{minipage}
    \par\medskip
    \begin{minipage}{0.31\linewidth}
        \centering
        \resizebox{\linewidth}{!}{%
            \begin{tikzpicture}[
                spy using outlines={circle,magnification=2.5,size=1.8cm,connect spies}
            ]
                \node[anchor=south west, inner sep=0] at (0,0) {
                    \includegraphics[width=7cm]{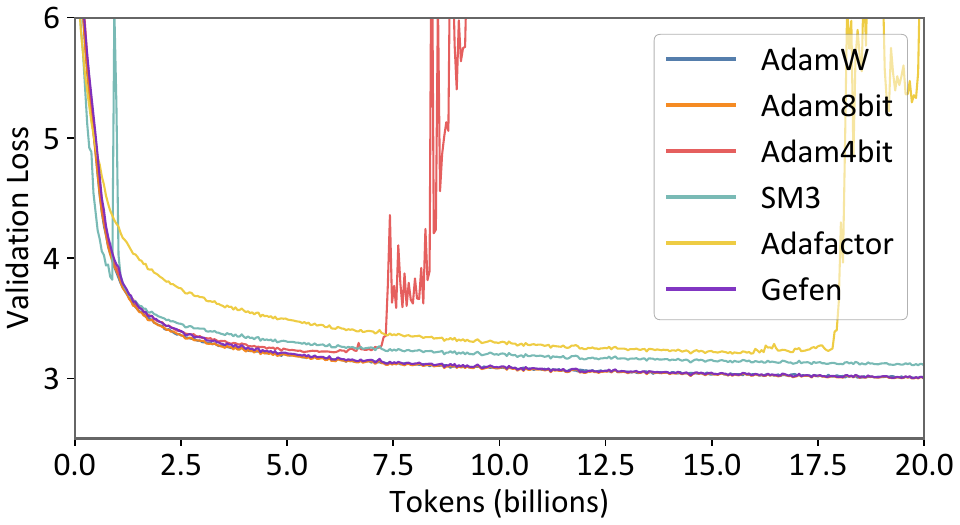}
                };
                \spy[red] on (3.5,1.2) in node at (2.0,2.6);
            \end{tikzpicture}
        }

        {\scriptsize (g) Loss: \gptsmall on OpenWebText}
    \end{minipage}
    \hfill
    \begin{minipage}{0.31\linewidth}
        \centering
        \resizebox{\linewidth}{!}{%
            \begin{tikzpicture}[
                spy using outlines={circle,magnification=8.7,size=2.0cm,connect spies}
            ]
                \node[anchor=south west, inner sep=0] at (0,0) {
                    \includegraphics[width=7cm]{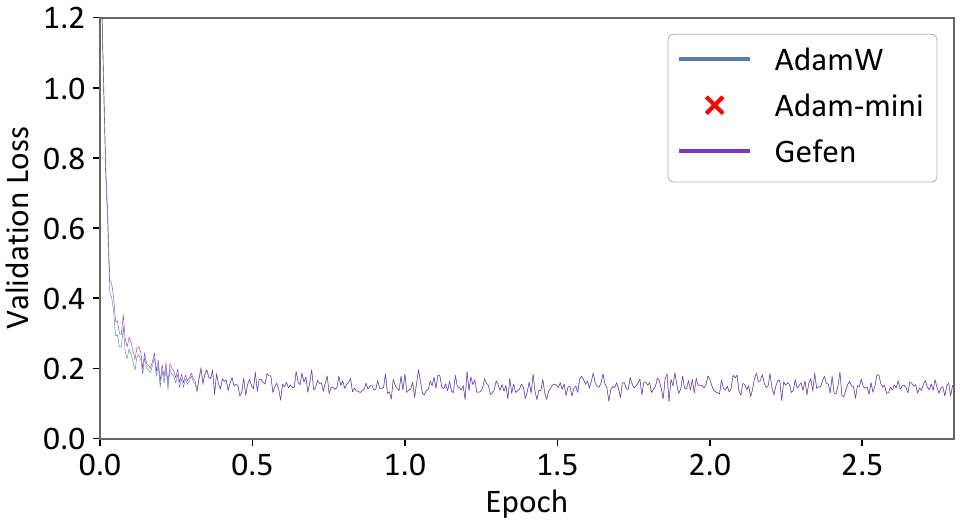}
                };
                \spy[red] on (6.65,1.01) in node at (3.39,2.3);
            \end{tikzpicture}
        }

        {\scriptsize (h) Loss: DiT-XL/2 on ImageNet}
    \end{minipage}
    \hfill
    \begin{minipage}{0.31\linewidth}
        \centering
        \resizebox{\linewidth}{!}{%
            \begin{tikzpicture}[
                spy using outlines={circle,magnification=2.5,size=1.5cm,connect spies}
            ]
                \node[anchor=south west, inner sep=0] at (0,0) {
                    \includegraphics[width=7cm]{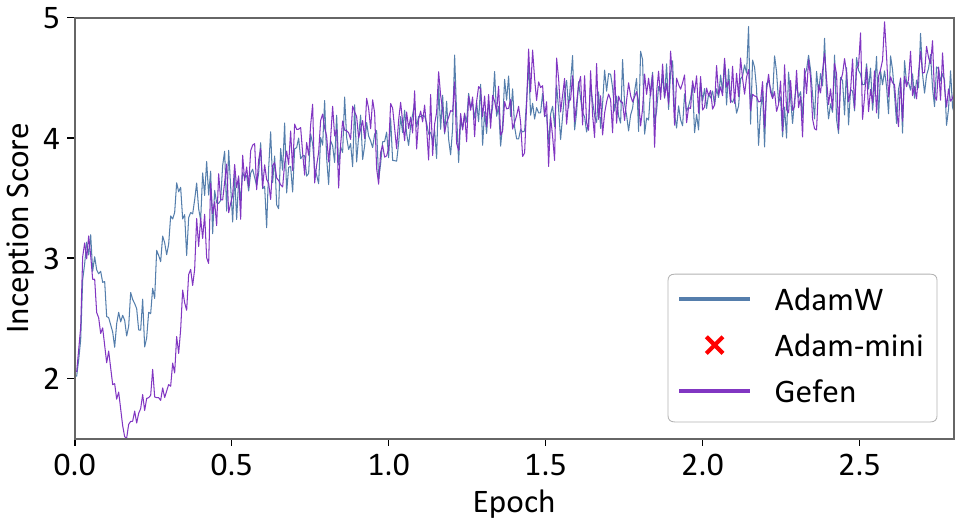}
                };
                \spy[red] on (6.5,3.3) in node at (3.5,1.6);
            \end{tikzpicture}            
        }

        {\scriptsize (i) Inception Score: DiT-XL/2 on ImageNet}
    \end{minipage}
    \caption{Training curves of DDPM on Bollywood-Celebs, ResNet18 and ResNet101 on ImageNet, \gptsmall on OpenWebText, \llamaoneb on C4, and DiT-XL/2 on ImageNet. \ourmethod performs on par with AdamW across the settings.}
    \label{fig:additional_experiments}
\end{figure}

\section{Extensions to \ourmethod}
\label{app:extensions}
\ourmethod modifies the AdamW state representation, and is therefore
complementary to several methods that reduce memory along different axes. We
list several natural extensions below.
\begin{enumerate}
    \item \textbf{Combining with Muon.} Muon, discussed in
    \cref{app:extended_related}, stores momentum for matrix parameters and then
    applies Newton--Schulz iterations to obtain a pseudo-orthogonalized update.
    Since Muon does not maintain an AdamW-style second moment accumulator, a
    natural combination with \ourmethod would focus on quantizing Muon's
    momentum buffer, while leaving the pseudo-orthogonalization step unchanged. A concrete implementation and evaluation of this combination is provided in \cref{app:gefen-muon}.
    \item \textbf{Combining with GaLore.} GaLore \citep{zhao2024galore}
    is a memory-efficient training method for LLMs.
    Given a gradient matrix $G_t$, GaLore projects it into a compact gradient $R_t$ using low-rank projection matrices, applies an optimizer in this compact space, and then projects the resulting update back to the original parameter space. 
    Since this projection is orthogonal to how optimizer statistics are represented, one could combine \ourmethod with GaLore when GaLore is used with AdamW-style optimizer states, by compressing the first and second moment estimates maintained inside the projected optimization step.
    \item \textbf{Combining with Sophia.} Sophia, discussed in \cref{app:extended_related}, replaces AdamW's second moment estimate with a lightweight diagonal Hessian estimate and a clipped update.
    A Gefen variant could potentially quantize Sophia's gradient momentum and store its diagonal Hessian estimator compactly across blocks, while keeping Sophia's Hessian-based preconditioner and clipping rule unchanged.
\end{enumerate}

\section{Gefen-Muon}
\label{app:gefen-muon}
To demonstrate that \ourmethod can be combined with other optimization methods, we provide a PyTorch implementation with CUDA kernels for \ourmethod-Muon. 
\begin{wraptable}{r}{0.50\linewidth}
\centering
\caption{Memory footprint normalized by parameter memory of \ourmethod-Muon vs Muon.
}
\label{tab:memory_footprint_gefen_muon}
\resizebox{\linewidth}{!}{%
\begin{tabular}{lcccc}
\\
\hline
Optimizer & \multicolumn{2}{c}{CNN (1.2M)} & \multicolumn{2}{c}{GPT-2 (125M)} \\

  & Persistent mem & Peak mem & Persistent mem & \textbf{Peak mem} \\
\hline
Muon & x1.00 & x3.12 & x1.01 & x1.63 \\
\ourmethod-Muon & x0.26 & x3.21 & \textbf{x0.25} & \textbf{x1.01} \\
\hline
\end{tabular}
}
\end{wraptable}

The idea is straightforward: take \ourmethod and apply the orthogonalization step from Muon. This yields a one-line code-change drop-in replacement for Muon, while reducing Muon's persistent memory state by $4\times$ and decreasing peak memory by approximately $38\%$ compared with Muon as the model gets larger (\cref{tab:memory_footprint_gefen_muon}). For the orthogonalization step, we use the official PyTorch implementation of Muon.

\begin{wraptable}[8]{r}{0.35\linewidth}
\centering
\caption{CNN on MNIST.
}
\label{tab:mnist_validation_loss_gefenmuon_vs_muon}
\tiny
\begin{tabular}{lc}
\\
\hline
Optimizer & Validation Loss $\downarrow$ \\
\hline
Muon & \numsub{0.0354 +- 0.0005} \\ %
\ourmethod-Muon & \numsub{0.0356 +- 0.0005} \\ %
\hline
\end{tabular}
\vspace{-6.2em}
\end{wraptable}

For simplicity, and to demonstrate that \ourmethod-Muon is a drop-in replacement, we use the same learning rate for Muon and \ourmethod-Muon. 
We also apply the same optimizer to all tensor types in the network by flattening tensors with more than two dimensions and expanding one-dimensional tensors into two dimensions.

As shown in \cref{fig:muon_vs_gefenmuon_curves}, the curves of \ourmethod-Muon and Muon are nearly identical in both pretraining and finetuning, suggesting that \ourmethod-Muon preserves Muon's optimization behavior while substantially reducing memory usage.
The left panel shows \gptsmall pretraining on OpenWebText, and the right panel shows Qwen3-1.7B finetuning on Alpaca.
We also compared the methods using CNN on MNIST as shown in \cref{tab:mnist_validation_loss_gefenmuon_vs_muon}.
Full experimental settings are provided in \cref{app:hyperparameters}.

\begin{figure}[H]
    \centering
    \begin{minipage}[t]{0.49\linewidth}
        \centering
        \resizebox{\linewidth}{!}{%
            \begin{tikzpicture}[
                spy using outlines={circle,magnification=2.5,size=0.8cm,connect spies}
            ]
                \node[anchor=south west, inner sep=0] at (0,0) {
                    \includegraphics[width=7cm]{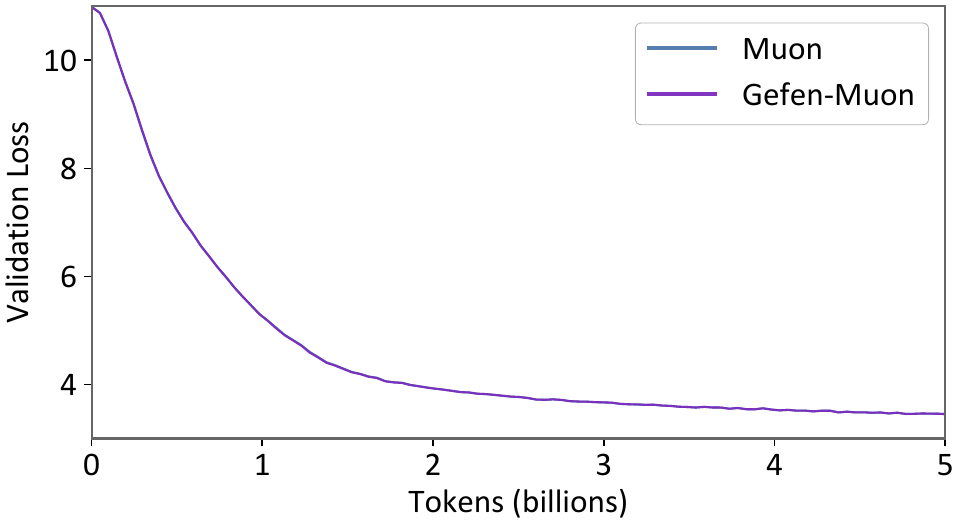}
                };
                \spy[red] on (6.5,0.85) in node at (3.0,2.6);
            \end{tikzpicture}%
        }

        \centering (a) Pretraining \gptsmall on OpenWebText.
    \end{minipage}\hfill
    \begin{minipage}[t]{0.49\linewidth}
        \centering
        \resizebox{\linewidth}{!}{%
            \begin{tikzpicture}[
                spy using outlines={circle,magnification=2.5,size=0.8cm,connect spies}
            ]
                \node[anchor=south west, inner sep=0] at (0,0) {
                    \includegraphics[width=7cm]{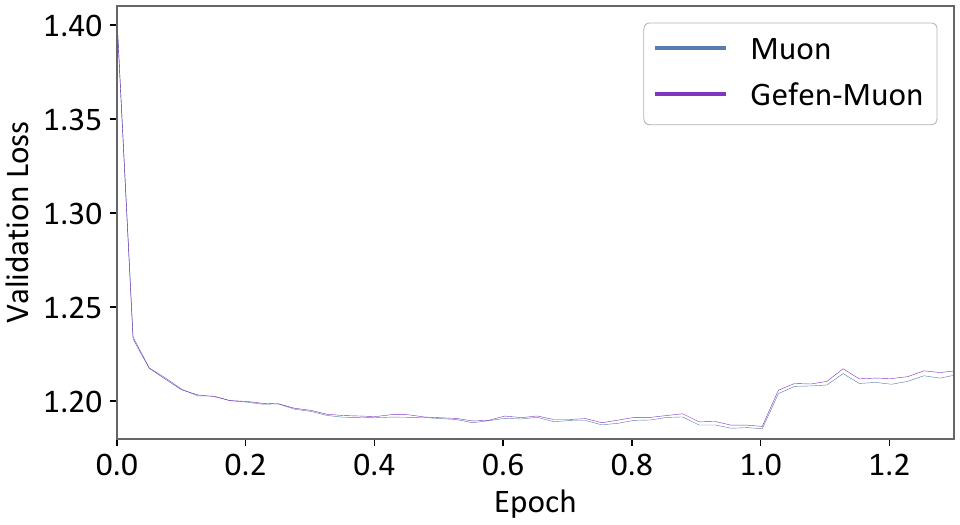}
                };
                \spy[red] on (5.5,0.8) in node at (3.2,2.5);
            \end{tikzpicture}%
        }

        \centering (b) Finetuning Qwen3-1.7B on Alpaca.
    \end{minipage}
    \caption{Muon and \ourmethod-Muon exhibit nearly identical validation-loss curves in both pretraining and finetuning.}
    \label{fig:muon_vs_gefenmuon_curves}
\end{figure}

\section{The Need for Moments}

To demonstrate the necessity of the first- and second-moment estimates, we
compare the performance of AdamW with that of AdamW without moving averages
when training a CNN on MNIST. Specifically, we set both $\beta_1$ and
$\beta_2$ to zero, effectively reducing AdamW to SignSGD. To ensure a fair
comparison, we sweep over multiple learning rates and report its
best-performing configuration. The results are shown in
\Cref{fig:need_for_moments}.
Even at its optimal learning rate, SignSGD cannot match the performance of
AdamW, demonstrating the importance of maintaining moment estimates.

\begin{figure}[!htbp]
    \centering
    \includegraphics[width=0.6\linewidth]{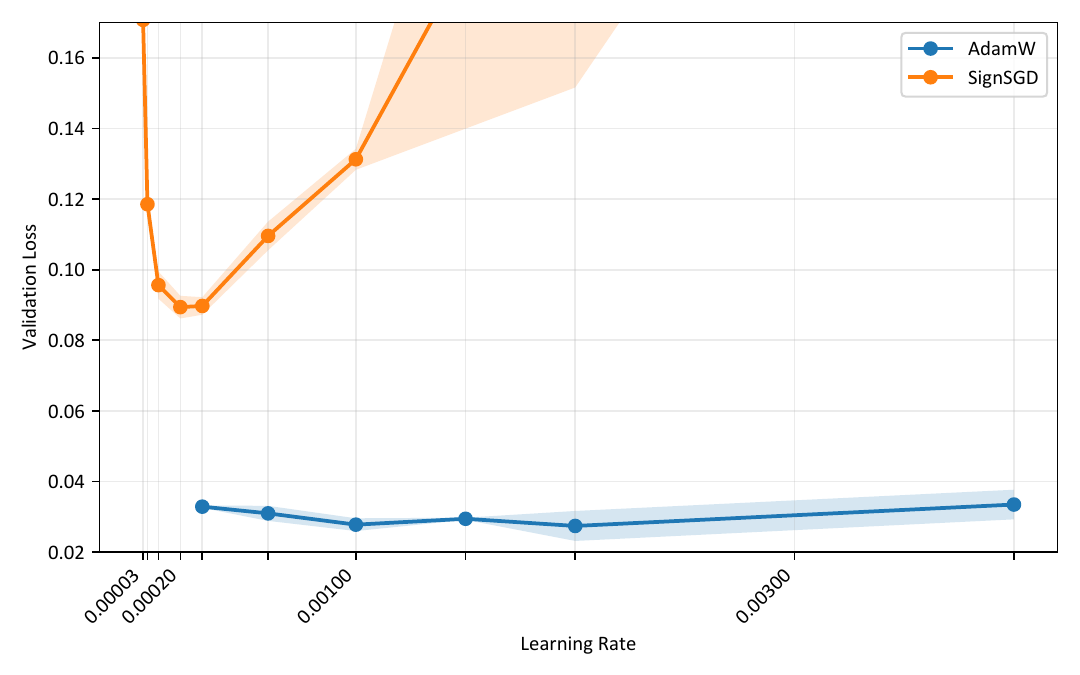}
    \caption{Comparison of AdamW and SignSGD when training a CNN on MNIST.
    SignSGD is obtained by setting $\beta_1=\beta_2=0$, and its learning rate
    is selected by sweeping over multiple values. Lines show the mean across three consecutive random seeds, while shaded bands represent the mean $\pm$ one standard deviation.}
    \label{fig:need_for_moments}
\end{figure}

\section{Limitations}
\label{app:limitations}

The partition is inferred from gradients at the beginning of training and is
then held fixed. Although we find this partition to be stable in the settings
we evaluate, distribution shifts or substantially different training regimes
could change the relevant structure later in training.

Our empirical evaluation focuses primarily on pretraining models. In
particular, the periodic squared-gradient structure used for second-moment
sharing is clearly visible in the pretraining settings we study, but is less apparent in finetuning (\cref{sec:pretraining_vs_finetuning}).
Fully characterizing how the periodic gradient structure evolves in the transition from pretraining to finetuning, including the settings and stages of training in which this transition occurs, remains an interesting open question.
The first-moment quantization component does not depend on this periodic structure and remains applicable when second-moment sharing is not.

\section{Hyperparameters}
\label[appendix]{app:hyperparameters}
\paragraph{Llama Pre-training.}
For \cref{fig:gpt2_and_llama_train_curve} we pre-trained \llamaoneb using the TorchTitan codebase~\citep{liang2025torchtitan} on the C4 dataset~\citep{raffel2020exploring}. We use weight decay coefficient $\lambda = 0.0$, $\epsilon = 10^{-8}$, $\beta_1 = 0.9$, $\beta_2 = 0.999$, and learning rate $=10^{-4}$. For the learning-rate schedule, we use warmup for 1\% of the total training steps, followed by linear decay. This matches the default scheduler used in the TorchTitan codebase. For \llamaoneb, we use sequence length 2048 and global batch size 256, with per-GPU batch size 2 and 64 gradient accumulation steps on 2 GeForce RTX 3090 GPUs. For Llama 8B, we use global batch size 128 and microbatch size 1 on 4 NVIDIA H100 GPUs with 80 GiB VRAM. In the same Llama 8B configuration, \ourmethod also fits on 4 NVIDIA RTX A6000 GPUs with 48 GiB VRAM, whereas AdamW runs out of memory.

\paragraph{GPT-2 training on OpenWebText.}
We use the nanoGPT codebase\footnote{\href{https://github.com/karpathy/nanoGPT/tree/master}{nanoGPT repository}} to train \gptsmall on OpenWebText (\cref{fig:gpt2_and_llama_train_curve}). We use the recommended hyperparameters: \texttt{seq\_len = 1024}, batch size $=480$, weight decay coefficient $\lambda=0.1$, $\epsilon=10^{-8}$, $\beta_1=0.9$, and $\beta_2=0.95$. We use a cosine-decay learning-rate schedule with 2000 warmup iterations. For GPT-2-small, we use the peak learning rate recommended by \citet{liu2024sophia}, which is reported to be optimal based on grid search. The chosen peak learning rate is $6\times10^{-4}$. The minimum learning rate is $3\times10^{-5}$. For SM3, this learning rate caused divergence, so we switched to a more stable learning rate of 0.225. For Adafactor and SM3, we added momentum with the same value for a fair comparison. In general, SM3 and Adafactor use different learning rates than AdamW, which makes them less suitable as drop-in replacements, but we report them for completeness. Better-tuned learning rates for SM3 and Adafactor may bring these optimizers closer to AdamW.

\paragraph{ResNet.}
For \cref{fig:gpt2_and_llama_train_curve} we use the official PyTorch implementation codebase\footnote{\href{https://github.com/pytorch/vision/blob/main/torchvision/models/resnet.py}{PyTorch torchvision ResNet implementation}} to train ResNet18~\citep{he2016deep} on ImageNet~\citep{deng2009imagenet}. We use a cosine-decay learning-rate schedule, 90 epochs, $\beta_1=0.9$, $\beta_2=0.999$, $\epsilon=10^{-8}$, and weight decay $10^{-4}$. For ResNet18, we use batch size $=256$ and peak learning rate $=0.005$.

\paragraph{CNN for MNIST.}
We use the official PyTorch implementation codebase\footnote{\href{https://github.com/pytorch/examples/tree/main/mnist}{PyTorch MNIST example implementation}} to train a CNN on MNIST on an RTX PRO 6000 Blackwell Max-Q GPU (\cref{tab:memory_footprint,tab:mnist_validation_loss}). The network consists of two convolutional layers with ReLU activations,
followed by max pooling and dropout, and then two fully connected layers ending in a 10-class digit classifier. We use batch size $=64$, test batch size $=1000$, and train for 4 epochs. For all optimizers, we use $\beta_1=0.9$, $\beta_2=0.999$, and learning rate $=0.001$, except for SM3 and Adafactor, for which we report the best result from learning rates in $\{0.001, 0.0006\}$. We use a \texttt{StepLR} scheduler with step size $1$ and $\gamma=0.7$, decaying the learning rate after each epoch. Standard deviations are computed over three consecutive seeds. For Muon, which only supports two-dimensional parameter tensors, we presented each model parameter to the optimizer as a 2D view: convolutional kernels and other higher-dimensional tensors were flattened across all dimensions except the output dimension, while one-dimensional tensors were expanded with a singleton dimension.
It is possible that other learning rates would produce stronger results at 4 epochs for some methods, but our goal is to test AdamW compatibility; exhaustively tuning all learning rates for all methods is outside the scope of this paper.

\paragraph{DDPM.}
For \cref{fig:gpt2_and_llama_train_curve} we train an unconditional DDPM~\citep{ho2020denoising} on the Bollywood-Celebs dataset using the Hugging Face Diffusers codebase\footnote{\href{https://github.com/huggingface/diffusers}{Hugging Face Diffusers repository}}, specifically the \texttt{UNet2DModel}. The image size is $64$, and the training objective is to predict the added noise with mean-squared error. We use base channel width $64$, two layers per block, and U-Net channel multipliers $(1,2,4,8)$, corresponding to block widths $(64,128,256,512)$. The downsampling blocks are two standard convolutional blocks followed by two attention blocks, and the upsampling path uses the corresponding attention and convolutional blocks in reverse order. We use a diffusion process with $1000$ training timesteps and the squared-cosine beta schedule. We train for $10$ epochs with batch size $128$, learning rate $5\times10^{-5}$, cosine learning-rate decay, weight decay $0$, $\epsilon=10^{-8}$, $\beta_1=0.9$, and $\beta_2=0.999$.

\paragraph{DiT-XL/2 on ImageNet.}
For the DiT-XL/2 experiment in \cref{fig:gpt2_and_llama_train_curve,app:more_experiments}, we train a class-conditional DiT-XL/2 on ImageNet at $256\times256$ resolution using the Hugging Face Diffusers implementation. Images are center-cropped, randomly flipped, and encoded into $32\times32$ latents by the EMA variant of the Stable Diffusion VAE. The transformer uses patch size $2$, 28 layers, 16 attention heads with head dimension 72, and AdaLN-Zero conditioning over the 1000 ImageNet classes. We train for 10 epochs with effective global batch size 16, learning rate $10^{-4}$, weight decay $0$, $\epsilon=10^{-8}$, $\beta_1=0.9$, and $\beta_2=0.999$, using bfloat16 precision and no learning-rate schedule. The diffusion process uses 1000 timesteps with a linear beta schedule, and the model is trained to predict the added noise using mean-squared error.

\paragraph{Qwen1.5-MoE-A2.7B pre-training.}
For \cref{fig:gpt2_and_llama_train_curve}, we pre-train Qwen1.5-MoE-A2.7B~\citep{qwen_moe}, a mixture-of-experts model with 14.3B total and 2.7B active parameters per token, from random initialization on C4~\citep{raffel2020exploring} using Hugging Face Transformers with FSDP2 and a custom training loop. We train for 5.0B tokens with sequence length 2048, global batch size 128, learning rate $10^{-4}$, weight decay $0$, $\epsilon=10^{-8}$, $\beta_1=0.9$, and $\beta_2=0.999$. We use 1\% linear warmup followed by linear decay to zero and gradient clipping at 1.0. Training uses bfloat16 compute with fp32 master weights and gradient reduction, per-GPU microbatch size 8, and 4 gradient accumulation steps on four NVIDIA B300 GPUs. Both optimizers use identical initial weights, data order, and training settings; \ourmethod uses its default configuration. AdamW's optimizer state occupies 106.7 GiB, compared with 13.4 GiB for \ourmethod.

\paragraph{Qwen3-4B pre-training.}
For \cref{fig:gpt2_and_llama_train_curve}, we pre-train Qwen3-4B~\citep{qwen3} on the FineWeb-Edu dataset~\citep{penedo2024fineweb} using the Axolotl codebase~\citep{axolotl}. We train with sequence length 2048, global batch size 128, and learning rate $3\times10^{-4}$. We use a cosine learning-rate schedule with 10 warmup steps. Training uses bfloat16 precision and FlashAttention~2. We use a per-GPU microbatch size of 1 and 32 gradient accumulation steps on four NVIDIA H100 GPUs.

\paragraph{Qwen3-1.7B finetuning on Alpaca.}
For the finetuning experiment in \cref{fig:muon_vs_gefenmuon_curves}, we train Qwen3-1.7B-Base~\citep{qwen3} on the Alpaca dataset using the Axolotl codebase~\citep{axolotl}. We train with sequence length 2048, global batch size of 128, and learning rate $5\times10^{-5}$. We use a constant learning-rate schedule with 5 warmup steps, bfloat16 precision, and FlashAttention~2. 

\paragraph{Throughput.} 
For FSDP in \cref{fig:throughput} we used \llamaoneb on C4 with the default sequence length 2048, global batch size 256, and two RTX 3090 GPUs. In this setting, AdamW and Adam-mini support a microbatch size of 1, whereas \ourmethod supports a microbatch size of 2.
For DDP in \cref{fig:throughput} we used \gptoneb on OpenWebText with the default sequence length 1024, global batch size 120, and two RTX 3090 GPUs. In this setting, AdamW cannot fit a microbatch size of 1, Adam-mini supports a microbatch size of 1, and \ourmethod supports a microbatch size of 2.

\end{document}